%% file: main.tex
\documentclass{article}

\usepackage[final]{neurips_2024}

\usepackage[utf8]{inputenc} %
\usepackage[T1]{fontenc}    %
\usepackage{booktabs}       %
\usepackage{amsfonts}       %
\usepackage{nicefrac}       %
\usepackage{microtype}      %
\usepackage{xcolor}         %

\input{preamble}

\title{Transductive Active Learning: \\ Theory and Applications}

\author{
    Jonas Hübotter\thanks{Correspondence to \texttt{jonas.huebotter@inf.ethz.ch}} \quad Bhavya Sukhija \quad Lenart Treven \quad Yarden As \quad Andreas Krause \\[3pt]
    Department of Computer Science\\
    ETH Zürich, Switzerland
}

\begin{document}

\maketitle

\vspace{-0.6\baselineskip}\begin{abstract}
  We study a generalization of classical active learning to real-world settings with concrete prediction targets where sampling is restricted to an accessible region of the domain, while prediction targets may lie outside this region.
  We analyze a family of decision rules that sample adaptively to minimize uncertainty about prediction targets.
  We are the first to show, under general regularity assumptions, that such decision rules converge uniformly to the smallest possible uncertainty obtainable from the accessible data.
  We demonstrate their strong sample efficiency in two key applications: active fine-tuning of large neural networks and safe Bayesian optimization, where they achieve state-of-the-art performance.
\end{abstract}

\input{mainmatter/01_introduction}

\input{mainmatter/011_problem_setting}
\input{mainmatter/02_algorithm}
\input{mainmatter/04_nns}
\input{mainmatter/05_safe_bo}

\input{mainmatter/06_related_work}
\input{mainmatter/07_conclusion}

\section*{Acknowledgements}

Many thanks to Armin Lederer, Johannes Kirschner, Jonas Rothfuss, Lars Lorch, Manish Prajapat, Nicolas~Emmenegger, Parnian Kassraie, and Scott Sussex for their insightful feedback on different versions of this manuscript, as well as Anton Baumann for helpful discussions.
We further thank Freddie Bickford Smith for a constructive discussion regarding the relationship between our work and prior work.\looseness=-1

This project was supported in part by the European Research Council (ERC) under the European Union's Horizon 2020 research and Innovation Program Grant agreement no.~815943, the Swiss National Science Foundation under NCCR Automation, grant agreement~51NF40~180545, and by a grant of the Hasler foundation (grant no.~21039).
\looseness=-1

{\bibliography{sources}}
\bibliographystyle{icml2024} %
\clearpage
\appendix
\section*{\LARGE Appendices}

A general principle of ``transductive learning'' was already formulated by the famous computer scientist Vladimir Vapnik in the 20th century.
Vapnik proposes the following ``imperative for a complex world'':

\begin{chapquote}{\cite{vapnik1982estimation}}
    When solving a problem of interest, do not solve a more general problem as an intermediate step. Try to get the answer that you really need but not a more general one.
\end{chapquote}

These appendices provide additional background, proofs, experiment details, and ablation studies.

\section*{Contents}
\startcontents
\printcontents{}{0}[2]{}
\clearpage

\input{backmatter/0_definitions}
\input{backmatter/C_proofs}
\input{backmatter/A_approximations}
\input{backmatter/D_generalizations}

\input{backmatter/1_closed_form}
\input{backmatter/E_computational_complexity}
\input{backmatter/G_additional_gp_experiments}
\input{backmatter/H_additional_nn_experiments}
\input{backmatter/I_additional_safe_bo_experiments}

\clearpage
\begin{table*}[t]
    \caption{Magnitudes of $\gamma_n$ for common kernels. The magnitudes hold under the assumption that $\spX$ is compact. Here, $B_{\nu}$ is the modified Bessel function. We take the magnitudes from Theorem 5 of \cite{srinivas2009gaussian} and Remark 2 of \cite{vakili2021information}. The notation $\smash{\BigOTil{\cdot}}$ subsumes log-factors. For $\nu = 1/2$, the Matérn kernel is equivalent to the Laplace kernel. For $\nu \to \infty$, the Matérn kernel is equivalent to the Gaussian kernel. The functions sampled from a Matérn kernel are $\lceil\nu\rceil-1$ mean square differentiable. The kernel-agnostic bound follows by simple reduction to a linear kernel in $\abs{\spX}$ dimensions.}
    \label{table:gamma_rates}
    \vskip 0.15in
    \begin{center}
    \begin{tabular}{@{}lll@{}}
      \toprule
      Kernel & $k(\vx,\vxp)$ & $\gamma_n$ \\
      \midrule
      Linear & $\transpose{\vx} \vxp$ & $\BigO{d \log(n)}$ \\
      Gaussian & $\exp\parentheses*{-\frac{\norm{\vx - \vxp}_2^2}{2 h^2}}$ & $\BigOTil{\log^{d+1}(n)}$ \\
      Laplace & $\exp\parentheses*{-\frac{\norm{\vx - \vxp}_1}{h}}$ & $\BigOTil{n^{\frac{d}{1 + d}}\log^{\frac{1}{1+d}}(n)}$ \\
      Matérn & $\frac{2^{1-\nu}}{\Gamma(\nu)}\parentheses*{\frac{\sqrt{2\nu}\norm{\vx-\vxp}_2}{h}}^{\nu} B_{\nu} \parentheses*{\frac{\sqrt{2\nu}\norm{\vx-\vxp}_2}{h}}$ & $\BigOTil{n^{\frac{d}{2\nu + d}}\log^{\frac{2\nu}{2\nu+d}}(n)}$ \\
      \midrule
      any & & $\BigO{\abs{\spX} \log(n)}$ \\
      \bottomrule
    \end{tabular}
    \end{center}
    \vskip -0.1in
\end{table*}

\end{document}

%% file: preamble.tex
\usepackage{microtype}
\usepackage{graphicx}
\usepackage{subfigure}
\usepackage{booktabs} %
\usepackage{xspace}
\usepackage{titletoc}
\usepackage{enumitem}
\usepackage{xfrac}
\usepackage{wrapfig}
\usepackage{algorithm,algorithmic}

\definecolor{chaptercolor}{HTML}{1A254B}
\definecolor{darkblue}{HTML}{1A254B}
\definecolor{linkcolor}{HTML}{2B50AA}
\definecolor{citecolor}{HTML}{2B50AA}
\definecolor{lightlinkcolor}{HTML}{9A8F97}
\definecolor{darklinkcolor}{HTML}{1A254B}
\definecolor{pink}{HTML}{E05F60}
\definecolor{lightblue}{HTML}{A7BED3}
\definecolor{red}{HTML}{F2545B}
\definecolor{blue}{HTML}{2b50aa}

\usepackage[colorlinks,linkcolor=linkcolor,citecolor=citecolor]{hyperref}
\usepackage{url}

\usepackage{amsmath}
\usepackage{amssymb}
\usepackage{mathtools}
\usepackage{amsthm}
\usepackage{bbm}
\usepackage{nicefrac}

\usepackage[capitalize,noabbrev]{cleveref}

\theoremstyle{plain}
\newtheorem{theorem}{Theorem}[section]
\newtheorem{proposition}[theorem]{Proposition}
\newtheorem{lemma}[theorem]{Lemma}
\newtheorem{inflemma}[theorem]{Informal Lemma}
\newtheorem{corollary}[theorem]{Corollary}
\theoremstyle{definition}
\newtheorem{definition}[theorem]{Definition}
\newtheorem{assumption}[theorem]{Assumption}
\newtheorem{example}[theorem]{Example}
\theoremstyle{remark}
\newtheorem{remark}[theorem]{Remark}

\Crefname{assumption}{Assumption}{Assumptions}
\crefname{assumption}{Assumption}{Assumptions}

\allowdisplaybreaks

\makeatletter
\renewcommand{\paragraph}{%
  \@startsection{paragraph}{4}%
  {\z@}{0ex \@plus 0ex \@minus 0ex}{-1em}%
  {\normalfont\normalsize\bfseries}%
}
\makeatother

\newlist{lemenum}{enumerate}{1} %
\setlist[lemenum]{label=(\roman*),ref=\thelemma\,(\roman*),topsep=0pt}
\Crefname{lemenumi}{Lemma}{Lemmas}

\newlist{corenum}{enumerate}{1} %
\setlist[corenum]{label=(\roman*),ref=\thecorollary\,(\roman*),topsep=0pt}
\Crefname{corenumi}{Corollary}{Corollaries}

\newlist{thmenum}{enumerate}{1} %
\setlist[thmenum]{label=(\roman*),ref=\thetheorem\,(\roman*),topsep=0pt}
\Crefname{thmenumi}{Theorem}{Theorems}

\newlist{propenum}{enumerate}{1} %
\setlist[propenum]{label=(\roman*),ref=\thedefinition\,(\roman*),topsep=0pt}
\Crefname{propenumi}{Property}{Properties}

\newlist{assenum}{enumerate}{1} %
\setlist[assenum]{label=(\roman*),ref=\theassumption\,(\roman*),topsep=0pt}
\Crefname{assenumi}{Assumption}{Assumptions}

\usepackage{svg}
\usepackage{graphicx}
\usepackage{rotating}
\usepackage{tikz}
\usepackage{pgfplots}
\pgfplotsset{compat=newest}
\usetikzlibrary{shapes.geometric}

\usepackage{import}
\usepackage{xifthen}
\usepackage{pdfpages}
\usepackage{transparent}

\NewDocumentCommand{\incfig}{mo}{
  \begin{center}
    \IfValueT{#2}{\def\svgwidth{#2}}{\def\svgwidth{\columnwidth}}
    \import{./figures/}{#1.pdf_tex}
  \end{center}
}

\usepackage{pgf}
\usepackage{adjustbox}

\NewDocumentCommand{\incplt}{O{\columnwidth}m}{%
  \begin{center}
    \adjustbox{width=#1}{\import{./plots/output/}{#2.pgf}}
  \end{center}
}

\makeatletter
\newenvironment{chapquote}[2][2em]
  {%
   \def\chapquote@author{#2}%
   \itshape}
  {\par\normalfont\hfill--\ \chapquote@author\par}
\makeatother

\usepackage{aligned-overset}
\newcommand{\rmnum}[1]{\romannumeral #1}

\newcommand{\eleq}[1]{\overset{(\rmnum{#1})}&{\leq}}
\newcommand{\eeq}[1]{\overset{(\rmnum{#1})}&{=}}
\newcommand{\e}[1]{\ensuremath{(\rmnum{#1})}}
\newcommand{\ebf}[1]{\ensuremath{\textbf{(\rmnum{#1})}}}

\newcommand{\itl}{{\textsc{ITL}}\xspace}
\newcommand{\gitl}{{\textbf{\texttt{G-ITL}}}\xspace}
\newcommand{\litl}{{\textbf{\texttt{L-ITL}}}\xspace}
\newcommand{\mmitl}{{\textsc{MM-ITL}}\xspace}
\newcommand{\vtl}{\textsc{VTL}\xspace}
\newcommand{\ctl}{\textsc{CTL}\xspace}
\newcommand{\gctl}{\textbf{\texttt{G-CTL}}\xspace}
\newcommand{\lctl}{\textbf{\texttt{L-CTL}}\xspace}

\DeclareFontFamily{U}{mathb}{\hyphenchar\font45}
\DeclareFontShape{U}{mathb}{m}{n}{
      <5> <6> <7> <8> <9> <10> gen * mathb
      <10.95> mathb10 <12> <14.4> <17.28> <20.74> <24.88> mathb12
      }{}
\DeclareSymbolFont{mathb}{U}{mathb}{m}{n}
\DeclareFontSubstitution{U}{mathb}{m}{n}
\DeclareMathSymbol{\Asterisk}      {2}{mathb}{"06}
\newcommand{\mainsym}{\textcolor{blue}{\parentheses*{\hspace{-0.1em}\Asterisk\hspace{-0.1em}}}}

\newcommand{\psym}{\ensuremath{\textcolor{blue}{\dagger}}\xspace}
\newcommand{\pref}{\hyperref[principle]{\psym}\xspace}

\newcommand*{\abs}[1]{| #1 |}

\NewDocumentCommand{\norm}{sm}{\IfBooleanTF{#1}{\|#2\|}{\left\| #2 \right\|}}

\newcommand*{\const}{\mathrm{const}}

\DeclareMathOperator*{\defeq}{\smash{\overset{\mathrm{def}}{=}}}
\DeclareMathOperator*{\eqdef}{\smash{\overset{\mathrm{def}}{=}}}

\DeclareMathOperator*{\argmax}{arg\,max}
\DeclareMathOperator*{\argmin}{arg\,min}

\DeclareMathOperator*{\spn}{span}

\DeclarePairedDelimiter\parentheses{(}{)}
\DeclarePairedDelimiter\brackets{[}{]}
\DeclarePairedDelimiter\braces{\{}{\}}

\newcommand{\R}{\mathbb{R}}

\newcommand{\Nat}{\mathbb{N}}

\renewcommand{\vec}[1]{\boldsymbol{#1}}
\newcommand{\mat}[1]{\boldsymbol{#1}}

\newcommand{\spa}[1]{\mathcal{#1}}
\newcommand{\opt}[1]{#1^\star}

\DeclareMathOperator*{\iid}{\smash{\overset{\mathrm{iid}}{\sim}}}
\DeclareMathOperator*{\uar}{\smash{\overset{\mathrm{u.a.r.}}{\sim}}}

\NewDocumentCommand{\Ind}{m}{\mathbbm{1}\{{#1}\}}
\NewDocumentCommand{\fnPr}{}{\mathbb{P}}
\RenewDocumentCommand{\Pr}{om}{\fnPr\IfValueT{#1}{_{#1}}\parentheses*{#2}}
\NewDocumentCommand{\Prsm}{om}{\fnPr\IfValueT{#1}{_{#1}}\parentheses{#2}}
\RenewDocumentCommand{\H}{mo}{\mathrm{H}\IfValueTF{#2}{\!\left(#1\ \middle|\ #2\right)}{\parentheses*{#1}}}
\NewDocumentCommand{\Hsm}{mo}{\mathrm{H}\IfValueTF{#2}{(#1 \mid #2)}{\parentheses{#1}}}
\NewDocumentCommand{\I}{mmo}{\mathrm{I}\IfValueTF{#3}{\!\left(#1;#2\ \middle|\ #3\right)}{\parentheses*{#1; #2}}}
\NewDocumentCommand{\Ism}{mmo}{\mathrm{I}\IfValueTF{#3}{(#1;#2 \mid #3)}{\parentheses{#1; #2}}}

\NewDocumentCommand{\E}{somo}{\ensuremath{\mathbb{E}\IfValueT{#2}{_{#2}}{} \IfBooleanTF{#1}{#3}{\IfValueTF{#4}{\!\left[#3\ \middle|\ #4\right]}{\brackets*{#3}}}}}
\NewDocumentCommand{\Esm}{somo}{\ensuremath{\mathbb{E}\IfValueT{#2}{_{#2}}{} \IfBooleanTF{#1}{#3}{\IfValueTF{#4}{\!\left[#3\ \middle|\ #4\right]}{\brackets{#3}}}}}
\NewDocumentCommand{\Var}{somo}{\mathrm{Var}\IfValueT{#2}{_{#2}}{} \IfBooleanTF{#1}{#3}{\IfValueTF{#4}{\!\left(#3\ \middle|\ #4\right)}{\parentheses*{#3}}}}
\NewDocumentCommand{\Varsm}{somo}{\mathrm{Var}\IfValueT{#2}{_{#2}}{} \IfBooleanTF{#1}{#3}{\IfValueTF{#4}{\left(#3\ \middle|\ #4\right)}{\parentheses{#3}}}}
\NewDocumentCommand{\Cov}{som}{\mathrm{Cov}\IfValueT{#2}{_{#2}}{} \IfBooleanTF{#1}{#3}{\parentheses*{#3}}}
\NewDocumentCommand{\Cor}{som}{\mathrm{Cor}\IfValueT{#2}{_{#2}}{} \IfBooleanTF{#1}{#3}{\parentheses*{#3}}}
\NewDocumentCommand{\KL}{mm}{\mathrm{KL}\parentheses*{#1 \| #2}}

\NewDocumentCommand{\grad}{e_}{\boldsymbol{\nabla}\IfValueT{#1}{_{\!\!#1}\,}}

\NewDocumentCommand{\BigO}{m}{O\parentheses*{#1}}
\NewDocumentCommand{\BigOTil}{m}{\widetilde{O}\parentheses*{#1}}

\NewDocumentCommand{\transpose}{m}{#1^\top}
\NewDocumentCommand{\inv}{m}{#1^{-1}}
\RenewDocumentCommand{\det}{m}{\left| #1 \right|}

\NewDocumentCommand{\tr}{m}{\mathrm{tr}\;#1}
\NewDocumentCommand{\diag}{som}{\mathrm{diag}\IfValueT{#2}{_{#2}}{}\,#3}
\NewDocumentCommand{\msqrt}{m}{#1^{\nicefrac{1}{2}}}

\NewDocumentCommand{\N}{somm}{\mathcal{N}\IfBooleanTF{#1}{\left(}{(}\IfValueT{#2}{#2;}{} #3, #4\IfBooleanTF{#1}{\right)}{)}}
\NewDocumentCommand{\GP}{omm}{\mathcal{GP}(\IfValueT{#1}{#1;}{} #2, #3)}

\newcommand{\vzero}{\vec{0}}

\newcommand{\vf}{\vec{f}}

\newcommand{\vfsub}[1]{\vec{f}_{\!\!#1}}

\newcommand{\vk}{\vec{k}}

\newcommand{\vu}{\vec{u}}

\newcommand{\vv}{\vec{v}}

\newcommand{\vx}{\vec{x}}
\newcommand{\vs}{\vec{s}}
\newcommand{\vT}{\vec{T}}
\newcommand{\vxhat}{\widehat{\vec{x}}}
\newcommand{\vxp}{\vec{x'}}

\newcommand{\vy}{\vec{y}}
\newcommand{\vysub}[1]{\vec{y}_{\!#1}}

\newcommand{\vz}{\vec{z}}
\newcommand{\vbeta}{\boldsymbol{\beta}}
\newcommand{\vdelta}{\boldsymbol{\delta}}

\newcommand{\vvarepsilon}{\boldsymbol{\varepsilon}}

\newcommand{\vmu}{\boldsymbol{\mu}}
\newcommand{\vmusub}[1]{\boldsymbol{\mu}_{\!#1}}

\newcommand{\vphi}{\boldsymbol{\phi}}
\newcommand{\vpi}{\boldsymbol{\pi}}

\newcommand{\vtheta}{\boldsymbol{\theta}}
\newcommand{\vthetap}{\boldsymbol{\theta'}}
\newcommand{\vthetahat}{\boldsymbol{\widehat{\theta}}}

\newcommand{\mzero}{\mat{0}}
\newcommand{\mA}{\mat{A}}

\newcommand{\mD}{\mat{D}}

\newcommand{\mI}{\mat{I}}

\newcommand{\mK}{\mat{K}}
\newcommand{\mKsub}[1]{\mat{K}_{\!#1}}

\newcommand{\mP}{\mat{P}}
\newcommand{\mPsub}[1]{\mat{P}_{\!#1}}
\newcommand{\mPhi}{\mat{\Phi}}

\newcommand{\mPi}{\mat{\Pi}}

\newcommand{\mQ}{\mat{Q}}
\newcommand{\mR}{\mat{R}}

\newcommand{\mV}{\mat{V}}

\newcommand{\mLambda}{\mat{\Lambda}}
\newcommand{\mSigma}{\mat{\Sigma}}

\newcommand{\spA}{\spa{A}}
\newcommand{\spB}{\spa{B}}
\newcommand{\spC}{\spa{C}}

\newcommand{\spD}{\spa{D}}
\newcommand{\spE}{\spa{E}}
\newcommand{\spF}{\spa{F}}
\newcommand{\spG}{\spa{G}}
\newcommand{\spH}{\spa{H}}
\newcommand{\spI}{\spa{I}}
\newcommand{\spL}{\spa{L}}
\newcommand{\spM}{\spa{M}}

\newcommand{\spO}{\spa{O}}

\newcommand{\spR}{\spa{R}}
\newcommand{\spS}{\spa{S}}

\newcommand{\spU}{\spa{U}}
\newcommand{\spW}{\spa{W}}
\newcommand{\spX}{\spa{X}}
\newcommand{\spY}{\spa{Y}}

\newcommand{\spPA}{\ensuremath{\spa{P}_{\!\!\spA}}}
\newcommand{\spPS}{\ensuremath{\spa{P}_{\!\spS}}}

\newcommand{\kappabar}{\bar{\kappa}}

%% file: mainmatter/01_introduction.tex
\vspace{-3pt}
\section{Introduction}
\vspace{-2pt}

\begin{wrapfigure}{r}{0.4\textwidth}
  \vspace{-0.3cm}
  \centering
  \includesvg[width=0.4\columnwidth]{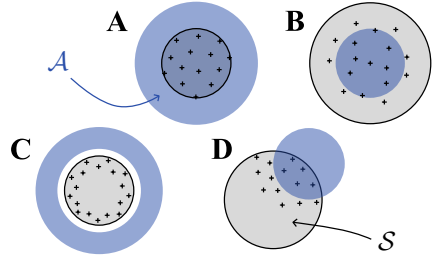}
  \vspace{-0.4cm}
  \caption{Instances of transductive active learning with target~space~$\spA$ shown in blue and sample~space~$\spS$ shown in gray. The points denote plausible observations within~$\spS$ to ``learn''~$\spA$. In \textbf{(A)}, the target space contains ``everything'' within~$\spS$ as well as points \emph{outside}~$\spS$. In \textbf{(B, C, D)}, one makes observations \emph{directed} towards learning about a particular target. Prior work on inductive active learning has focused on the instance $\spA = \spS$.\looseness=-1}
  \label{fig:transductive_active_learning}
  \vspace{-0.3cm}
\end{wrapfigure}

Machine learning, at its core, is about designing systems that can extract knowledge or patterns from data.
One part of this challenge is determining not just how to learn from observational data but deciding what data to obtain next, given the information already available.
More formally, for an unknown and sufficiently regular function~$f$ over a domain~$\spX$, we ask: \emph{How can we learn $f$ sample-efficiently from (noisy) observations?}
This problem is widely studied in active learning and experimental design \citep{chaloner1995bayesian,settles2009active}.\looseness-1

Active learning methods commonly aim to learn~$f$ globally, i.e., across the entire domain~$\spX$.
However, in many real-world problems, \ebf{1} the domain is so large that learning~$f$ globally is hopeless or \ebf{2} agents have limited information and cannot access the entire domain~(e.g., due to restricted access or to act safely). %
Thus, global learning is often not desirable or even possible.
Instead, intelligent systems are typically required to act in a more \emph{directed} manner and \emph{extrapolate} beyond their limited information.
This work formalizes the above two aspects of active learning, which have remained largely unaddressed by prior work.
We provide a comprehensive overview of related work in \cref{sec:related_work}.\looseness-1

\paragraph{``Directed'' transductive active learning}

We consider the generalized problem of \emph{transductive active learning} (TAL), where given two arbitrary subsets of the domain $\spX$; a \emph{target space} $\spA \subseteq \spX$, and a \emph{sample space} $\spS \subseteq \spX$, we study the question: \vspace{-0.2cm}\begin{center}
    \emph{How can we learn $f$ within $\spA$ by actively sampling observations within $\spS$?}
\end{center}

This problem is ubiquitous in real-world applications such as safe Bayesian optimization, where~$\spS$ is a set of safe parameters and~$\spA$ might represent parameters outside~$\spS$ whose safety we want to infer.
Active fine-tuning of neural networks is another example, where the target space~$\spA$ represents the test set over which we want to minimize risk, and the sample space~$\spS$ represents the dataset from which we can retrieve data points to fine-tune our model to~$\spA$.
\Cref{fig:transductive_active_learning} visualizes some instances of TAL.\looseness=-1

Whereas most prior work has focused on the ``global'' inductive instance ${\spX = \spA = \spS}$, \cite{mackay1992information} was the first to consider specific target spaces~$\spA$ and proposed the principle of selecting points in~$\spS$ to minimize the ``posterior uncertainty'' about points in~$\spA$.
Since then, several works have studied this principle empirically~\citep[e.g.,][]{seo2000gaussian,yu2006active,bogunovic2016truncated,wang2021beyond,kothawade2021similar,smith2023prediction}.
In this work, we model $f$ as a Gaussian process or (equivalently) as a function in a reproducing kernel Hilbert space, for which the above principle is analytically and computationally tractable.
Our contributions are: \begin{itemize}
  \item \textbf{Theory (\cref{sec:algorithm}):} We are the first to give rates for the uniform convergence of uncertainty over the target space~$\spA$ to the smallest attainable value, given samples from the sample space~$\spS$~(\cref{thm:variance_convergence}),
  Our results provide a theoretical justification for the principle of minimizing posterior uncertainty in transductive active learning, and indicate that transductive active learning can be more sample efficient than inductive active learning.

  \item \textbf{Applications:} We show that transductive active learning improves upon the state-of-the-art in the batch-wise \emph{active fine-tuning} of neural networks for image classification~(\cref{sec:nns}) and in \emph{safe Bayesian optimization}~(\cref{sec:safe_bo}).\looseness=-1
\end{itemize}

%% file: mainmatter/011_problem_setting.tex
\section{Problem Setting}\label{sec:problem setting}

In this work, we study algorithmic principles which can be derived from various perspectives.
In our presentation, we adopt a probabilistic view and model~$f$ as a stochastic process with marginals~$f_{\vx}$ and joint random vectors~$\{f_{\vx}\}_{\vx \in X}$ for~${X \subseteq \spX}, {\abs{X} < \infty}$ denoted by~$\vfsub{X}$.
Let~$\vysub{X}$ denote noisy observations of~$\vfsub{X}$, ${\{y_{\vx} = f_{\vx} + \varepsilon_{\vx}\}_{\vx \in X}}$, where $\varepsilon_{\vx}$ is independent noise.\footnote{$X$ may be a multiset in which case repeated occurrence of $\vx$ corresponds to independent observations of~$y_{\vx}$.}
We study the ``adaptive'' setting, where in round $n$ the agent selects a point $\vx_n \in \spS$ and observes ${y_n = y_{\vx_n}}$.
The agent's choice of $\vx_n$ may depend on the outcome of prior observations $\spD_{n-1} \defeq \{(\vx_i, y_i)\}_{i<n}$.
We assume here for simplicity that target space $\spA$ and sample space $\spS$ are finite, and relax this in \cref{sec:proofs:variance_convergence_generalization,sec:generalizations:roi}.
\looseness=-1

\paragraph{Background on information theory}

We briefly recap several important concepts from information theory of which we provide formal definitions in \cref{sec:definitions}.
The (differential) entropy~$\H{\vf}$ is one possible measure of uncertainty about~$\vf$ and the conditional entropy $\H{\vf}[\vy]$ is the (expected) posterior uncertainty about~$\vf$ after observing~$\vy$.
The information gain~${\I{\vf}{\vy} = \H{\vf} - \H{\vf}[\vy]}$ measures the (expected) reduction in uncertainty about $\vf$ due to $\vy$.
We denote the information gain about $\spA$ from observing $X$ by $\I{\vfsub{\spA}}{\vysub{X}}$.
The maximum information gain about $\spA$ from $n$ observations within $\spS$ is \begin{align*}
    \gamma_{\spA,\spS}(n) \defeq \max_{\substack{X \subseteq \spS \\ \abs{X} \leq n}} \I{\vfsub{\spA}}{\vysub{X}}.
\end{align*}
This ``information capacity'' measures the information about $\vfsub{\spA}$ that is accessible from within $\spS$, and has been used previously \citep[e.g., by][]{srinivas2009gaussian,chowdhury2017kernelized,vakili2021information} in the setting where ${\spX = \spA = \spS}$, taking the form of ${\gamma_n \defeq \gamma_{\spX}(n) \defeq \gamma_{\spX, \spX}(n)}$.
We remark that ${\gamma_{\spA, \spS}(n) \leq \gamma_{\spS}(n)}$ holds uniformly for all $\spA$, $\spS$, and $n$ due to the data processing inequality.
Generally, $\gamma_{\spA, \spS}(n)$ can be substantially smaller if the target space is a sparse subset of the sample space.\looseness=-1

\subsection{Algorithms for Transductive Active Learning}

We analyze the following unifying principle for transductive active learning: \begin{center}
    \emph{Select samples to minimize the posterior ``uncertainty'' about $f$ within $\spA$ (in expectation).} %
\end{center}
This principle yields a family of simple and natural decision rules which depend on the chosen measure of \emph{uncertainty} $\mathrm{Unc}_{\!\spA}$: \begin{align}
    \vx_{n} &= \argmin_{\vx \in \spS} \E*[y \sim p(y_{\vx} \mid \spD_{n-1})]{\mathrm{Unc}_{\!\spA}(\spD_{n-1} \cup \{(\vx, y)\})}. \label{eq:alg}
\end{align}
That is, the decision rules select $\vx_n$ so as to minimize the uncertainty about the prediction targets~$\vfsub{\spA}$ as quantified by the uncertainty $\mathrm{Unc}_{\!\spA}$ (in expectation) after having received the feedback $y_n$.
The following measures of uncertainty have been explored in prior work.\looseness=-1 \begin{itemize}
    \item \textbf{\vtl}: the total \emph{variance} of prediction targets, ${\mathrm{Unc}_{\!\spA}(\spD) = \sum_{\vxp \in \spA} \Var{f_{\vxp}}[\spD]}$, previously explored by \cite{seo2000gaussian} and \cite{yu2006active}.
    \item \textbf{\itl}: the \emph{entropy} of prediction targets, ${\mathrm{Unc}_{\!\spA}(\spD) = \H{\vfsub{\spA}}[\spD]}$, which was originally proposed by \cite{mackay1992information} and later studied by \cite{kothawade2021similar}.
    Note that unlike variance, entropy takes into account the mutual dependence between points in~$\spA$.
    Furthermore, with entropy, \cref{eq:alg} can be expressed simply as maximizing \emph{information gain}:\looseness=-1 \begin{align}
        \vx_{n} &= \argmin_{\vx \in \spS} \H{\vfsub{\spA}}[\spD_{n-1}, y_{\vx}] = \argmax_{\vx \in \spS} \Ism{\vfsub{\spA}}{y_{\vx}}[\spD_{n-1}].
    \end{align}
\end{itemize}
Other decision rules such as the \emph{mean-marginal information gain}~\citep{mackay1992information,wang2021beyond} and its recent extension EPIG~\citep{smith2023prediction} also lie within this framework, and we discuss this instance further in \cref{sec:interpretations_approximations:mm_itl}.\looseness=-1

The transductive decision rules can recover standard active learning algorithms from the inductive setting where ${\spS \subseteq \spA}$.
For example, \itl reduces to ${\vx_n = \argmax_{\vx \in \spS} \Ism{f_{\vx}}{y_{\vx}}[\spD_{n-1}]}$, i.e., is ``undirected'' and reduces to standard uncertainty-based active learning (cf.~\cref{sec:proofs:undirected_itl}).
The convergence properties for the special instance ${\spS = \spA}$ have been studied extensively. %
Next, we derive similar guarantees in the more general setting of transductive active learning.
\looseness=-1

%% file: mainmatter/02_algorithm.tex
\section{Main Results}\label{sec:algorithm}

We show in the following that the above algorithms for TAL are tractable, both computationally and theoretically, when $f$ is a Gaussian process~\citep[GP,][]{williams2006gaussian}.
GPs are a flexible class of probabilistic models where probabilistic inference can be performed analytically.\looseness=-1

In particular, when $f \mid \spD_{n} \sim \GP{\mu_{n}}{k_{n}}$ and the noise $\varepsilon_{\vx}$ is independent and zero-mean Gaussian with known variance $\rho^2$, the \vtl and \itl objectives have a closed form: \begin{align}
    \vx_{n+1} &= \argmin_{\vx \in \spX} \sum_{\vxp \in \spA} \parentheses*{k_{n}(\vxp, \vxp) - \frac{k_{n}(\vx, \vxp)^2}{k_{n}(\vx, \vx) + \rho^2}}, \tag{\textbf{\vtl}} \\
    \vx_{n+1} &= \argmin_{\vx \in \spX} \frac{k_n(\vx,\vx) - \transpose{\vk_n(\vx)}\inv{\mK_n}\vk_n(\vx) + \rho^2}{k_n(\vx,\vx) + \rho^2}, \tag{\textbf{\itl}}
\end{align} where $(\vk_n(\vx))_{\vxp \in \spA} = k_n(\vx,\vxp)$ and $(\mK_n)_{\vx,\vxp \in \spA} = k_n(\vx,\vxp)$.
These objectives can be computed efficiently in the size of the domain $\spX$, and are easily parallelizable.

\subsection{Convergence Guarantees}

Our main contribution of this section is to derive convergence rates for \vtl and \itl.
To the best of our knowledge, this work is first to present convergence guarantees for TAL.

Consider first the non-probabilistic setting where the ``ground truth'' function $\opt{f}$ may be any sufficiently regular fixed function on $\spX$.
\cite{abbasi2013online} and \cite{chowdhury2017kernelized} have shown that in this case we may use the GP model~$f$ as a misspecified model of~$\opt{f}$:\looseness=-1
\begin{inflemma}\label{lem:confidence_intervals:informal}
    Pick ${\delta \in (0,1)}$, let $\opt{f}$ be sufficiently regular,\footnote{More formally, $\opt{f}$ must lie in the reproducing kernel Hilbert space of kernel $k$.} and let the observation noise $\{\varepsilon_i\}_{i \leq n}$ be sub-Gaussian.
    Then, for all $\vx \in \spX$ and $n \geq 0$ jointly with probability at least $1-\delta$, \begin{align}
        |\opt{f}(\vx) - \mu_{n}(\vx)| \leq \beta_{n}(\delta) \cdot \sigma_{n}(\vx)
    \end{align} with the scaling factor $\smash{\beta_n(\delta) \leq O(\norm{\opt{f}}_k + \sqrt{\gamma_n + \log(1/\delta)})}$.
    Here, ${\mu_{n}(\vx) \defeq \E{f_{\vx}}[\spD_n]}$ and ${\sigma_{n}^2(\vx) \defeq \Var{f_{\vx}}[\spD_n]}$ are mean and variance of the GP posterior of $f_{\vx}$ conditional on the observations $\spD_n$, pretending that $\{\varepsilon_i\}_{i \leq n}$ is independent Gaussian noise.\looseness=-1
\end{inflemma}
The above lemma shows that a GP may accurately model any sufficiently regular function.
The error of the GP model at a prediction target $\vx$ is governed by the posterior variance $\sigma_n(\vx)$, which depends on the selected observations $\vx_{1:n}$ and which is minimized directly by \vtl.
In the following, we will prove rates for how fast $\sigma_n(\vx)$ converges to its smallest possible value.\looseness=-1

Our convergence analysis shows that TAL can fully reduce the uncertainty $\sigma_n(\vx)$ in the prediction only if the sample space contains sufficient information to determine the true value $\opt{f}(\vx)$.
If the sample space does not contain all relevant information, the remaining uncertainty is quantified by the limiting uncertainty after seeing ``all data in the sample space infinitely often'', which we call the \emph{irreducible uncertainty}:\looseness=-1 \begin{equation*}
    \eta^2_{\spS}(\vx) \defeq \Varsm{f_{\vx} \mid \vfsub{\spS}}.
\end{equation*}
Clearly, we have $\sigma_n^2(\vx) \geq \eta^2_{\spS}(\vx)$ for any $n$ observations, that is, $\eta^2_{\spS}(\vx)$ represents the smallest uncertainty one can hope to achieve from observing only within $\spS$.
It can be seen that $\eta^2_{\spS}(\vx) = 0$ for all $\vx \in \spS$.
However, for points $\vx \not\in \spS$, it may be (and typically is!) strictly positive.\looseness=-1

We will denote the \emph{uncertainty reduction} of data $X = \{\vx_i\}_{i=1}^m$ about the prediction targets $\spA$ by \begin{equation*}
    \psi_{\spA}(X) \defeq \underbrace{\mathrm{Unc}_{\!\spA}(\emptyset)}_{\text{prior uncertainty}} -\ \underbrace{\E*[\spD_X \sim p(\vysub{X})]\mathrm{Unc}_{\!\spA}(\spD_X)}_{\text{posterior uncertainty}}
\end{equation*}
and note that, in the GP setting, \cref{eq:alg} selects $\vx_n = \argmax_{\vx \in \spS} \psi_{\spA}(\vx_{1:n-1} \cup \{\vx\})$.
In case of \itl, where uncertainty is measured by entropy, this uncertainty reduction is simply the information gain of observing $X$.
In our presented convergence guarantees, we make the following assumption.\looseness=-1

\begin{assumption}\label{asm:submodularity}
    The uncertainty reduction $\psi_{\spA}(X)$ is submodular.
\end{assumption}\vspace{-0.2cm}

Intuitively, \cref{asm:submodularity} states that the marginal uncertainty reduction achieved by adding a point to the selected data (i.e., the ``marginal gain'') decreases as the size of the selected data increases, which is a common assumption in prior work.\footnote{Similar assumptions have been made by \cite{bogunovic2016truncated} and \cite{kothawade2021similar}.}
Formally, \cref{asm:submodularity} is satisfied if, for all $\vx \in \spS$ and $X' \subseteq X \subseteq \spS$, \begin{align*}
    \Delta_{\spA}(\vx \mid X') \geq \Delta_{\spA}(\vx \mid X)
\end{align*} where $\Delta_{\spA}(\vx \mid X) \defeq \psi_{\spA}(X \cup \{\vx\}) - \psi_{\spA}(X)$ is the \emph{marginal uncertainty reduction} of $\vx$ given~$X$.

This assumption is satisfied exactly for \itl when ${\spS \subseteq \spA}$~(cf.~\cref{lem:information_ratio_monotonicity}), but is violated by some other instances (e.g., \cref{ex:gaussian_synergies}).
We observe in our experiments that the assumption holds approximately, and we provide an extensive discussion of our results in the appendix for such instances, relying on the notion of weak submodularity~\citep{das2018approximate}.
Notably, under \cref{asm:submodularity}, the achieved uncertainty reduction,~$\psi_{\spA}(\vx_{1:n})$, is a constant factor approximation of the maximal uncertainty reduction,~$\max_{X \subseteq \spS, \abs{X} \leq n}\psi_{\spA}(X)$, due to the seminal result on submodular function maximization by \cite{nemhauser1978analysis}.
That is, the iterative scheme of \cref{eq:alg} achieves a constant factor approximation of the optimal uncertainty reduction.\looseness=-1

With this, we can now state our main results on the convergence of \vtl and \itl:\looseness=-1

\begin{theorem}[Convergence rates]\label{thm:variance_convergence}
    Let \cref{asm:submodularity} hold and the data be selected by either \vtl or \itl.
    Assume that ${f \sim \GP{\mu}{k}}$ with known mean function $\mu$ and kernel $k$, the noise $\varepsilon_{\vx}$ is mutually independent and zero-mean Gaussian with known variance, and $\gamma_n$ is sublinear in $n$.
    Then there exists a constant~$C$ such that for any $n \geq 1$ and $\vx \in \spA$, \begin{align}
        \sigma_n^2(\vx) \leq \underbrace{\eta^2_{\spS}(\vx)}_{\text{irreducible}} +\ \underbrace{C \gamma_{\spA, \spS}(n) / \sqrt{n}}_{\text{reducible}}. \label{eq:variance_convergence:extrapolation}
    \end{align}
    Moreover, if ${\vx \in \spA \cap \spS}$, there exists a constant $C'$ such that \begin{align}
        \sigma_n^2(\vx) \leq C' \gamma_{\spA, \spS}(n) / n. \label{eq:variance_convergence:interpolation}
    \end{align}
\end{theorem}

Intuitively, \cref{eq:variance_convergence:extrapolation} can be understood as bounding an epistemic ``generalization gap'' \citep{wainwright2019high} of the learner, while \cref{eq:variance_convergence:interpolation} matches prior results for the setting $\spS = \spA$.
Naturally, a smaller target space (i.e., more targeted sampling) leads to faster convergence due to a smaller information capacity $\gamma_{\spA,\spS}(n) \ll \gamma_n$.
The reducible uncertainty is guaranteed to converge to zero at all prediction targets, e.g., for linear, Gaussian, and smooth Matérn kernels based on known rates for $\gamma_n$ which we summarize in \cref{table:gamma_rates} of the appendix.
We provide a formal proof of \cref{thm:variance_convergence} in \cref{sec:proofs:variance_convergence}.
In the following, we validate this result experimentally by showing that TAL exhibits strong empirical performance across a broad range of applications.\looseness=-1

\subsection{Experiments in the Gaussian Process Setting}\label{sec:gps}

First, we aim to answer the following questions in the well-specified GP setting.

\paragraph{How does the smoothness of $f$ affect \itl?}

We contrast two ``extreme'' kernels: the \emph{Gaussian kernel} $\smash{k(\vx,\vxp) = \exp(- \norm{\vx - \vxp}_2^2 / 2)}$ and the \emph{Laplace kernel} $\smash{k(\vx,\vxp) = \exp(- \norm{\vx - \vxp}_1)}$.
In the mean-squared sense, the Gaussian kernel yields a smooth process $f$ whereas the Laplace kernel yields a continuous but non-differentiable~$f$ \citep{williams2006gaussian}.
\Cref{fig:gps:directed_advantage} shows how \itl adapts to the smoothness of~$f$:
Under the ``smooth'' Gaussian kernel, points outside~$\spA$ provide higher-order information.
In contrast, under the ``rough'' Laplace kernel and if~${\spA \subseteq \spS}$, points outside~$\spA$ do not provide any additional information, and therefore are not sampled by \itl.
If, however, ${\spA \not\subseteq \spS}$, information ``leaks''~$\spA$ even under a Laplace kernel prior.
That is, even for non-smooth functions, the point with most information need not be in~$\spA$.\looseness=-1

\paragraph{Does TAL outperform uncertainty sampling?}

Uncertainty sampling \citep[\textsc{UnSa},][]{lewis1994heterogeneous} is one of the most popular active learning methods.
\textsc{UnSa} selects points $\vx$ with high \emph{prior} uncertainty: ${\vx_{n} = \argmax_{\vx \in \spS} \sigma_{n-1}^2(\vx)}$.
This is in stark contrast to \itl and \vtl which select points~$\vx$ that minimize \emph{posterior} (epistemic) uncertainty about~$\spA$.
It can be seen that \textsc{UnSa} is a special case of \itl when $\spS \subseteq \spA$ and noise is homoscedastic.\looseness=-1

In \cref{fig:gps:directed_advantage2}, We compare \textsc{UnSa} to \itl, \vtl, and \ctl which is a heuristic for TAL that selects: $\vx_{n} = \argmax_{\vx \in \spS} \sum_{\vxp \in \spA} \Cor{f_{\vx},f_{\vxp} \mid \spD_{n-1}}$.
We observe that \itl and \vtl outperform \textsc{UnSa} which also samples points that are not informative about $\spA$.
Further, \itl and \vtl outperform ``local''~\textsc{UnSa}~(i.e., \textsc{UnSa} constrained to $\spA \cap \spS$) which neglects all information provided by points outside~$\spA$.\footnote{If $\spA \not\subseteq \spS$ then ``local''~\textsc{UnSa} does \emph{not even} converge to the irreducible uncertainty.}
As one would expect, \vtl has an advantage with respect to reducing the total variance of $\vfsub{\spA}$, whereas \itl reduces the entropy of~$\vfsub{\spA}$ faster.
We include ablations in \cref{sec:gps_appendix} where we, in particular, observe that the advantage of \itl and \vtl over \textsc{UnSa} increases as the volume of prediction targets shrinks in comparison to the size of domain, supporting our theoretical results in~\cref{thm:variance_convergence}.\looseness=-1

\begin{figure*}[t]
    \centering
    \adjustbox{valign=c}{\includesvg[width=0.75\textwidth]{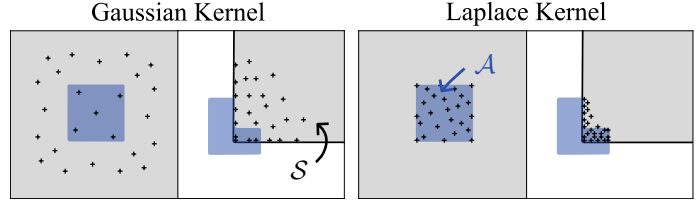}}
    \vspace{0.2cm}
    \caption{Initial $25$ samples of \itl under a Gaussian kernel with lengthscale~$1$~\textbf{(left)} and a Laplace kernel with lengthscale~$10$~\textbf{(right)}. Shown in gray is the sample space $\spS$ and shown in blue is the target space $\spA$. In three of the four examples, points outside the target space provide useful information.\looseness=-1
    }
    \label{fig:gps:directed_advantage}
\end{figure*}

\begin{figure*}[t]
    \centering
    \adjustbox{width=0.7\textwidth,valign=c}{\import{./plots/output/}{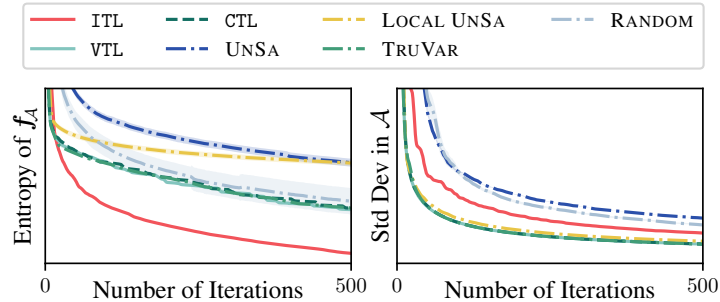}}
    \caption{Entropy of $\vfsub{\spA}$ ranging from $-3850$ to $-3725$ and the mean marginal standard deviations of $\vfsub{\spA}$ ranging from $0$ to $0.15$.
    Experiment is using the Gaussian kernel of the left instance ($\spA \subset \spS$) from~\cref{fig:gps:directed_advantage}.
    It can be seen that \itl and \vtl outperform \textsc{UnSa} and \textsc{Random}.
    Uncertainty bands correspond to one standard error over $10$ random seeds.
    }
    \label{fig:gps:directed_advantage2}
\end{figure*}

%% file: mainmatter/04_nns.tex
\section{Active Fine-Tuning of Neural Networks}\label{sec:nns}

As a first application of transductive active learning, we study automatic data selection for fine-tuning neural networks.
Fine-tuning a large pre-trained model can be a cost- and computation-effective approach to improve performance on a given target domain \citep[e.g.,][]{lee2022surgical}.
While previous work has extensively studied the effectiveness of various training procedures for fine-tuning, we ask: \emph{How can we select the right data for fine-tuning to a specific task?}
This \emph{active} fine-tuning problem is an instance of the ``directed'' transductive learning problem:
Concretely, consider a supervised setting, where the function $f$ maps inputs ${\vx \in \spX}$ to outputs ${y \in \spY}$.
We have access to noisy samples from a training set~$\spS$ on $\spX$, and we would like to learn $f$ such that our estimate minimizes a given risk measure, such as classification error, with respect to a test distribution~$\spPA$ on~$\spX$.
The goal is to actively and efficiently sample from $\spS$ to minimize risk with respect to $\spPA$.\footnote{The setting with target distributions $\spPA$ can be reduced to considering target sets $\spA$~(cf.~\cref{sec:generalizations:roi}).}
We show in this section that TAL can learn $f$ with only \emph{few examples} from $\spS$.\looseness -1%

\paragraph{How can we leverage the latent structure learned by the pre-trained model?}

As common in related work, we approximate the (pre-trained) neural network (NN) $f(\cdot; \vtheta)$ as a linear function in a latent embedding space, $\smash{f(\vx; \vtheta) \approx \transpose{\vbeta} \vphi_{\vtheta}(\vx)}$, with weights $\vbeta \in \R^p$ and embeddings $\vphi_{\vtheta} : \spX \to \R^p$.
Common choices of embeddings include last-layer embeddings \citep{devlin2018bert,holzmuller2023framework}, neural tangent embeddings arising from neural tangent kernels \citep{jacot2018neural} which are motivated by their relationship to the training and fine-tuning of ultra-wide NNs \citep{arora2019exact,lee2019wide,khan2019approximate,he2020bayesian,malladi2023kernel}, and loss gradient embeddings \citep{ash2020deep}.
We provide a comprehensive overview of embeddings in \cref{sec:nns_appendx:embeddings}.
Now, supposing the prior $\vbeta \sim \N{\vzero}{\mSigma}$, often with ${\mSigma = \mI}$~\citep{khan2019approximate,he2020bayesian,antoran2022adapting,wei2022more}, this approximation of $f$ is a Gaussian process with kernel $k(\vx,\vxp) = \transpose{\vphi_{\vtheta}(\vx)} \mSigma \vphi_{\vtheta}(\vxp)$ which quantifies the similarity between points in terms of their alignment in the learned latent space.
In this context, \cref{thm:variance_convergence,lem:confidence_intervals:informal} bound the prediction error of the trained model under the approximation $\smash{\transpose{\vbeta} \vphi_{\vtheta}(\vx)}$.\looseness -1

\paragraph{Batch selection: Diversity via conditional embeddings}

Efficient labeling and training necessitates a batch-wise selection of inputs.
The selection of a batch of size ${b > 1}$ can be seen as an individual \emph{non-adaptive} active learning problem, and recent work has shown that batch diversity is crucial in this setting \citep{ash2020deep,zanette2021design,holzmuller2023framework,pacchiano2024experiment}.
We can formalize a batch-wise selection strategy for TAL by the following non-adaptive optimization problem \citep{chen2013near} which selects elements~$\vx_{n,i}$ of the $n$-th batch iteratively:\looseness=-1 \begin{align}
  \vx_{n,i} = \argmax_{\vx \in \spS} \psi_{\spA}(\vx_{1:n-1,1:b} \cup \vx_{n,1:i-1} \cup \{\vx\}).
  \label{eq:batch_selection}
\end{align}
\vspace{-0.5cm}

\begin{figure*}[t]
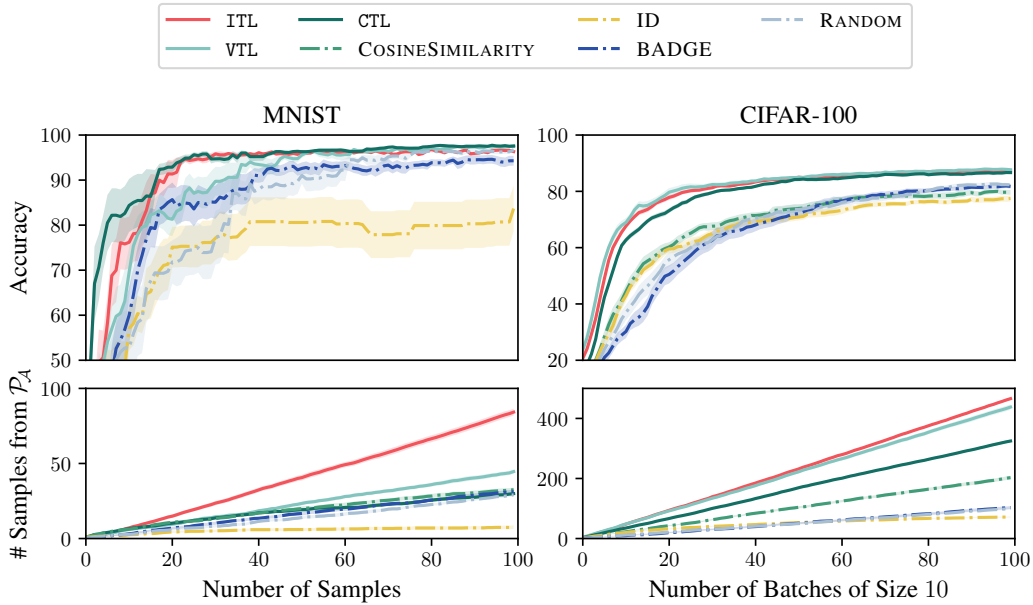

    \vspace{-0.3cm}
    \incplt[\textwidth]{nns}
    \caption{Active fine-tuning on MNIST (left) and CIFAR-100 (right). \textsc{Random} selects each observation uniformly at random from $\spS$. The batch size is $1$ for MNIST and $10$ for CIFAR-100. Uncertainty bands correspond to one standard error over $10$ random seeds. We see that transductive active learning with \itl and \vtl significantly outperforms competing methods, and in particular, retrieves substantially more samples from the support of $\spPA$. See \cref{sec:nns_appendix} for details and ablations.}
    \label{fig:nns}
\end{figure*}

\subsection{Experiments}

Our empirical evaluation is motivated by the following practical example:
We deploy a pre-trained image classifier to user's phones who use it within their local environment.
We would like to locally fine-tune a user's model to their environment.
Since the users' images~$\spA$ are unlabeled, this requires selecting a small number of relevant and diverse images from the set of labeled images~$\spS$.
As such, we will focus here on the setting where the points in our test set do not lie in our training set~(i.e.,~${\spA \cap \spS = \emptyset}$), and discuss alternative instances such as active domain adaptation in \cref{sec:nns_fine_tuning_settings}.\looseness=-1

\paragraph{Testbeds \& architectures}

We use the MNIST \citep{lecun1998mnist} and CIFAR-100 \citep{krizhevsky2009learning} datasets as testbeds.
In both cases, we take $\spS$ to be the training set, and we consider the task of learning the digits $3$, $6$, and $9$ (MNIST) or the first $10$ categories of CIFAR-100.\footnote{That is, we restrict $\spPA$ to the support of points with labels $\{3,6,9\}$ (MNIST) or labels $\{0, \dots, 9\}$ (CIFAR-100) and train a neural network using few examples drawn from the training set $\spS$.}
For MNIST, we train a simple convolutional neural network with ReLU activations, three convolutional layers with max-pooling, and two fully-connected layers.
For CIFAR-100, we fine-tune an EfficientNet-B0 \citep{tan2019efficientnet} pre-trained on ImageNet \citep{deng2009imagenet}, augmented by a final fully-connected layer.
We train the NNs using the cross-entropy loss and the ADAM optimizer \citep{kingma2014adam}.\looseness=-1

\paragraph{Results}

In \cref{fig:nns}, We compare against \textbf{(i)} active learning methods which largely aim for sample diversity but which are not directed towards the target distribution $\spPA$ \citep[e.g., BADGE;][]{ash2020deep}, and \textbf{(ii)} search methods that aim to retrieve the most relevant samples from $\spS$ with respect to the targets $\spPA$ \citep[e.g., maximizing cosine similarity to target embeddings as is common in vector databases;][]{settles2008analysis,johnson2019billion}. \textsc{InformationDensity} \citep[ID,][]{settles2008analysis} is a heuristic approach aiming to combine \textbf{(i)} diversity and \textbf{(ii)} relevance.
In \cref{sec:nns_appendix:undirected}, we also compare against a wide range of additional baselines (e.g., \textsc{CoreSet}~\citep{sener2017active}, \textsc{TypiClust}~\citep{hacohen2022active}, \textsc{ProbCover}~\citep{yehuda2022active}, etc.) that fall into one of the categories \textbf{(i)} and \textbf{(ii)}, and which perform similar to the baselines listed here.

We observe that \itl, \vtl, and \ctl consistently and significantly outperform random sampling from $\spS$ as well as all baselines.
We see that relevance-based methods such as \textsc{CosineSimilarity} have an initial advantage over \textsc{Random} but for batch sizes larger than $1$ they quickly fall behind due to diminishing informativeness of the selected data.
In contrast, diversity-based methods such as \textsc{BADGE} are more competitive with \textsc{Random} but do not explicitly aim to retrieve relevant samples.

Remarkably, transductive active learning outperforms random data selection even in the MNIST experiment where the model is randomly initialized.
This suggests that the learned embeddings can be informative for data selection even in the early stages of training, bootstrapping the learning progress.\looseness=-1

\paragraph{Balancing sample relevance and diversity}

Our proposed methods unify approaches to coverage (promoting \emph{diverse} samples) and search (aiming for \emph{relevant} samples with respect to a given query~$\spA$) which leads to the significant improvement upon the state-of-the-art in \cref{fig:nns}.
Notably, for a batch size and query size of $1$ and if correlations are non-negative, \itl, \vtl, \ctl, and the canonical cosine similarity are equivalent.
\ctl can be seen as a direct generalization of cosine similarity-based retrieval to batch and query sizes larger than one.
In contrast to \ctl, \itl and \vtl may also sample points which exhibit a strong negative correlation (which is also informative).\looseness=-1

We observe empirically that \itl obtains samples from $\spPA$ at more than twice the rate of \textsc{CosineSimilarity}, which translates to a significant improvement in accuracy in more difficult learning tasks, while requiring fewer (labeled) samples from $\spS$.
This phenomenon manifests for both MNIST and CIFAR-100, as well as imbalanced datasets $\spS$ or imbalanced reference samples from $\spPA$ (cf. \cref{sec:nns_appendix:additional_experiments}).
The improvement in accuracy appears to increase in the large-data regime, where the learning tasks become more~difficult.
Akin to a previously identified scaling trend with size of the pre-training dataset \citep{tamkin2022active}, this suggests a potential scaling trend where the improvement of \itl over random batch selection grows as models are fine-tuned on a larger pool of data.\looseness=-1

\paragraph{Towards task-driven few-shot learning}

Being able to efficiently and automatically select data may allow dynamic few-shot fine-tuning to individual tasks~\citep{vinyals2016matching,hardt2024test}, e.g., fine-tuning the model to each test point / query / prompt.
Such task-driven few-shot learning can be seen as a form of ``memory recall'' akin to associative memory~\citep{hopfield1982neural}.
Our results are a first indication that task-driven learning can lead to substantial performance gains, and we believe that this is a promising direction for future studies.\looseness=-1

%% file: mainmatter/05_safe_bo.tex
\section{Safe Bayesian Optimization}\label{sec:safe_bo}

Another practical problem that can be cast as ``directed'' learning is safe Bayesian optimization \citep[Safe BO,][]{sui2015safe, berkenkamp2021bayesian} which has applications in natural science~\citep{cooper2022multidimensional} and robotics~\citep{wischnewski2019a,sukhija2022scalable,widmer2023tuning}.
Safe BO solves the following optimization problem \begin{equation}
    \max_{\vx \in \opt{\spS}} \opt{f}(\vx) \quad\text{where}\quad \opt{\spS} = \{\vx \in \spX \mid \opt{g}(\vx) \geq 0\} \label{eq:safe_bo}
\end{equation} which can be generalized to multiple constraints.
The functions $\opt{f}$ and $\opt{g}$, and hence also the ``safe set'' $\opt{\spS}$, are unknown and have to be actively learned from data.
However, it is crucial that the data collection does not violate the constraint, i.e., ${\vx_n \in \opt{\spS}, \forall n \geq 1}$.\looseness=-1

\paragraph{Safe Bayesian optimization as Transductive Active Learning}

In the non-probabilistic setting from \cref{sec:algorithm}, GPs $f$ and $g$ can be used as well-calibrated models of the ground truths $\opt{f}$ and $\opt{g}$, and we denote lower- and upper-confidence bounds by $l_n^f(\vx), l_n^g(\vx)$ and $u_n^f(\vx), u_n^g(\vx)$, respectively.
These confidence bounds induce a \emph{pessimistic} safe set ${\spS_n \defeq \{\vx \mid l_n^g(\vx) \geq 0\}}$ and an \emph{optimistic} safe set ${\smash{\widehat{\spS}}_n \defeq \{\vx \mid u_n^g(\vx) \geq 0\}}$ which satisfy ${\spS_n \subseteq \opt{\spS} \subseteq \smash{\widehat{\spS}}_n}$ with high probability at all times.
Similarly, the set of \emph{potential maximizers} \begin{align}
    \spA_n \defeq \{\vx \in \widehat{\spS}_n \mid u_n^f(\vx) \geq \max_{\vxp \in \spS_n} l_n^f(\vxp)\} \label{eq:potential_maximizers}
\end{align} contains the solution to \cref{eq:safe_bo} at all times with high probability.\looseness=-1

The (simple) regret ${r_n(\spS) \defeq \max_{\vx \in \spS} \opt{f}(\vx) - \opt{f}(\vxhat_n)}$ with ${\vxhat_n \defeq \argmax_{\vx \in \spS_n} l_n^f(\vx)}$ measures the worst-case performance of a decision rule.
To achieve small regret, one faces an \emph{exploration-expansion} dilemma wherein
one needs to explore points that are known-to-be-safe, i.e., lie in the estimated safe set $\spS_n$, and might be optimal, while at the same time discovering new safe points by ``expanding''~$\spS_n$.
Accordingly, a natural choice for the target space of Safe~BO is~$\spA_n$ since it captures both exploration and expansion \emph{simultaneously}.\footnote{An alternative possibility is to weigh each point in $\spA_n$ according to how likely it is to be the safe optimum. Which approach performs better is task-dependent, and we include a detailed discussion in \cref{sec:safe_bo_appendix:thompson_sampling}.}
To prevent constraint violation, the sample space is restricted to the pessimistic safe set $\spS_n$.
In Safe~BO, both the target space  and sample space change with each round~$n$, and we generalize our theoretical results from \cref{sec:algorithm} in \cref{sec:proofs} to this setting.\looseness=-1

\begin{theorem}[Convergence to safe optimum]\label{thm:safebo_main}
    Pick any ${\epsilon > 0}$, ${\delta \in (0,1)}$.
    Assume that $\opt{f}$, $\opt{g}$ lie in the reproducing kernel Hilbert space $\spH_k(\spX)$ of the kernel $k$, and that the noise $\varepsilon_n$ is conditionally $\rho$-sub-Gaussian.
    Then, we have with probability at least $1-\delta$,
    \begin{enumerate}[align=left, topsep=0pt,noitemsep]
        \item[Safety:] for all $n \geq 1$,\quad $\vx_n \in \opt{\spS}$.
    \end{enumerate}
    Moreover, assume $\spS_0 \neq \emptyset$ and denote with $\spR$ the largest reachable safe set starting from $\spS_0$.
    Then, the convergence of reducible uncertainty implies that there exists $\opt{n} > 0$ such that
    with probability at least $1-\delta$, \begin{enumerate}[align=left, topsep=0pt,noitemsep]
        \item[Optimality:] for all $n \geq \opt{n}$,\quad $r_n(\spR) \leq \epsilon$.
    \end{enumerate}
\end{theorem}

We provide a formal proof in \cref{sec:proofs:safe_bo}.
Central to the proof is the application of \cref{thm:variance_convergence} to show that the safety of parameters \emph{outside} the safe set $\spS_n$ can be inferred efficiently.
In \cref{sec:algorithm}, we outline settings where the reducible uncertainty converges which is the case for a very general class of functions, and for such instances \cref{thm:safebo_main} guarantees optimality in the largest reachable safe set~$\spR$.
$\spR$ represents the largest set any safe learning algorithm can explore without violating the safety constraints (with high probability) during learning~(cf.~\cref{defn:reachable_safe_set}).
Our guarantees are similar to those of other Safe BO algorithms~\citep{berkenkamp2021bayesian} but require fewer assumptions and generalize to continuous domains. %
We obtain \cref{thm:safebo_main} from a more general result (\cref{lem:convergence_to_reachable_safe_region}) which can be specialized to yield ``free'' novel convergence guarantees for problems other than Bayesian optimization, such as level set estimation, by choosing an appropriate target space.\looseness=-1

\subsection{Experiments}

We evaluate two synthetic experiments for a 1d and 2d parameter space, respectively (cf. \cref{sec:safe_bo_appendix:details} for details), which demonstrate the various shortcomings of existing Safe BO baselines.
Additionally, as third experiment, we safely tune the controller of a quadcopter.\looseness=-1

\paragraph{Safe controller tuning for a quadcopter}

We consider a quadcopter with unknown dynamics; ${\vs_{t+1} = \vT(\vs_t, \vu_t)}$ where ${\vu_t \in \R^{d_u}}$ is the control signal and $\vs_t \in \R^{d_s}$ is the state at time $t$.
The inputs $\vu_t$ are calculated through a deterministic function of the state $\vpi : \spS \to \spU$ which we call the policy.
The policy is parameterized via parameters $\vx \in \spX$, e.g., PID controller gains, such that $\vu_t = \vpi_{\vx}(\vs_t)$. %
The goal is to find the optimal parameters with respect to an unknown objective~$\opt{f}$ while satisfying some unknown constraint(s)~${\opt{g}(\vx) \geq 0}$, e.g., the quadcopter does not fall on the ground.
This is a typical Safe BO problem which is widely applied for safe controller learning in robotics~\citep{berkenkamp2021bayesian,baumann2021gosafe,widmer2023tuning}.\looseness=-1

\begin{figure*}[t]
    \vspace{-0.3cm}
    \incplt[\textwidth]{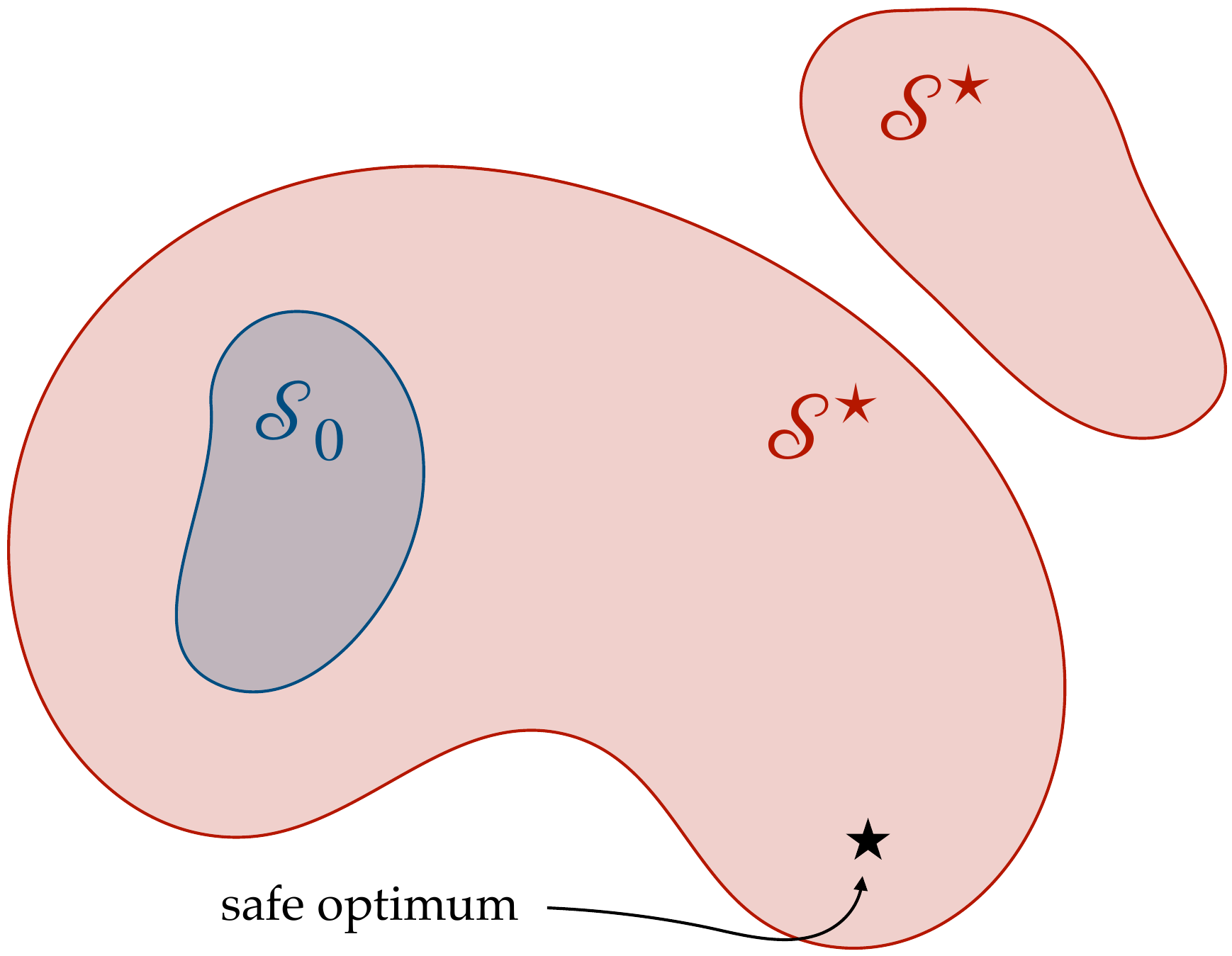}
    \caption{We compare \itl and \vtl to \textsc{Oracle SafeOpt}, which has oracle knowledge of the Lipschitz constants, \textsc{SafeOpt}, where the Lipschitz constants are estimated from the GP, as well as \textsc{Heuristic SafeOpt} and ISE, and observe that \itl and \vtl systematically perform well. We compare against additional baselines in \cref{sec:safe_bo_appendix:thompson_sampling}. The regret is evaluated with respect to the ground truth objective $\opt{f}$ and constraint $\opt{g}$, and averaged over 10 (in synthetic experiments) and 25 (in the quadcopter experiment) random seeds. Additional details can be found in \cref{sec:safe_bo_appendix:details}.}
    \label{fig:safe_bo}
    \vspace{-0.1cm}
\end{figure*}

\paragraph{Results}
We compare \itl and \vtl to \textsc{SafeOpt}~\citep{berkenkamp2021bayesian}, which is undirected, i.e., expands in all directions including ones that are known-to-be suboptimal, and ISE~\citep{bottero2022information}, which is solely expansionist --- does not trade-off expansion-exploration. We provide a detailed discussion of baselines in \cref{sec:safe_bo_appendix:comparison}.
In all our experiments, summarized in \cref{fig:safe_bo}, we observe that \itl and \vtl systematically perform well, i.e., better or on par with the state-of-the-art.
We attribute this to its directed exploration and less conservative expansion over \textsc{Safeopt} (cf.~1d task and quadcopter experiment), and natural trade-off between expansion and exploration as opposed to ISE (see 2d task).
Generally, \vtl has a slight advantage over \itl, which is because \vtl minimizes marginal variances (as opposed to entropy), which are decisive for expanding the safe set.
While \itl and \vtl do not violate constraints, we observe that other methods that do not explicitly enforce safety such as EIC~\citep{gardner2014bayesian} lead to constraint violation (cf. \cref{sec:safe_bo_appendix:details:quadcopter}).\looseness=-1

%% file: mainmatter/06_related_work.tex
\vspace{-2pt}
\section{Related Work}\label{sec:related_work}
\vspace{-2pt}

\paragraph{(Inductive) active learning}

The special case of transductive active learning where ${\spA = \spS = \spX}$ has been widely studied.
We refer to this special instance as \emph{inductive} active learning, since the goal is to extract as much information as possible as opposed to making predictions on a specific target set.

Several works have previously found uncertainty-based decision rules to be useful for inductive active learning~\citep{cohn1993neural,krause2007nonmyopic,guo2007optimistic,krause2008near} and semi-supervised learning~\citep{grandvalet2004semi}.
\vtl was recently re-derived by \cite{shoham2023experimental} along other experimental design criteria under the lens of minimizing risk for inductive one-shot learning in overparameterized models.
Substantial work on active learning has studied entropy-based criteria in \emph{parameter-space}, most notably BALD \citep{mackay1992information,houlsby2011bayesian,gal2017deep,kirsch2019batchbald}, which selects ${\vx_n = \argmax_{\vx \in \spX} \I{\vtheta}{y_{\vx}}[\spD_{n-1}]}$, where $\vtheta$ is the probabilistic parameter vector of a parametric model.
Such methods are inherently inductive in the sense that they do not facilitate learning on specific prediction targets.\looseness=-1

\paragraph{Transductive active learning}

In contrast, TAL operates in \emph{output-space} where it is straightforward to specify prediction targets, and which is computationally easier.
Special cases of \itl when ${\spS = \spX}$ and ${\abs{\spA} = 1}$ have been proposed in the foundational work of \cite{mackay1992information}.
As generalization to larger target spaces, \cite{mackay1992information} proposed the \emph{mean-marginal information gain}, which measures uncertainty by $\mathrm{Unc}_{\!\spA}(\spD) = \sum_{\vxp \in \spA} \H{f_{\vxp}}[\spD]$: \begin{align*}
  \vx_{n+1} = \argmax_{\vx \in \spS} \sum_{\vxp \in \spA} \I{f_{\vxp}}{y_{\vx}}[\spD_{n}] = \argmin_{\vx \in \spX} \sum_{\vxp \in \spA} \log\parentheses*{k_{n}(\vxp, \vxp) - \frac{k_{n}(\vx, \vxp)^2}{k_{n}(\vx, \vx) + \rho^2}}, \tag{\textbf{\mmitl}}
\end{align*} with the second equality holding in the GP setting of \cref{sec:algorithm}.
We derive an analogous version of \cref{thm:variance_convergence} for \mmitl in \cref{sec:interpretations_approximations:mm_itl}.
Similarly to \vtl, \mmitl disregards the mutual dependence of points in the target space $\spA$ and differs from \vtl only in a different weighting of the posterior marginal variances of the prediction targets.
Recently, \cite{smith2023prediction} generalized \mmitl by treating the prediction target as a random variable, and
\cite{kothawade2021similar} and \cite{smith2024making} demonstrated the use of output-space decision rules for image classification tasks in a pre-training context.\looseness=-1

\paragraph{Influence functions}

Influence functions measure the change in a model's prediction when a single data point is removed from the training data~\citep{cook1977detection,koh2017understanding,pruthi2020estimating}.
They have been used for data selection in settings closely related to the transductive active fine-tuning of neural networks proposed in this work~\citep{xia2024less}.
They select data that reduces a first-order Taylor approximation to the test loss after fine-tuning a neural network, which corresponds to maximizing cosine similarity to the prediction targets in a loss-gradient embedding space.
We show in our experiments that transductive active learning can substantially outperform \textsc{CosineSimilarity}.
We attribute this primarily to influence functions implicitly assuming that the influence of selected data adds linearly \citep[i.e., two equally scored data points are expected to doubly improve the model performance,][Section 3.2]{xu2019understanding}.
This assumption does not hold in practice as seen, e.g., by simply duplicating data.
The same limitation applies to the related approach of datamodels~\citep{ilyas2022datamodels}.\looseness=-1

\paragraph{Other work on directed active learning}

Directed active learning methods have been proposed for the problem of determining the optimum of an unknown function, also known as best-arm identification \citep{audibert2010best} or pure exploration bandits \citep{bubeck2009pure}.
Entropy search methods \citep{hennig2012entropy,hernandez2014predictive} are widely used and select ${\vx_n = \argmax_{\vx \in \spX} \I{\vx^*}{y_{\vx}}[\spD_{n-1}]}$ in \emph{input-space} where ${\vx^* = \argmax_{\vx} f_{\vx}}$.
Similarly to \itl, \emph{output-space} entropy search methods \citep{hoffman2015output,wang2017max}, which select ${\vx_n = \argmax_{\vx \in \spX} \I{f^*}{y_{\vx}}[\spD_{n-1}]}$ with ${f^* = \max_{\vx} f_{\vx}}$, are more computationally tractable.
In fact, output-space entropy search is a special case of \itl with a stochastic target space~(cf.~\cref{eq:safe_bo_stochastic_target_space}~in~\cref{sec:safe_bo_appendix:thompson_sampling}).
\cite{bogunovic2016truncated} proposed \textsc{TruVar} in the context of Bayesian optimization and level set estimation.
\textsc{TruVar} is akin to \vtl with a similar notion of ``target space'', but their algorithm and analysis rely on a threshold scheme which requires that ${\spA \subseteq \spS}$.
\cite{fiez2019sequential} introduce the \emph{transductive linear bandit} problem, which is a special case of transductive active learning limited to a linear function class and with the objective of determining the maximum within an initial candidate set.\footnote{The transductive bandit problem can be solved analogously to Safe BO, by maintaining the set $\spA_n$.}
We mention additional more loosely related works in \cref{sec:related_work_additional}.\looseness=-1

%% file: mainmatter/07_conclusion.tex
\vspace{-2pt}
\section{Conclusion}\label{sec:limitations}
\vspace{-2pt}

We investigated the generalization of active learning to settings with concrete prediction targets and/or with limited information due to constrained sample spaces, which we call transductive active learning~(TAL).
TAL provides a flexible framework, applicable also to other domains than were discussed~(such as recommender systems, molecular design, robotics, etc.) by varying the choice of target space and sample space.
Further, we proved novel generalization bounds which indicate that TAL can be more sample-efficient than inductive active learning.
Finally, we demonstrated across broad applications that using TAL to sample \emph{relevant and diverse} points (as opposed to only one of the two) can lead to a substantial improvement upon the state-of-the-art.\looseness=-1

%% file: backmatter/0_definitions.tex
\section{Additional Related Work}\label{sec:related_work_additional}

The general principle of non-active ``transductive learning'' was introduced by \cite{vapnik1982estimation}.
The notion of ``target'' from transductive active learning is akin to the notion of ``task'' in curriculum learning \citep{bengio2009curriculum,graves2017automated,soviany2022curriculum}.
The study of settings where the irreducible uncertainty is zero is related to the study of estimability in experimental design \citep{graybill1961introduction,mutny2022experimental}.
In feature selection, selecting features that maximize information gain with respect to a to-be-predicted label is a standard approach \citep{peng2005feature,vergara2014review,beraha2019feature} which is akin to~\itl~(cf.~\cref{sec:interpretations_approximations}).
The themes of relevance and diversity are also important for efficient in-context learning \citep[e.g.,][]{ye2023compositional,kumari2024end} and data pruning \citep{zheng2023coverage}.
Transductive active learning is complimentary to other learning methodologies, such as semi-supervised learning \citep{gao2020consistency}, self-supervised learning \citep{shwartz2023compress,balestriero2023cookbook}, and meta-learning \citep{kaddour2020probabilistic,rothfuss2023meta}.\looseness=-1

\section{Background}\label{sec:definitions}

\subsection{Information Theory}

Throughout this work, $\log$ denotes the natural logarithm.
Given random vectors $\vx$ and $\vy$, we denote~by \begin{align*}
    \H{\vx} &\defeq \E[p(\vx)]{- \log p(\vx)}, \\
    \H{\vx}[\vy] &\defeq \E[p(\vx, \vy)]{- \log p(\vx \mid \vy)}, \quad\text{and} \\
    \I{\vx}{\vy} &\defeq \H{\vx} - \H{\vx}[\vy]
\end{align*} the (differential) entropy, conditional entropy, and information gain, respectively \citep{cover1999elements}.\footnote{One has to be careful to ensure that $\I{\vx}{\vy}$ exists, i.e., $\abs{\I{\vx}{\vy}} < \infty$. We will assume that this is the case throughout this work. When $\vx$ and $\vy$ are jointly Gaussian, this is satisfied when the noise variance $\rho^2$ is positive.}\looseness=-1

The \emph{multivariate information gain} \citep{murphy2023probabilistic} between random vectors $\vx, \vy, \vz$ is given by \begin{align}
  \I{\vx}{\vy ; \vz} &\defeq \I{\vx}{\vy} - \I{\vx}{\vy}[\vz] \label{eq:multivariate_information_gain2} \\
  &= \I{\vx}{\vy} + \I{\vx}{\vz} - \I{\vx}{\vy, \vz}. \label{eq:multivariate_information_gain}
\end{align}
When $\I{\vx}{\vy ; \vz} \neq 0$ it is said that $\vy$ and $\vz$ \emph{interact} regarding their information about $\vx$.
If the interaction is positive, it is said that the information of $\vz$ about $\vx$ is \emph{redundant} given $\vy$.
Conversely, if the interaction is negative, it is said that the information of $\vz$ about $\vx$ is \emph{synergistic} with $\vy$.
The notion of synergy is akin to the frequentist notion of ``suppressor variables'' in linear regression \citep{das2008algorithms}.\looseness=-1

\subsection{Gaussian Processes}\label{sec:definitions:gps}

The stochastic process $f$ is a Gaussian process (GP, \cite{williams2006gaussian}), denoted ${f \sim \GP{\mu}{k}}$, with mean function $\mu$ and kernel $k$ if for any finite subset ${X = \{\vx_1, \dots, \vx_n\} \subseteq \spX}$, ${\vfsub{X} \sim \N{\vmusub{X}}{\mKsub{XX}}}$ is jointly Gaussian with mean vector ${\vmusub{X}(i) = \mu(\vx_i)}$ and covariance matrix~${\mKsub{XX}(i, j) = k(\vx_i, \vx_j)}$.\looseness=-1

In the following, we formalize the assumptions from the GP setting.\looseness=-1

\begin{assumption}[Gaussian prior]\label{asm:bayesian_prior}
   We assume that ${f \sim \GP{\mu}{k}}$ with known mean function $\mu$ and kernel $k$.
\end{assumption}

\begin{assumption}[Gaussian noise]\label{asm:bayesian_noise}
   We assume that the noise $\varepsilon_{\vx}$ is mutually independent and zero-mean Gaussian with known variance $\rho^2(\vx) > 0$.
   We write ${\mPsub{X} = \diag{\rho^2(\vx_1), \dots, \rho^2(\vx_n)}}$.
\end{assumption}

Under \cref{asm:bayesian_prior,asm:bayesian_noise}, the posterior distribution of~$f$ after observing points $X$ is $\GP{\mu_n}{k_n}$ with \begin{align*}
    \mu_n(\vx) &= \mu(\vx) + \mKsub{\vx X} \inv{(\mKsub{XX} + \mPsub{X})} (\vysub{X} - \vmusub{X}), \\
    k_n(\vx,\vxp) &= k(\vx,\vxp) - \mKsub{\vx X} \inv{(\mKsub{XX} + \mPsub{X})} \mKsub{X\vxp}, \\
    \sigma_n^2(\vx) &= k_n(\vx,\vx).
\end{align*}
For Gaussian random vectors $\vf$ and $\vy$, the entropy is ${\H{\vf} = \frac{n}{2} \log(2 \pi e) + \frac{1}{2} \log \det{\Var{\vf}}}$, the information gain is ${\I{\vf}{\vy} = \frac{1}{2} \parentheses*{\log\det{\Var{\vy}} - \log\det{\Var{\vy \mid \vf}}}}$, and \begin{align*}
    \gamma_n &= \max_{\substack{X \subseteq \spX \\ |X| \leq n}} \frac{1}{2} \log\det{\mI + \inv{\mPsub{X}} \mKsub{XX}}.
\end{align*}

%% file: backmatter/C_proofs.tex
\section{Proofs}\label{sec:proofs}

We will write \begin{itemize}
  \item $\sigma^2 \defeq \max_{\vx \in \spX} \sigma_{0}^2(\vx)$, and
  \item $\tilde{\sigma}^2 \defeq \max_{\vx \in \spX} \sigma_{0}^2(\vx) + \rho^2(\vx)$.
\end{itemize}

The following is a brief overview of the structure of this section: \begin{enumerate}
  \item \Cref{sec:proofs:undirected_itl} relates \itl in the inductive learning setting ($\spS \subseteq \spA$) to prior work.
  \item \Cref{sec:proofs:non_adaptive_data_selection} relates the designs selected by \itl and \vtl to the optimal designs for corresponding non-adaptive objectives.
  \item \Cref{sec:batch_diversity_proof} shows that batch selection via \itl or \vtl leads to informative and diverse batches, utilizing the results from \cref{sec:proofs:non_adaptive_data_selection}.
  \item \Cref{sec:synergies} introduces measures of synergies that generalize the submodularity assumption~(cf.~\cref{asm:submodularity}).
  \item \Cref{sec:proofs:objective_convergence} proves key results on the convergence of the \itl and \vtl objectives.
  \item \Cref{sec:proofs:variance_convergence} proves \cref{thm:variance_convergence} (convergence in GP setting).
  \item \Cref{sec:proofs:width_convergence} proves \cref{thm:width_convergence} (convergence in agnostic setting).
  \item \Cref{sec:proofs:safe_bo} proves \cref{thm:safebo_main} (convergence in safe BO application).
  \item \Cref{sec:proofs:useful_facts} includes useful facts.
\end{enumerate}

\subsection{Undirected Case of \itl}\label{sec:proofs:undirected_itl}

We briefly examine the important special case of \itl where $\spS \subseteq \spA$.
In this setting, for all $\vx \in \spS$, the decision rule of \itl simplifies to \begin{align*}
  \Ism{\vfsub{\spA}}{y_{\vx}}[\spD_n] \eeq{1} \Ism{\vfsub{\spA \setminus \{\vx\}}}{y_{\vx}}[f_{\vx}, \spD_n] + \I{f_{\vx}}{y_{\vx}}[\spD_n] \\
  \eeq{2} \Ism{f_{\vx}}{y_{\vx}}[\spD_n] \\
  &= \Hsm{y_{\vx} \mid \spD_n} - \Hsm{\varepsilon_{\vx}}
\end{align*} where \e{1} follows from the chain rule of information gain and $\vx \in \spS \subseteq \spA$; and \e{2} follows from the conditional independence $\vfsub{\spA} \perp y_{\vx} \mid f_{\vx}$.

If additionally $f$ is a GP then \begin{align*}
    \Hsm{y_{\vx} \mid \spD_n} - \H{\varepsilon_{\vx}} = \frac{1}{2} \log \parentheses*{1 + \frac{\Var{f_{\vx} \mid \spD_n}}{\Var{\varepsilon_{\vx}}}}.
\end{align*}
This decision rule has also been termed \emph{total information gain} \citep{mackay1992information}.
When $\spS \subseteq \spA$ and observation noise is homoscedastic, this decision rule is equivalent to uncertainty sampling.

\subsection{Non-adaptive Data Selection \& Submodularity}\label{sec:proofs:non_adaptive_data_selection}

Recall the non-myopic information gain ${\psi_{\spA}(X) = \Ism{\vfsub{\spA}}{\vysub{X}}}$ (\itl) and variance reduction ${\psi_{\spA}(X) = \tr{\Var{\vfsub{\spA}}} - \tr{\Var{\vfsub{\spA} \mid \vysub{X}}}}$ (\vtl) objective functions from \cref{asm:submodularity}.
In this section, we will relate the designs selected by \itl and \vtl to the optimal designs for these objectives.
To this end, consider the non-adaptive optimization problem \begin{align*}
  \opt{X} = \argmax_{\substack{X \subseteq \spS \\ \abs{X} = k}} \psi_{\spA}(X).
\end{align*}

\begin{lemma}
  For both \itl and \vtl, $\psi_{\spA}$ is non-negative and monotone.
\end{lemma}
\begin{proof}
  For \itl, $\psi_{\spA}(X) \geq 0$ follows from the non-negativity of mutual information.
  To conclude monotonicity, note that for any $X' \subseteq X \subseteq \spS$, \begin{align*}
    \Ism{\vfsub{\spA}}{\vysub{X'}} = \H{\vfsub{\spA}} - \H{\vfsub{\spA}}[\vysub{X'}] \leq \H{\vfsub{\spA}} - \H{\vfsub{\spA}}[\vysub{X}] = \Ism{\vfsub{\spA}}{\vysub{X}}
  \end{align*} due to monotonicity of conditional entropy (which is also called the ``information never hurts'' principle).\looseness=-1

  For \vtl, recall that $\tr{\Var{\vfsub{\spA} \mid \vysub{X}}} \leq \tr{\Var{\vfsub{\spA} \mid \vysub{X'}}}$ for any $X' \subseteq X \subseteq \spS$ (with an implicit expectation over $\vysub{X}, \vysub{X'}$).
  Non-negativity and monotonicity of $\psi_{\spA}$ then follow analogously to~\itl.
\end{proof}

\begin{lemma}
  The marginal gain $\Delta_{\spA}(\vx \mid X) \defeq \psi_{\spA}(X \cup \{\vx\}) - \psi_{\spA}(X)$ of $\vx \in \spS$ given $X \subseteq \spS$ is the \itl and \vtl objective, respectively.
\end{lemma}
\begin{proof}
  For \itl, \begin{align*}
    \Delta_{\spA}(\vx \mid X) &= \I{\vfsub{\spA}}{\vysub{X}, y_{\vx}} - \I{\vfsub{\spA}}{\vysub{X}} \\
    &= \H{\vfsub{\spA} \mid \vysub{X}} - \H{\vfsub{\spA} \mid \vysub{X}, y_{\vx}} \\
    &= \I{\vfsub{\spA}}{y_{\vx} \mid \vysub{X}}
  \end{align*} which is precisely the \itl objective.

  For \vtl, \begin{align*}
    \Delta_{\spA}(\vx \mid X) &= \tr{\Var{\vfsub{\spA} \mid \vysub{X}}} - \tr{\Var{\vfsub{\spA} \mid \vysub{X}, y_{\vx}}} \\
    &= - \tr{\Var{\vfsub{\spA} \mid \vysub{X}, y_{\vx}}} + \const
  \end{align*} which is precisely the \vtl objective.
\end{proof}

\begin{definition}[Submodularity]
  $\psi_{\spA}$ is submodular if and only if for all $\vx \in \spS$ and $X' \subseteq X \subseteq \spS$, \begin{align*}
    \Delta_{\spA}(\vx \mid X') \geq \Delta_{\spA}(\vx \mid X).
  \end{align*}
\end{definition}

\begin{theorem}[\cite{nemhauser1978analysis}]\label{thm:design_guarantee}
  Let \cref{asm:submodularity} hold.
  For any $n \geq 1$, if \itl or \vtl selected $\vx_{1:n}$, respectively, then \vspace{-0.2cm}\begin{align*}
    \psi_{\spA}(\vx_{1:n}) \geq (1-1/e) \max_{\substack{X \subseteq \spS \\ \abs{X} \leq n}} \psi_{\spA}(X).
    \vspace{-0.2cm}
  \end{align*}
\end{theorem}
\begin{proof}
  This is a special case of a canonical result from non-negative monotone submodular function maximization \citep{nemhauser1978analysis,krause2014submodular}.
\end{proof}

\subsection{Batch Diversity: Batch Selection as Non-adaptive Data Selection}\label{sec:batch_diversity_proof}

Recall the non-adaptive optimization problem \begin{align*}
  B_{n,k} = \argmax_{\substack{B \subseteq \spS \\ \abs{B} = k}} \I{\vfsub{\spA}}{\vysub{B}}[\spD_{n-1}]
\end{align*} from \cref{eq:batch_selection} with batch size $k > 0$, and denote by ${B'_{n,k} = \vx_{n,1:k}}$ the greedy approximation from \cref{eq:batch_selection}.
The selection of an individual batch can be seen as a single non-adaptive optimization problem with marginal gain \begin{align*}
  \Delta_n(\vx \mid B) &= \I{\vfsub{\spA}}{\vysub{B}, y_{\vx}}[\spD_{n-1}] - \I{\vfsub{\spA}}{\vysub{B}}[\spD_{n-1}] \\
  &= \H{\vfsub{\spA} \mid \spD_{n-1}, \vysub{B}} - \H{\vfsub{\spA} \mid \spD_{n-1}, \vysub{B}, y_{\vx}} \\
  &= \I{\vfsub{\spA}}{y_{\vx} \mid \spD_{n-1}, \vysub{B}}
\end{align*} and which is precisely the objective function of \itl from \cref{eq:batch_selection}.
Hence, the approximation guarantees from \cref{thm:design_guarantee,thm:approx_design_guarantee} apply.
The derivation is analogous for \vtl.\looseness=-1

Prior work has shown that the greedy solution $B'_n$ is also competitive with a fully sequential ``batchless'' decision rule \citep{chen2013near,esfandiari2021adaptivity}.\looseness=-1

\subsection{Measures of Synergies \& Approximate Submodularity}\label{sec:synergies}

We will now show that ``downstream synergies'', if present, can be seen as a source of learning complexity, which is orthogonal to the information capacity $\gamma_n$.

\begin{example}
  Consider the example where $f$ is a stochastic process of three random variables $X, Y, Z$ where $X$ and~$Y$ are Bernoulli ($p=\frac{1}{2}$), and $Z$ is the XOR of $X$ and~$Y$.
  Suppose that observations are exact (i.e., $\varepsilon_n = 0$), that the target space $\spA$ comprises the output variable $Z$ while the sample space $\spS$ comprises the input variables $X$ and~$Y$.
  Observing any single $X$ or $Y$ yields no information about~$Z$: $\I{Z}{X} = \I{Z}{Y} = 0$, however, observing both inputs jointly perfectly determines $Z$: $\I{Z}{X, Y} = 1$.
  Thus, $\gamma_n(\spA; \spS) = 1$ if $n \geq 2$ and $\gamma_n(\spA; \spS) = 0$ else.
\end{example}\vspace{-0.15cm}

Learning about $Z$ in examples of this kind is difficult for agents that make decisions greedily, since the next action (observing $X$ or $Y$) yields no signal about its long-term usefulness.
We call a sequence of observations, such as~$\{X,Y\}$, \emph{synergistic} since its combined information value is larger than the individual values.
The prevalence of synergies is not captured by the information capacity~$\gamma_n(\spA; \spS)$ since it measures only the joint information gain of $n$~samples within~$\spS$.
Instead, the prevalence of synergies is captured by the sequence ${\Gamma_n \defeq \max_{\vx \in \spS} \Delta_{\spA}(\vx \mid \vx_{1:n})}$, which measures the maximum information gain of $y_{n+1}$.
If $\Gamma_{n} > \Gamma_{n-1}$ at any round $n$, this indicates a synergy.
The following key object measures the additional complexity due to synergies.\looseness=-1

\begin{definition}[Task complexity]\label{defn:task_complexity}
  For $n \geq 1$, assuming ${\Gamma_i > 0}$ for all $1 \leq i \leq n$, we define the \emph{task complexity} as \vspace{-0.1cm}\begin{align*}
      \alpha_{\spA,\spS}(n) \defeq \max_{i \in \{0, \dots, n-1\}} \frac{\Gamma_{n-1}}{\Gamma_i}.
  \end{align*}
\end{definition}\vspace{-0.2cm}

Note that $\alpha_{\spA,\spS}(n)$ is large only if the information gain of $y_{n}$ is larger than that of a previous observation $y_i$.
Intuitively, if $\alpha_{\spA,\spS}(n)$ is large, the agent had to discover the \emph{implicit} intermediate observations $y_{1}, \dots, y_{n-1}$ that lead to downstream synergies.
We will subsequently formalize the intimate connections of the task complexity to synergies and submodularity.
Note that in the GP setting, $\alpha_{\spA,\spS}(n)$ can be computed online by keeping track of the smallest $\Gamma_i$ during previous rounds~$i$.
Further, note that $\alpha_{\spA,\spS}(n) \leq 1$ if $\psi_{\spA}$ is submodular.\looseness=-1

\subsubsection{The Information Ratio}

Another object will prove useful in our analysis of synergies.

Consider an alternative multiplicative interpretation of the multivariate information gain~(cf.~\cref{eq:multivariate_information_gain}), which we call the \emph{information ratio} of ${X \subseteq \spS}$ given $D \subseteq \spS$, ${\abs{X}, \abs{D} < \infty}$: \begin{align}
  \kappabar(X \mid D) \defeq \frac{\sum_{\vx \in X} \Delta_{\spA}(\vx \mid D)}{\Delta_{\spA}(X \mid D)} \in [0, \infty). \label{eq:information_ratio}
\end{align}
Observe that $\kappabar(X \mid D)$ measures the synergy properties of $\vysub{X}$ with respect to $\vfsub{\spA}$ given $\vysub{D}$ in a multiplicative sense.
That is, if ${\kappabar(X \mid D) > 1}$ then information in $\vysub{X}$ is redundant, whereas if ${\kappabar(X \mid D) < 1}$ then information in $\vysub{X}$ is synergistic, and if ${\kappabar(X \mid D) = 1}$ then $\vysub{X}$ do not mutually interact with respect to~$\vfsub{\spA}$ (all given $\vysub{D}$).
In the degenerate case where ${\Delta_{\spA}(X \mid D) = 0}$ (which implies $\sum_{\vx \in X} \Delta_{\spA}(\vx \mid D) = 0$) we therefore let ${\kappabar(X \mid D) = 1}$.\looseness=-1

\paragraph{The information ratio of \itl is strictly positive in the Gaussian case}

We prove the following straightforward lower bound to the information ratio of \itl.\looseness=-1

\begin{lemma}\label{lem:information_ratio_lower_bound}
  Let $X, D \subseteq \spS, \abs{X}, \abs{D} < \infty$.
  If $\vfsub{\spA}$ and $\vysub{X \cup D}$ are jointly Gaussian then ${\kappabar(X \mid D) > 0}$.
\end{lemma}
\begin{proof}
  W.l.o.g. assume ${D = \emptyset}$.
  We let ${X = \{\vx_1, \dots, \vx_k\}}$ and prove lower and upper bound separately.
  We assume w.l.o.g. that ${\I{\vfsub{\spA}}{\vysub{X}} > 0}$ which implies ${\det{\Var{\vfsub{\spA} \mid \vysub{X}}} < \det{\Var{\vfsub{\spA}}}}$.
  Thus, there exists some $i$ such that $\vfsub{\spA}$ and $y_{\vx_i}$ are dependent, so ${\det{\Var{\vfsub{\spA} \mid y_{\vx_i}}} < \det{\Var{\vfsub{\spA}}}}$ which implies ${\I{\vfsub{\spA}}{y_{\vx_i}} > 0}$.
  We therefore conclude that $\kappabar(X) > 0$.
\end{proof}

The following example shows that this lower bound is tight.\looseness=-1

\begin{example}[Synergies of Gaussian random variables, inspired by Section 3 of \cite{barrett2015exploration}]\label{ex:gaussian_synergies}
  Consider the three random variables $X$, $Y$, and $Z$ (think ${\spA = \{X\}}$ and ${\spS = \{Y,Z\}}$) which are jointly Gaussian with mean vector~$\vzero$ and covariance matrix \begin{align*}
    \mSigma = \begin{bmatrix}
      1 & a & a \\
      a & 1 & 0 \\
      a & 0 & 1
    \end{bmatrix}, \qquad \text{for $2 a^2 < 1$}
  \end{align*} where the constraint on $a$ is to ensure that $\mSigma$ is positive definite.
  Computing the mutual information, we have \begin{align*}
    \I{X}{Y} = \I{X}{Z} = -\frac{1}{2}\log(1-a^2)
  \end{align*} and $\I{X}{Y,Z} = -\frac{1}{2}\log(1-2a^2)$.
  Therefore, \begin{align*}
    \frac{\I{X}{Y} + \I{X}{Z}}{\I{X}{Y, Z}} = \frac{\log(1 - 2a^2 + a^4)}{\log(1 - 2a^2)} < 1.
  \end{align*}
  Note that \begin{align*}
    \lim_{a \to \frac{1}{\sqrt{2}}} \frac{\log(1 - 2a^2 + a^4)}{\log(1 - 2a^2)} = 0,
  \end{align*} and hence --- perhaps unintuitively --- even if $Y$ and $Z$ are uncorrelated, their information about $X$ may be arbitrarily synergistic.\looseness=-1
\end{example}

\subsubsection{The Submodularity of the Special ``Undirected'' Case of \itl}

In the inductive active learning problem considered in most prior works, where $\spS \subseteq \spA$ and $f$ is a Gaussian process, it holds for \itl that $\alpha_{\spA,\spS}(n) = 1$ since all learning targets appear \emph{explicitly} in~$\spS$:\looseness=-1

\begin{lemma}\label{lem:information_ratio_monotonicity}
  Let ${\spS \subseteq \spA}$.
  Then $\psi_{\spA}$ of \itl is submodular.
\end{lemma}
\begin{proof}
  Fix any $\vx \in \spS$ and $X' \subseteq X \subseteq \spS$.
  Let $\bar{X} \defeq X \setminus X'$.
  By the definition of conditional information gain, we have \begin{align*}
    \Delta_{\spA}(\vx \mid X) = \Ism{y_{\vx}}{\vfsub{\spA} \mid \vysub{X}} = \Ism{y_{\vx}}{\vfsub{\spA}, \vysub{X'} \mid \vysub{\bar{X}}} - \Ism{y_{\vx}}{\vysub{X'} \mid \vysub{\bar{X}}}.
  \end{align*}
  Since for any $\vx \in \spS$ and $X \subseteq \spS$, ${y_{\vx} \perp \vysub{X} \mid \vfsub{\spA}}$, this simplifies to \begin{align*}
    \Ism{y_{\vx}}{\vfsub{\spA} \mid \vysub{X}} = \Ism{y_{\vx}}{\vfsub{\spA} \mid \vysub{\bar{X}}} - \Ism{y_{\vx}}{\vysub{X'} \mid \vysub{\bar{X}}}.
  \end{align*}
  It then follows from $\Ism{y_{\vx}}{\vysub{X'} \mid \vysub{\bar{X}}} \geq 0$ that \begin{align*}
    \Delta_{\spA}(\vx \mid X) = \Ism{y_{\vx}}{\vfsub{\spA} \mid \vysub{X}} \leq \Ism{y_{\vx}}{\vfsub{\spA} \mid \vysub{\bar{X}}} = \Delta_{\spA}(\vx \mid X').
  \end{align*}
\end{proof}

This implies that $\alpha_{\spA,\spS}(n) \leq 1$ for any $n$ and $\kappabar(X \mid D) \geq 1$ for any $X, D \subseteq \spS$ when $\spS \subseteq \spA$.\looseness=-1

\subsubsection{The Submodularity Ratio}

Building upon the theory of maximizing non-negative monotone submodular functions \citep{nemhauser1978analysis,krause2014submodular}, \cite{das2018approximate} define the following notion of ``approximate'' submodularity:\looseness=-1

\begin{definition}[Submodularity ratio]
  The \emph{submodularity ratio} of $\psi_{\spA}$ up to cardinality $n \geq 1$ is \begin{align}
    \kappa_{\spA}(n) \defeq \min_{\substack{D \subseteq \vx_{1:n} \\ X \subseteq \spS : |X| \leq n \\ D \cap X = \emptyset}} \kappabar(X \mid D), \label{eq:submodularity_ratio}
  \end{align} where they define $\frac{0}{0} \equiv 1$.
  $\psi_{\spA}$ is said to be \emph{$\kappa$-weakly submodular} for some $\kappa > 0$ if ${\inf_{n \in \Nat} \kappa_{\spA}(n) \geq \kappa}$.
\end{definition}

As a special case of Theorem 6 from \cite{das2018approximate}, applying that $\psi_{\spA}$ is non-negative and monotone, we obtain the following result.\looseness=-1

\begin{theorem}[\cite{das2018approximate}]\label{thm:approx_design_guarantee}
  For any $n \geq 1$, if \itl or \vtl selected $\vx_{1:n}$, respectively, then \begin{align*}
    \psi_{\spA}(\vx_{1:n}) \geq (1-e^{-\kappa_{\spA}(n)}) \max_{\substack{X \subseteq \spS \\ \abs{X} \leq n}} \psi_{\spA}(X).
    \vspace{-0.2cm}
  \end{align*}
\end{theorem}

If $\psi_{\spA}$ is submodular, it is implied that $\kappa_{\spA}(n) \geq 1$ for all $n \geq 1$ in which case \cref{thm:approx_design_guarantee} recovers \cref{thm:design_guarantee}.\looseness=-1

\subsection{Convergence of Marginal Gain}\label{sec:proofs:objective_convergence}

Our following analysis allows for changing target spaces $\spA_n$ and sample spaces $\spS_n$ (cf. \cref{sec:safe_bo}), and to this end, we redefine $\Gamma_n \defeq \max_{\vx \in \spS_n} \Delta_{\spA_n}(\vx \mid \vx_{1:n})$.
The following theorems show that the marginal gains of \itl and \vtl converge to zero, and will serve as the main tool for establishing \cref{thm:variance_convergence,thm:width_convergence}.
We will abbreviate $\alpha_{\spA,\spS}(n)$ by $\alpha_n$.

\begin{theorem}[Convergence of Marginal Gain for \itl]\label{thm:objective_convergence_itl}
  Assume that \cref{asm:bayesian_prior,asm:bayesian_noise} are satisfied.
  Fix any integers ${n_1 > n_0 \geq 0}$, ${\Delta = n_1 - n_0 + 1}$ such that for all ${i \in \{n_0, \dots, n_1 - 1\}}$, ${\spA_{i+1} \subseteq \spA_i}$ and ${\spS \defeq \spS_{i+1} = \spS_i}$.
  Further, assume $\abs{\spA_{n_0}} < \infty$.
  Then, if the sequence $\{\vx_{i+1}\}_{i = n_0}^{n_1}$ was generated by \itl, \begin{align}
    \Gamma_{n_1} \leq \alpha_{n_1} \frac{\gamma_{\Delta}}{\Delta}. \label{eq:objective_convergence_appendix}
  \end{align}
  Moreover, if $n_0 = 0$, \begin{align}
    \Gamma_{n_1} \leq \alpha_{n_1} \frac{\gamma_{\spA_0, \spS}(\Delta)}{\Delta}.
  \end{align}
\end{theorem}
\begin{proof}
  We have \begin{align*}
    \Gamma_{n_1} &= \frac{1}{\Delta} \sum_{i=n_0}^{n_1} \Gamma_{n_1} \\
    \eleq{1} \frac{\alpha_{n_1}}{\Delta} \sum_{i=n_0}^{n_1} \Gamma_i \\
    &= \frac{\alpha_{n_1}}{\Delta} \sum_{i=n_0}^{n_1} \max_{\vx \in \spS} \Ism{\vfsub{\spA_{i}}}{y_{\vx}}[\vy_{1:i}] \\
    \eeq{2} \frac{\alpha_{n_1}}{\Delta} \sum_{i=n_0}^{n_1} \Ism{\vfsub{\spA_{i}}}{y_{\vx_{i+1}}}[\spD_{i}] \\
    \eleq{3} \frac{\alpha_{n_1}}{\Delta} \sum_{i=n_0}^{n_1} \Ism{\vfsub{\spA_{n_0}}}{y_{\vx_{i+1}}}[\spD_{i}] \\
    \eeq{4} \frac{\alpha_{n_1}}{\Delta} \sum_{i=n_0}^{n_1} \Ism{\vfsub{\spA_{n_0}}}{y_{\vx_{i+1}}}[\vy_{\vx_{n_0+1:i}}, \spD_{n_0}] \\
    \eeq{5} \frac{\alpha_{n_1}}{\Delta} \Ism{\vfsub{\spA_{n_0}}}{\vy_{\vx_{n_0+1:n_1+1}}}[\spD_{n_0}] \\
    &\leq \frac{\alpha_{n_1}}{\Delta} \max_{\substack{X \subseteq \spS \\ |X| = \Delta}} \Ism{\vfsub{\spA_{n_0}}}{\vysub{X}}[\spD_{n_0}] \\
    \eleq{6} \frac{\alpha_{n_1}}{\Delta} \max_{\substack{X \subseteq \spS \\ |X| = \Delta}} \Ism{\vfsub{X}}{\vysub{X}}[\spD_{n_0}] \\
    \eleq{7} \frac{\alpha_{n_1}}{\Delta} \max_{\substack{X \subseteq \spS \\ |X| = \Delta}} \Ism{\vfsub{X}}{\vysub{X}} \\
    &= \alpha_{n_1} \frac{\gamma_\Delta}{\Delta}
  \end{align*} where \e{1} follows from the definition of the task complexity $\alpha_{n_1}$ (cf.~\cref{defn:task_complexity}); \e{2} uses the objective of \itl and that the posterior variance of Gaussians is independent of the realization and only depends on the \emph{location} of observations; \e{3} uses ${\spA_{i+1} \subseteq \spA_i}$ and monotonicity of information gain; \e{4} uses that the posterior variance of Gaussians is independent of the realization and only depends on the \emph{location} of observations; \e{5} uses the chain rule of information gain; \e{6} uses ${\vysub{X} \perp \vfsub{\spA_{n_0}} \mid \vfsub{X}}$ and the data processing inequality.
  The conditional independence follows from the assumption that the observation noise is independent.
  Similarly, $\vysub{X} \perp \spD_{n_0} \mid \vfsub{X}$ which implies \e{7}.

  If $n_0 = 0$, then the bound before line \e{6} simplifies to $\alpha_{n_1} \gamma_{\spA_0,\spS}(\Delta) / \Delta$.
\end{proof}

The result for \vtl is stated, for simplicity, only for the case where the target space and sample space are fixed.\looseness=-1

\begin{theorem}[Convergence of Marginal Gain for \vtl]\label{thm:objective_convergence_vtl}
  Assume that \cref{asm:bayesian_prior,asm:bayesian_noise} are satisfied.
  Then for any $n \geq 1$, if the sequence $\{\vx_i\}_{i=1}^n$ is generated by \vtl, \begin{align}
  \Gamma_{n-1} &\leq \frac{2 \sigma^2 \alpha_n}{n} \sum_{\vxp \in \spA} \gamma_{\{\vxp\},\spS}(n).
\end{align}
\end{theorem}
We remark that $\sum_{\vxp \in \spA} \gamma_{\{\vxp\},\spS}(n) \leq \abs{\spA} \gamma_{\spA,\spS}(n)$.\looseness=-1
\begin{proof}
  We have \begin{align*}
      \Gamma_{n-1} &= \frac{1}{n} \sum_{i=0}^{n-1} \Gamma_{n-1} \\
      \eleq{1} \frac{\alpha_n}{n} \sum_{i=0}^{n-1} \Gamma_i \\
      &= \begin{multlined}[t]
          \frac{\alpha_n}{n} \sum_{i=0}^{n-1} \Big[ \tr{\Var{\vfsub{\spA} \mid \vy_{1:i}}} - \min_{\vx \in \spS} \tr{\Var{\vfsub{\spA} \mid \vy_{1:i}, y_{\vx}}} \Big]
      \end{multlined} \\
      \eeq{2} \frac{\alpha_n}{n} \sum_{i=0}^{n-1} \Big[ \tr{\Var{\vfsub{\spA} \mid \spD_{i}}} - \tr{\Var{\vfsub{\spA} \mid \spD_{i+1}}} \Big] \\
      \eleq{3} \frac{\sigma^2 \alpha_n}{n} \sum_{\vxp \in \spA} \sum_{i=0}^{n-1} \log\parentheses*{\frac{\Var{f_{\vxp} \mid \spD_{n}}}{\Var{f_{\vxp} \mid \spD_{n+1}}}} \\
      &= \frac{2 \sigma^2 \alpha_n}{n} \sum_{\vxp \in \spA} \sum_{i=0}^{n-1} \I{f_{\vxp}}{y_{\vx_{n+1}}}[\spD_{n}] \\
      \eeq{4} \frac{2 \sigma^2 \alpha_n}{n} \sum_{\vxp \in \spA} \sum_{i=0}^{n-1} \I{f_{\vxp}}{y_{\vx_{n+1}}}[\vysub{\vx_{1:n}}] \\
      \eeq{5} \frac{2 \sigma^2 \alpha_n}{n} \sum_{\vxp \in \spA} \I{f_{\vxp}}{\vysub{\vx_{1:n}}} \\
      &\leq \frac{2 \sigma^2 \alpha_n}{n} \sum_{\vxp \in \spA} \max_{\substack{X \subseteq \spS \\ \abs{X} = n}} \I{f_{\vxp}}{\vysub{X}} \\
      &= \frac{2 \sigma^2 \alpha_n}{n} \sum_{\vxp \in \spA} \gamma_{\{\vxp\},\spS}(n)
  \end{align*} where \e{1} follows from the definition of the task complexity $\alpha_{n_1}$ (cf.~\cref{defn:task_complexity}); \e{2} follows from the \vtl decision rule and that the posterior variance of Gaussians is independent of the realization and only depends on the \emph{location} of observations; \e{3} follows from \cref{lem:difference_bound_by_log} and monotonicity of variance; \e{4} uses that the posterior variance of Gaussians is independent of the realization and only depends on the \emph{location} of observations; and \e{5} uses the chain rule of mutual information.
  The remainder of the proof is analogous to the proof of \cref{thm:objective_convergence_itl} (cf.~\cref{sec:proofs:objective_convergence}).
\end{proof}

\paragraph{Keeping track of the task complexity online}

In general, the task complexity $\alpha_{n}$ may be larger than one in the ``directed'' setting (i.e., when $\spS \not\subseteq \spA$).
However, note that $\alpha_n$ can easily be evaluated online by keeping track of the smallest $\Gamma_i$ during previous rounds $i$.

\subsection{Proof of \cref{thm:variance_convergence}}\label{sec:proofs:variance_convergence}

We will now prove \cref{thm:variance_convergence}.
We first prove the convergence of marginal variance within $\spS$ for \itl, before proving the convergence outside $\spS$ in \cref{sec:proofs:variance_convergence:outside}.

\begin{lemma}[Uniform convergence of marginal variance within $\spS$ for \itl]\label{lem:convergence_within_S}
  Assume that \cref{asm:bayesian_prior,asm:bayesian_noise} are satisfied.
  For any $n \geq 0$ and $\vx \in \spA \cap \spS$, \begin{align}
    \sigma_{n}^2(\vx) \leq 2 \tilde{\sigma}^2 \cdot \Gamma_n.
  \end{align}
\end{lemma}
\begin{proof}
  We have \begin{align*}
    \sigma_{n}^2(\vx) &= \Var{f_{\vx} \mid \spD_n} - \underbrace{\Varsm{f_{\vx} \mid f_{\vx}, \spD_n}}_{0} \\
    \eeq{1} \begin{multlined}[t]
      \Var{y_{\vx} \mid \spD_n} - \rho^2(\vx) - (\Varsm{y_{\vx} \mid f_{\vx}, \spD_n} - \rho^2(\vx))
    \end{multlined} \\
    &= \Var{y_{\vx} \mid \spD_n} - \Varsm{y_{\vx} \mid f_{\vx}, \spD_n} \\
    \eleq{2} \tilde{\sigma}^2 \log\parentheses*{\frac{\Var{y_{\vx} \mid \spD_n}}{\Varsm{y_{\vx} \mid f_{\vx}, \spD_n}}} \\
    &= 2 \tilde{\sigma}^2 \cdot \Ism{f_{\vx}}{y_{\vx}}[\spD_n] \\
    \eleq{3} 2 \tilde{\sigma}^2 \cdot \Ism{\vfsub{\spA}}{y_{\vx}}[\spD_n] \\
    \eleq{4} 2 \tilde{\sigma}^2 \cdot \max_{\vxp \in \spS} \Ism{\vfsub{\spA}}{y_{\vxp}}[\spD_n] \\
    \eeq{5} 2 \tilde{\sigma}^2 \cdot \max_{\vxp \in \spS} \Ism{\vfsub{\spA}}{y_{\vxp}}[\vy_{1:n}] \\
    &= 2 \tilde{\sigma}^2 \cdot \Gamma_n
  \end{align*} where \e{1} follows from the noise assumption (cf.~\cref{asm:bayesian_noise}); \e{2} follows from \cref{lem:difference_bound_by_log} and using monotonicity of variance; \e{3} follows from $\vx \in \spA$ and monotonicity of information gain; \e{4} follows from ${\vx \in \spS}$; and \e{5} uses that the posterior variance of Gaussians is independent of the realization and only depends on the \emph{location} of observations.
\end{proof}

\subsubsection{Convergence outside $\spS$ for \itl}\label{sec:proofs:variance_convergence:outside}

We will now show convergence of marginal variance to the irreducible uncertainty for points outside the sample space.

Our proof roughly proceeds as follows:
We construct an ``approximate Markov boundary'' of $\vx$ in $\spS$, and show (1) that the size of this Markov boundary is independent of $n$, and (2) that a small uncertainty reduction within the Markov boundary implies that the marginal variances at the Markov boundary and(!) $\vx$ are small.

\begin{definition}[Approximate Markov boundary]\label{defn:approx_markov_boundary}
  For any $\epsilon > 0$, $n \geq 0$, and $\vx \in \spX$, we denote by $B_{n,\epsilon}(\vx)$ the smallest (multi-)subset of $\spS$ such that \begin{align}
    \Varsm{f_{\vx} \mid \spD_n, \vysub{B_{n,\epsilon}(\vx)}} &\leq \eta_{\spS}^2(\vx) + \epsilon. \label{eq:approx_markov_boundary}
  \end{align}
  We call $B_{n,\epsilon}(\vx)$ an \emph{$\epsilon$-approximate Markov boundary} of $\vx$ in~$\spS$.
\end{definition}

\Cref{eq:approx_markov_boundary} is akin to the notion of the smallest Markov blanket in $\spS$ of some $\vx \in \spX$ (called a \emph{Markov boundary}) which is the smallest set $\spB \subseteq \spS$ such that $f_{\vx} \perp \vfsub{\spS} \mid \vfsub{\spB}$.

\begin{lemma}[Existence of an approximate Markov boundary]\label{lem:approx_markov_boundary}
  For any $\epsilon > 0$, let $k$ be the smallest integer satisfying \begin{align}
    \frac{\gamma_k}{k} \leq \frac{\epsilon \lambda_{\min}(\mKsub{\spS\spS})}{2 \abs{\spS} \sigma^2 \tilde{\sigma}^2}. \label{eq:markov_boundary_size_condition}
  \end{align}
  Then, for any $n \geq 0$ and $\vx \in \spX$, there exists an $\epsilon$-approximate Markov boundary $B_{n,\epsilon}(\vx)$ of $\vx$ in $\spS$ with size at most $k$.
\end{lemma}
\Cref{lem:approx_markov_boundary} shows that for any $\epsilon > 0$ there exists a universal constant $b_\epsilon$ (with respect to $n$ and $\vx$) such that \begin{align}
  \abs{B_{n,\epsilon}(\vx)} \leq b_\epsilon \qquad \forall n \geq 0, \vx \in \spX.
\end{align}
We defer the proof of \cref{lem:approx_markov_boundary} to \cref{sec:missing_proofs:existence_of_an_approximate_markov_boundary} where we also provide an algorithm to compute $B_{n,\epsilon}(\vx)$.

\begin{lemma}\label{lem:convergence_helper1}
  For any $\epsilon > 0$, $n \geq 0$, and $\vx \in \spX$, \begin{align}
    \sigma_{n}^2(\vx) \leq 2 \sigma^2 \cdot \Ism{f_{\vx}}{\vysub{B_{n,\epsilon}(\vx)}}[\spD_n] + \eta_{\spS}^2(\vx) + \epsilon
  \end{align} where $B_{n,\epsilon}(\vx)$ is an $\epsilon$-approximate Markov boundary of $\vx$ in $\spS$.
\end{lemma}
\begin{proof}
  We have \begin{align*}
    \sigma_{n}^2(\vx) &= \begin{multlined}[t]
      \Var{f_{\vx} \mid \spD_n} - \eta_{\spS}^2(\vx) + \eta_{\spS}^2(\vx)
    \end{multlined} \\
    \eleq{1} \begin{multlined}[t]
      \Var{f_{\vx} \mid \spD_n} - \Varsm{f_{\vx} \mid \vysub{B_{n,\epsilon}(\vx)}, \spD_n} + \eta_{\spS}^2(\vx) + \epsilon
    \end{multlined} \\
    \eleq{2} \begin{multlined}[t]
      \sigma^2 \log\parentheses*{\frac{\Var{f_{\vx} \mid \spD_n}}{\Varsm{f_{\vx} \mid \vysub{B_{n,\epsilon}(\vx)}, \spD_n}}} + \eta_{\spS}^2(\vx) + \epsilon
    \end{multlined} \\
    &= 2 \sigma^2 \cdot \Ism{f_{\vx}}{\vysub{B_{n,\epsilon}(\vx)}}[\spD_n] + \eta_{\spS}^2(\vx) + \epsilon
  \end{align*} where \e{1} follows from the defining property of an $\epsilon$-approximate Markov boundary (cf.~\cref{eq:approx_markov_boundary}); and \e{2} follows from \cref{lem:difference_bound_by_log} and using monotonicity of variance.
\end{proof}

\begin{lemma}\label{lem:convergence_helper2}
  For any $\epsilon > 0, n \geq 0$, and $\vx \in \spA$, \begin{align}
    \Ism{f_{\vx}}{\vysub{B_{n,\epsilon}(\vx)}}[\spD_n] \leq \frac{b_{\epsilon}}{\kappabar_{n}(B_{n,\epsilon}(\vx))} \Gamma_n
  \end{align} where $B_{n,\epsilon}(\vx)$ is an $\epsilon$-approximate Markov boundary of $\vx$ in $\spS$, $\abs{B_{n,\epsilon}(\vx)} \leq b_\epsilon$, and where $\kappabar_{n}(\cdot) \defeq \kappabar(\cdot \mid \vx_{1:n})$ denotes the information ratio from \cref{eq:information_ratio}.
\end{lemma}
We remark that $\kappabar_{n}(\cdot) > 0$ as is shown in \cref{lem:information_ratio_lower_bound}, and hence, the right-hand side of the inequality is well-defined.
\begin{proof}
  We use the abbreviated notation ${B = B_{n,\epsilon}(\vx)}$.
  We have \begin{align*}
    \Ism{f_{\vx}}{\vysub{B}}[\spD_n] \eleq{1} \Ism{\vfsub{\spA}}{\vysub{B}}[\spD_n] \\
    \eleq{2} \frac{1}{\kappabar_{n,b_\epsilon}} \sum_{\tilde{\vx} \in B} \Ism{\vfsub{\spA}}{y_{\tilde{\vx}}}[\spD_n] \\
    \eleq{3} \frac{b_{\epsilon}}{\kappabar_{n,b_\epsilon}} \max_{\tilde{\vx} \in B} \Ism{\vfsub{\spA}}{y_{\tilde{\vx}}}[\spD_n] \\
    \eleq{4} \frac{b_{\epsilon}}{\kappabar_{n,b_\epsilon}} \max_{\tilde{\vx} \in \spS} \Ism{\vfsub{\spA}}{y_{\tilde{\vx}}}[\spD_n] \\
    \eeq{5} \frac{b_{\epsilon}}{\kappabar_{n,b_\epsilon}} \max_{\tilde{\vx} \in \spS} \Ism{\vfsub{\spA}}{y_{\tilde{\vx}}}[\vy_{1:n}] \\
    &= \frac{b_{\epsilon}}{\kappabar_{n,b_\epsilon}} \Gamma_n
  \end{align*} where \e{1} follows from monotonicity of mutual information; \e{2} follows from the definition of the information ratio $\kappabar_{n,b_\epsilon}$ (cf.~\cref{eq:information_ratio}); \e{3} follows from $b \leq b_\epsilon$; \e{4} follows from $B \subseteq \spS$; and \e{5} uses that the posterior variance of Gaussians is independent of the realization and only depends on the \emph{location} of observations.
\end{proof}

\begin{proof}[Proof of \cref{thm:variance_convergence} for \itl]
  The case where $\vx \in \spA \cap \spS$ is shown by \cref{lem:convergence_within_S} with $C = 2 \tilde{\sigma}^2$.

  To prove the more general result, fix any ${\vx \in \spA}$ and ${\epsilon > 0}$.
  By \cref{lem:approx_markov_boundary}, there exists an $\epsilon$-approximate Markov boundary $B_{n,\epsilon}(\vx)$ of $\vx$ in $\spS$ such that $\abs{B_{n,\epsilon}(\vx)} \leq b_\epsilon$.
  We have \begin{align*}
    \sigma_{n}^2(\vx) \eleq{1} 2 \sigma^2 \cdot \Ism{f_{\vx}}{\vysub{B_{n,\epsilon}(\vx)}}[\spD_{n}] + \eta_{\spS}^2(\vx) + \epsilon \\
    \eleq{2} \frac{2 \sigma^2 b_{\epsilon}}{\kappabar_{n}(B_{n,\epsilon}(\vx))} \Gamma_n + \eta_{\spS}^2(\vx) + \epsilon
  \end{align*} where \e{1} follows from \cref{lem:convergence_helper1}; and \e{2} follows from \cref{lem:convergence_helper2}.

  Let $\smash{\epsilon = c \frac{\gamma_{\sqrt{n}}}{\sqrt{n}}}$ with ${c = 2 \abs{\spS} \sigma^2 \tilde{\sigma}^2 / \lambda_{\min}(\mKsub{\spS \spS})}$.
  Then, by \cref{eq:markov_boundary_size_condition}, $b_\epsilon$ can be bounded for instance by $\sqrt{n}$.
  Together with \cref{thm:objective_convergence_itl} this implies for \itl that \begin{align*}
    \sigma_{n}^2(\vx) &\leq \eta_{\spS}^2(\vx) + 2 \sigma^2 \sqrt{n} \, \Gamma_n + c \gamma_{\sqrt{n}} / \sqrt{n} \\
    &\leq \eta_{\spS}^2(\vx) + c' \gamma_n / \sqrt{n}
  \end{align*} for a constant $c'$, e.g., $c' = 2 \sigma^2 + c$.
\end{proof}

\subsubsection{Convergence outside $\spS$ for \vtl}\label{sec:proofs:variance_convergence:outside_vtl}

\begin{proof}[Proof of \cref{thm:variance_convergence} for \vtl]
  Analogously to \cref{lem:convergence_helper1}, we have \begin{align*}
      \sigma_{n}^2(\vx) &= \begin{multlined}[t]
        \Var{f_{\vx} \mid \spD_n} - \eta_{\spS}^2(\vx) + \eta_{\spS}^2(\vx)
      \end{multlined} \\
      \eleq{1} \begin{multlined}[t]
        \Var{f_{\vx} \mid \spD_n} - \Varsm{f_{\vx} \mid \vysub{B_{n,\epsilon}(\vx)}, \spD_n} + \eta_{\spS}^2(\vx) + \epsilon
      \end{multlined}
  \end{align*} where \e{1} follows from the defining property of an $\epsilon$-approximate Markov boundary (cf.~\cref{eq:approx_markov_boundary}).
  Further, we have \begin{align*}
      &\Var{f_{\vx} \mid \spD_n} - \Varsm{f_{\vx} \mid \vysub{B_{n,\epsilon}(\vx)}, \spD_n} \\
      \eleq{1} \sum_{\tilde{\vx} \in B_{n,\epsilon}(\vx)} \parentheses*{\Var{f_{\vx} \mid \spD_n} - \Varsm{f_{\vx} \mid y_{\tilde{\vx}}, \spD_n}} \\
      \eleq{2} \sum_{\tilde{\vx} \in B_{n,\epsilon}(\vx)} \parentheses*{\tr{\Var{\vfsub{\spA} \mid \vy_{1:n}}} - \tr{\Varsm{\vfsub{\spA} \mid y_{\tilde{\vx}}, \vy_{1:n}}}} \\
      \eleq{3} b_\epsilon \Gamma_n
  \end{align*} where \e{1} follows from the submodularity of $\psi_{\spA}$; \e{2} uses that the posterior variance of Gaussians is independent of the realization and only depends on the \emph{location} of observations; and \e{3} follows from the definition of $\Gamma_n$ and \cref{lem:approx_markov_boundary}.

  The remainder of the proof is analogous to the result for \itl, using \cref{thm:objective_convergence_vtl} to bound $\Gamma_n$.
\end{proof}

\subsubsection{Existence of an Approximate Markov Boundary}\label{sec:missing_proofs:existence_of_an_approximate_markov_boundary}

We now derive \cref{lem:approx_markov_boundary} which shows the existence of an approximate Markov boundary of~$\vx$ in~$\spS$.

\begin{lemma}\label{lem:learn_markov_boundary}
  For any ${S \subseteq \spS}$ and $k \geq 0$, there exists ${B \subseteq S}$ with ${|B| = k}$ such that for all ${\vxp \in S}$, \begin{align}
    \Var{f_{\vxp} \mid \vysub{B}} \leq 2 \tilde{\sigma}^2 \frac{\gamma_k}{k}.
  \end{align}
\end{lemma}
\begin{proof}
  We choose $B \subseteq S$ greedily using the acquisition function \begin{align*}
    \tilde{\vx}_{k} \defeq \argmax_{\tilde{\vx} \in S} \Ism{\vfsub{S}}{y_{\tilde{\vx}}}[\vysub{B_{k-1}}]
  \end{align*} where $B_k = \tilde{\vx}_{1:k}$.
  Note that this is the ``undirected'' special case of \itl, and hence, we have \begin{align*}
    \Var{f_{\vxp} \mid \vysub{B_k}} \eleq{1} 2 \tilde{\sigma}^2 \Gamma_k \\
    \eleq{2} 2 \tilde{\sigma}^2 \frac{\gamma_k}{k}
  \end{align*} where \e{1} is due to \cref{lem:convergence_within_S}; and \e{2} is due to \cref{thm:objective_convergence_itl} and $\alpha_{S,S}(k) \leq 1$.
\end{proof}

\begin{lemma}\label{lem:approx_markov_boundary_property}
  Given any $\epsilon > 0$ and $B \subseteq S \subseteq \spS$ with $\abs{S} < \infty$, such that for any $\vxp \in S$, \begin{align}
    \Var{f_{\vxp} \mid \vysub{B}} \leq \frac{\epsilon \lambda_{\min}(\mKsub{S S})}{\abs{S} \sigma^2}. \label{eq:approximation_condition}
  \end{align}
  Then for any ${\vx \in \spX}$, \begin{align}
    \Var{f_{\vx} \mid \vysub{B}} \leq \Var{f_{\vx} \mid \vfsub{S}} + \epsilon.
  \end{align}
\end{lemma}
\begin{proof}
  We will denote the right-hand side of \cref{eq:approximation_condition} by $\epsilon'$.
  We have \begin{align*}
    &\Var{f_{\vx} \mid \vysub{B}} \\
    \eeq{1} \begin{multlined}[t]
      \E[\vfsub{S}]{\Var[f_{\vx}]{f_{\vx} \mid \vfsub{S}, \vysub{B}} \mid \vysub{B}} \\ + \Var[\vfsub{S}]{\E[f_{\vx}]{f_{\vx} \mid \vfsub{S}, \vysub{B}} \mid \vysub{B}}
    \end{multlined} \\
    \eeq{2} \Var[f_{\vx}]{f_{\vx} \mid \vfsub{S}, \vysub{B}} + \Var[\vfsub{S}]{\E[f_{\vx}]{f_{\vx} \mid \vfsub{S}, \vysub{B}} \mid \vysub{B}} \\
    \eeq{3} \underbrace{\Var[f_{\vx}]{f_{\vx} \mid \vfsub{S}}}_{\text{irreducible uncertainty}} + \underbrace{\Var[\vfsub{S}]{\E[f_{\vx}]{f_{\vx} \mid \vfsub{S}} \mid \vysub{B}}}_{\text{reducible (epistemic) uncertainty}}
  \end{align*} where \e{1} follows from the law of total variance; \e{2} uses that the conditional variance of a Gaussian depends only on the location of observations and not on their value; and \e{3} follows from $f_{\vx} \perp \vysub{B} \mid \vfsub{S}$ since $B \subseteq S$.
  It remains to bound the reducible uncertainty.

  Let $h_{\vx} : \R^d \to \R, \; \vfsub{S} \mapsto \E{f_{\vx} \mid \vfsub{S}}$ where we write $d \defeq \abs{S}$.
  Using the formula for the GP posterior mean, we have \begin{align*}
    h_{\vx}(\vfsub{S}) = \E{f_{\vx}} + \transpose{\vz} (\vfsub{S} - \E{\vfsub{S}})
  \end{align*} where $\vz \defeq \inv{\mKsub{S S}} \mKsub{S \vx}$.
  Because $h$ is a linear function in $\vfsub{S}$ we have for the reducible uncertainty that \begin{align*}
    \Var[\vfsub{S}]{h_{\vx}(\vfsub{S}) \mid \vysub{B}} &= \transpose{\vz} \Var{\vfsub{S} \mid \vysub{B}} \vz \\
    \eleq{1} d \cdot \transpose{\vz} \diag{\Var{\vfsub{S} \mid \vysub{B}}} \vz \\
    \eleq{2} \epsilon' d \; \transpose{\vz} \vz \\
    &= \epsilon' d \; \mKsub{\vx S} \inv{\mKsub{S S}} \inv{\mKsub{S S}} \mKsub{S \vx} \\
    &\leq \frac{\epsilon' d}{\lambda_{\min}(\mKsub{S S})} \mKsub{\vx S} \inv{\mKsub{S S}} \mKsub{S \vx} \\
    \eleq{3} \frac{\epsilon' d \sigma^2}{\lambda_{\min}(\mKsub{S S})}
  \end{align*} where \e{1} follows from \cref{lem:qf_upper_bound_}; \e{2} follows from the assumption that $\Var{f_{\vxp} \mid \vysub{B}} \leq \epsilon'$ for all ${\vxp \in S}$; and \e{3} follows from \begin{align*}
    \mKsub{\vx S} \inv{\mKsub{S S}} \mKsub{S \vx} \leq \mKsub{\vx\vx} = \sigma^2
  \end{align*} since $\mKsub{\vx\vx} - \mKsub{\vx S} \inv{\mKsub{S S}} \mKsub{S \vx} \geq 0$.
\end{proof}

\begin{proof}[Proof of \cref{lem:approx_markov_boundary}]
  Let $B \subseteq \spS$ be the set of size $k$ generated by \cref{lem:learn_markov_boundary} to satisfy ${\Var{f_{\vxp} \mid \vysub{B}} \leq 2 \tilde{\sigma}^2 \gamma_k / k}$ for all $\vxp \in \spS$.
  We have for any ${\vx \in \spX}$, \begin{align*}
    \Var{f_{\vx} \mid \spD_n, \vysub{B}} \eleq{1} \Var{f_{\vx} \mid \vysub{B}} \\
    \eleq{2} \Var{f_{\vx} \mid \vfsub{\spS}} + \epsilon
  \end{align*} where \e{1} follows from monotonicity of variance; and \e{2} follows from \cref{lem:approx_markov_boundary_property}; using $\abs{\spS} < \infty$ and the condition on $k$.
\end{proof}

We remark that \cref{lem:learn_markov_boundary} provides an algorithm (just ``undirected'' \itl!) to compute an approximate Markov boundary, and the set $B$ returned by this algorithm is a valid approximate Markov boundary for all $\vx \in \spX$.
One can simply swap-in \itl with target space $\{\vx\}$ for ``undirected'' \itl to obtain tighter (but instance-dependent) bounds on the size of the approximate Markov boundary.

\subsubsection{Generalization to Continuous $\spS$ for Finite Dimensional RKHSs}\label{sec:proofs:variance_convergence_generalization}

In this subsection we generalize \cref{thm:variance_convergence} to continuous sample spaces~$\spS$.
We will make the following assumption:

\begin{assumption}\label{asm:finite_dimensional_rkhs}
  The RKHS of the kernel $k$ is finite dimensional.
  In other words, the kernel $k$ can be expressed as ${k(\vx, \vx') = \transpose{\vphi(\vx)} \vphi(\vx')}$ for some feature map ${\vphi : \spX \to \R^d}$ with ${d < \infty}$.
\end{assumption}

In the following, we will denote the design matrix of the sample space $\spS$ by ${\mPhi \defeq \transpose{[\vphi(\vx) : \vx \in \spS]} \in \R^{\abs{\spS} \times d}}$, and we denote by $\mPi_{\mPhi}$ its orthogonal projection onto the orthogonal complement of the span of~$\mPhi$.
In particular, it holds that \begin{enumerate}
  \item $\mPi_{\mPhi} \vv = \vzero$ for all $\vv \in \spn{\mPhi}$, and
  \item $\mPi_{\mPhi} \vv = \vv$ for all $\vv \in (\spn{\mPhi})^{\perp}$.
\end{enumerate}
Especially, ${\vv \in \ker{\mPi_{\mPhi}}}$ if and only if ${\vv \in \spn{\mPhi}}$.
This projection can be computed as follows:

\begin{lemma}\label{lem:projection_onto_span}
  It holds that \begin{align}
    \mPi_{\mPhi} = \mI - \transpose{\mPhi} \inv{(\mPhi \transpose{\mPhi})} \mPhi.
  \end{align}
\end{lemma}
\begin{proof}
  $\transpose{\mPhi} \inv{(\mPhi \transpose{\mPhi})} \mPhi$ is the orthogonal projection onto the span of $\mPhi$ \citep[see, e.g.,][page 211]{strang2016introduction}.
\end{proof}

\begin{lemma}\label{lem:irreducible_uncertainty_linear_kernel}
  Under \cref{asm:finite_dimensional_rkhs}, the irreducible uncertainty $\eta_{\spS}^2(\vx)$ of $\vx \in \spX$ is \begin{align}
    \eta_{\spS}^2(\vx) = \norm{\vphi(\vx)}_{\mPi_{\mPhi}}^2
  \end{align} where ${\norm{\vv}_{\mA} = \sqrt{\transpose{\vv} \mA \vv}}$ denotes the Mahalanobis distance.
\end{lemma}
\begin{proof}
  This is an immediate consequence of the formula for the conditional variance of multivariate Gaussians~(cf.~\cref{sec:definitions:gps}), applied to the linear kernel.
\end{proof}

\Cref{lem:projection_onto_span,lem:irreducible_uncertainty_linear_kernel} imply that $\eta_{\spS}^2(\vx^\parallel) = 0$ for all $\vx^\parallel \in \spX$ with $\vphi(\vx^\parallel) \in \spn{\mPhi}$.
That is, the irreducible uncertainty is zero for points in the span of the sample space.
In contrast, for points $\vx^\perp$ with $\vphi(\vx^\perp) \in (\spn{\mPhi})^{\perp}$, the irreducible uncertainty equals the initial uncertainty: $\eta_{\spS}^2(\vx^\perp) = \sigma_0^2(\vx^\perp)$.
The irreducible uncertainty of any other point $\vx$ can be computed by simple decomposition of $\vphi(\vx)$ into parallel and orthogonal components.

Assuming that \cref{asm:finite_dimensional_rkhs} holds and given any (non-finite) $\spS \subseteq \spX$, there exists a basis $\Omega_{\spS} \subseteq \spX$ in the space of embeddings $\vphi(\cdot)$ such that ${\spn{\spS} = \spn{\Omega_{\spS}}}$ and $\abs{\Omega_{\spS}} \leq d$.
The generalized existence of an approximate Markov boundary for continuous domains can then be shown analogously to \cref{lem:approx_markov_boundary}:

\begin{lemma}[Existence of an approximate Markov boundary for a continuous domain]
  Let $\spS$ be any (continuous) subset of $\spX$ and let \cref{asm:finite_dimensional_rkhs} hold with $d < \infty$.
  Further, for any $\epsilon > 0$, let $k$ be the smallest integer satisfying \begin{align}
    \frac{\gamma_k}{k} \leq \frac{\epsilon \lambda_{\min}(\mKsub{\Omega_{\spS}\Omega_{\spS}})}{2 d \sigma^2 \tilde{\sigma}^2}.
  \end{align}
  Then, for any $n \geq 0$ and $\vx \in \spX$, there exists an $\epsilon$-approximate Markov boundary $B_{n,\epsilon}(\vx)$ of $\vx$ in $\spS$ with size at most $k$.
\end{lemma}
\begin{proof}[Proof sketch]
  The proof follows analogously to \cref{lem:approx_markov_boundary} by conditioning on the finite set $\Omega_{\spS}$ as opposed to $\spS$.
\end{proof}

\subsection{Proof of Convergence in Non-probabilistic Setting}\label{sec:proofs:width_convergence}

We first formalize the assumptions of the non-probabilistic setting:

\begin{assumption}[Regularity of $\opt{f}$]\label{asm:rkhs}
    We assume that $\opt{f}$ is in a reproducing kernel Hilbert space $\spH_k(\spX)$ associated with a kernel $k$ and has bounded norm, that is, $\norm{\opt{f}}_k \leq B$ for some finite $B \in \R$.
\end{assumption}

\begin{assumption}[Sub-Gaussian noise]\label{asm:noise}
    We further assume that each $\varepsilon_{n}$ from the noise sequence $\{\varepsilon_n\}_{n=1}^\infty$ is conditionally zero-mean $\rho(\vx_n)$-sub-Gaussian with known constants ${\rho(\vx) > 0}$ for all $\vx \in \spX$.
    Concretely, \begin{align*}
      \forall n \geq 1, \lambda \in \R : \quad \E{e^{\lambda \epsilon_{n}}}[\spD_{n-1}] \leq \exp\parentheses*{\frac{\lambda^2 \rho^2(\vx_n)}{2}}
    \end{align*} where $\spD_{n-1}$ corresponds to the $\sigma$-algebra generated by the random variables $\{\vx_i,\epsilon_i\}_{i=1}^{n-1}$ and $\vx_n$.
\end{assumption}

We make use of the following foundational result, showing that under the above two assumptions the (misspecified) Gaussian process model from \cref{sec:definitions:gps} is an all-time well-calibrated model of $\opt{f}$:

\begin{lemma}[Well-calibrated confidence intervals; \cite{abbasi2013online,chowdhury2017kernelized}]\label{lem:confidence_intervals}
    Pick ${\delta \in (0,1)}$ and let \cref{asm:rkhs,asm:noise} hold.
    Let \begin{align*}
      \beta_{n}(\delta) = \norm{\opt{f}}_k + \rho \sqrt{2(\gamma_{n} + 1 + \log(1 / \delta))}
    \end{align*} where $\rho = \max_{\vx \in \spX} \rho(\vx)$.\footnote{$\beta_{n}(\delta)$ can be tightened adaptively \citep{emmenegger2023likelihood}.}
    Then, for all $\vx \in \spX$ and $n \geq 0$ jointly with probability at least $1-\delta$, \begin{align*}
        |\opt{f}(\vx) - \mu_{n}(\vx)| \leq \beta_{n}(\delta) \cdot \sigma_{n}(\vx)
    \end{align*} where $\mu_{n}(\vx)$ and $\sigma_{n}^2(\vx)$ are mean and variance (as defined in \cref{sec:definitions:gps}) of the GP posterior of $f(\vx)$ conditional on the observations $\spD_n$, pretending that $\varepsilon_i$ is Gaussian with variance $\rho^2(\vx_i)$.
\end{lemma}

With this result, we apply \cref{thm:variance_convergence} to prove convergence of the width of the confidence intervals at prediction targets:

\begin{theorem}[Bound on approximation error for \itl and \vtl]\label{thm:width_convergence}
    Let \cref{asm:submodularity} hold and the data be selected by either \itl or \vtl.
    Pick any $\delta \in (0,1)$.
    Assume that $\opt{f}$ lies in the reproducing kernel Hilbert space $\spH_k(\spX)$ of the kernel $k$ with norm $\norm{\opt{f}}_k < \infty$, the noise $\varepsilon_n$ is conditionally $\rho$-sub-Gaussian, and $\gamma_n$ is sublinear in $n$.
    Let $\smash{\beta_{n}(\delta) = \norm{\opt{f}}_k + \rho \sqrt{2(\gamma_{n} + 1 + \log(1 / \delta))}}$.
    Then for any $n \geq 1$ and $\vx \in \spA$, jointly with probability at least $1-\delta$, \vspace{-0.1cm}\begin{align*}
        |\opt{f}(\vx) - \mu_{n}(\vx)| \leq \beta_{n}(\delta) \Big[ \underbrace{\eta_{\spS}(\vx)}_{\text{irreducible}} + \underbrace{\nu_{\spA, \spS}(n)}_{\text{reducible}} \Big]
        \vspace{-0.1cm}
    \end{align*} where $\nu_{\spA, \spS}^2(n)$ denotes the reducible part of \cref{eq:variance_convergence:extrapolation}.
\end{theorem}

\begin{proof}[Proof of \cref{thm:width_convergence}]
  By \cref{thm:variance_convergence}, we have that for all ${\vx \in \spA}$, \begin{align*}
    \sigma_n(\vx) &\leq \sqrt{\eta_{\spS}^2(\vx) + \nu_{\spA,\spS}^2(n)} \leq \eta_{\spS}(\vx) + \nu_{\spA,\spS}(n).
  \end{align*}
  The result then follows by application of \cref{lem:confidence_intervals}.
\end{proof}

\subsection{Proof of \cref{thm:safebo_main}}\label{sec:proofs:safe_bo}

In this section, we derive our main result on Safe BO.
In \cref{sec:proofs:convergence_to_reachable_safe_set}, we give the definition of the reachable safe set $\spR$ and derive the conditions under which convergence to the reachable safe set is guaranteed.
Then, in \cref{sec:proofs:convergence_to_optimum}, we prove \cref{thm:safebo_main}.

\paragraph{Notation}

In the agnostic setting from \cref{sec:proofs:width_convergence} (i.e., under \cref{asm:rkhs,asm:noise}), \cref{lem:confidence_intervals} provides us with the following $(1-\delta)$-confidence intervals (CIs) \begin{align}
    \spC_{n}(\vx) \defeq \spC_{n-1}(\vx) \cap [\mu_n(\vx) \pm \beta_n(\delta) \cdot \sigma_n(\vx)] \label{eq:confidence_interval}
\end{align} where ${\spC_{-1}(\vx) = \R}$.
We write ${u_n(\vx) \defeq \max \spC_n(\vx)}$, ${l_n(\vx) \defeq \min \spC_n(\vx)}$, and ${w_n(\vx) \defeq u_n(\vx) - l_n(\vx)}$ for its upper bound, lower bound, and width, respectively.

We learn separate statistical models $f$ and $\{g_1, \dots, g_q\}$ for the ground truth objective $\opt{f}$ and ground truth constraints $\{\opt{g}_1, \dots, \opt{g}_q\}$.
We write $\spI \defeq \{f, 1, \dots, q\}$ and collect the constraints in $\spI_s \defeq \{1, \dots, q\}$.
Without loss of generality, we assume that the confidence intervals include the ground truths with probability at least $1-\delta$ jointly for all $i \in \spI$.\footnote{This can be achieved by taking a union bound and rescaling $\delta$.}
For $i \in \spI$, denote by $u_{n,i}, l_{n,i}, w_{n,i}, \eta_i, \beta_{n,i}$ the respective quantities.
In the following, we do not explicitly denote the dependence of $\beta_n$ on~$\delta$.

To improve clarity, we will refer to the set of potential maximizers defined in \cref{eq:potential_maximizers} as $\spM_n$ and denote by $\spA_n$ an arbitrary target space.

We point out the following corollary:

\begin{corollary}[Safety]\label{lem:safety}
  With high probability, jointly for any $n \geq 0$ and any $i \in \spI_s$, \begin{align}
    \forall \vx \in \spS_n : \opt{g}_i(\vx) \geq 0.
  \end{align}
\end{corollary}

\subsubsection{Convergence to Reachable Safe Set}\label{sec:proofs:convergence_to_reachable_safe_set}

\begin{definition}[Reachable safe set]\label{defn:reachable_safe_set}
  Given any pessimistic safe set $\spS \subseteq \spX$ and any $\epsilon \geq 0$ and $\beta \geq 0$, we define the \emph{reachable safe set} up to $(\epsilon,\beta)$-slack and its closure as \begin{align*}
    \spR_{\epsilon,\beta}(\spS) &\defeq \begin{multlined}[t]
      \spS \cup \{\vx \in \spX \setminus \spS \mid \\ \!\!\!\!\!\text{$\opt{g}_i(\vx) - \beta (\eta_i(\vx; \spS) + \epsilon) \geq 0$ for all $i \in \spI_s$}\}
    \end{multlined} \\
    \bar{\spR}_{\epsilon,\beta}(\spS) &\defeq \lim_{n\to\infty} (\spR_{\epsilon,\beta})^n(\spS)
  \end{align*} where $(\spR_{\epsilon,\beta})^n$ denotes the $n$-th composition of $\spR_{\epsilon,\beta}$ with itself.
\end{definition}

\begin{remark}
  Convergence of the safe set to the closure of the reachability operator can only be guaranteed for finite safe sets ($\abs{\opt{\spS}} < \infty$).
  The following proofs readily generalize to continuous domains by considering convergence within the $k$-th composition of the reachability operator with itself for some $k < \infty$.
  In this case the sample complexity grows with $k$ rather than $\abs{\opt{\spS}}$.
  The only required modification is to lift the assumption of \cref{thm:objective_convergence_itl} that information is gained only while safe sets remain constant (i.e., $\spS_{i+1} = \spS_i$ for all $i$).
  This assumption is straightforward to lift since for any $n \geq 0$ and $T \geq 1$, \begin{align*}
    \max_{\vx \in \spS_n} \Delta_{\spA}(\vx \mid \vx_{1:n+T}) \leq \frac{1}{T} \sum_{t=1}^T \max_{\vx \in \spS_{n}} \Delta_{\spA}(\vx \mid \vx_{1:n+t}) \leq \frac{1}{T} \sum_{t=1}^T \max_{\vx \in \spS_{n+t}} \Delta_{\spA}(\vx \mid \vx_{1:n+t}) \leq \frac{\gamma_T}{T},
  \end{align*} using submodularity for the first inequality and the monotonicity of the safe set for the second inequality.
  In particular, this shows that one continues learning about points in the original safe set --- even as the safe set grows.
\end{remark}

We denote by $\spS_0$ the initial pessimistic safe set induced by the (prior) statistical model $g$ (cf.~\cref{sec:safe_bo}) and write $\bar{\spR}_{\epsilon,\beta} \defeq \bar{\spR}_{\epsilon,\beta}(\spS_0)$.

\begin{lemma}[Properties of the reachable safe set]
  For all $\spS, \spS' \subseteq \spX$, $\epsilon \geq 0$, and $\beta \geq 0$:
  \begin{lemenum}
    \item\label{lem:safe_region_prop1} $\spS' \subseteq \spS \implies \spR_{\epsilon,\beta}(\spS') \subseteq \spR_{\epsilon,\beta}(\spS)$,
    \item\label{lem:safe_region_prop2} $\spR_{\epsilon,\beta}(\spS) \subseteq \spS \implies \bar{\spR}_{\epsilon,\beta}(\spS) \subseteq \spS$, and
    \item $\spR_{0,0}(\emptyset) = \bar{\spR}_{0,0} = \opt{\spS}$.
  \end{lemenum}
\end{lemma}
\begin{proof}[Proof (adapted from lemma 7.1 of \cite{berkenkamp2021bayesian})]
  \leavevmode\begin{enumerate}
    \item Let ${\vx \in \spR_{\epsilon,\beta}(\spS')}$.
    If $\vx \in \spS$ then ${\vx \in \spR_{\epsilon,\beta}(\spS)}$, so let $\vx \not\in \spS$.
    Then, by definition, for all ${i \in \spI_s}$, ${\opt{f}_i(\vx) - \beta \eta_i(\vx; \spS') - \epsilon \geq 0}$.
    By the monotonicity of variance, ${\eta_i(\vx; \spS') \geq \eta_i(\vx; \spS)}$ for all $i \in \spI$, and hence ${\opt{f}_i(\vx) - \beta \eta_i(\vx; \spS) - \epsilon \geq 0}$ for all $i \in \spI_s$.
    It follows that $\vx \in \spR_{\epsilon,\beta}(\spS)$.
    \item By the monotonicity of variance, ${\eta_i(\vx; \spR_{\epsilon,\beta}(\spS)) \geq \eta_i(\vx; \spS)}$ for all ${\vx \in \spX}$ and $i \in \spI$.
    Thus, by definition of the safe region, we have that ${\spR_{\epsilon,\beta}(\spR_{\epsilon,\beta}(\spS)) \subseteq \spS}$.
    The result follows by taking the limit.
    \item The result follows directly from the definition of the true safe set $\opt{\spS}$ (cf.~\cref{eq:safe_bo}). \qedhere
  \end{enumerate}
\end{proof}

Clearly, we cannot expand the safe set beyond $\bar{\spR}_{0,0}$.
The following is our main intermediate result, showing that either we expand the safe set at some point or the uncertainty converges to the irreducible uncertainty.

\begin{lemma}\label{lem:main}
  Given any $n_0 \geq 0$, $\epsilon > 0$, let $n'$ be the smallest integer such that $\nu_{n',\tilde{\epsilon}^2} \leq \tilde{\epsilon}$ where $\tilde{\epsilon} = \epsilon / 2$.
  Let ${\beta_{n_0+n'} = \max_{i \in \spI_s} \beta_{n_0+n',i}}$.
  Assume that the sequence of target spaces is monotonically decreasing, i.e., $\spA_{n+1} \subseteq \spA_n$.
  Then, we have with high probability (at least) one of \begin{align*}
    \begin{multlined}[t]
      \Big(\forall \vx \in \spA_{n_0+n'}, \; \forall i \in \spI : \\ w_{n_0+n',i}(\vx) \leq \beta_{n_0+n'} \brackets*{\eta_{i}(\vx; \spS_{n_0+n'}) + \epsilon} \\
      \text{and}\quad \spA_{n_0+n'} \cap \spR_{\epsilon,\beta_{n_0+n'}}(\spS_{n_0+n'}) \subseteq \spS_{n_0+n'}\Big)
    \end{multlined}
  \end{align*} or $|\spS_{n_0+n'+1}| > |\spS_{n_0}|$.
\end{lemma}
\begin{proof}
  Suppose that $|\spS_{n_0+n'+1}| = |\spS_{n_0}|$.
  Then, by \cref{thm:width_convergence} (using that the sequence of target spaces is monotonically decreasing), for any $\vx \in \spA_{n_0+n'}$ and $i \in \spI$, \begin{align*}
    w_{n_0+n',i}(\vx) \leq \beta_{n_0+n'} \brackets*{\eta_i(\vx; \spS_{n_0+n'}) + \epsilon}.
  \end{align*}
  As $\spS_{n_0+n'+1} = \spS_{n_0+n'}$ we have for all ${\vx \in \spA_{n_0+n'} \setminus \spS_{n_0+n'}}$ and $i \in \spI_s$, with high probability that \begin{align*}
    0 > l_{n_0+n',i}(\vx) &\geq \opt{g}_i(\vx) - w_{n_0+n',i}(\vx) \\
    &\geq \opt{g}_i(\vx) - \beta_{n_0+n'} \brackets*{\eta_i(\vx; \spS_{n_0+n'}) + \epsilon}.
  \end{align*}
  It follows that $\spA_{n_0+n'} \cap \spR_{\epsilon,\beta_{n_0+n'}}(\spS_{n_0+n'}) \subseteq \spS_{n_0+n'}$.
\end{proof}

To gather more intuition about the above lemma, consider the target space \begin{align}
  \spE_n \defeq \widehat{\spS}_n \setminus \spS_n. \label{eq:potential_expanders}
\end{align}
We call $\spE_n$ the \emph{potential expanders} since it contains all points which might be safe, but are not yet known to be safe.
Under this target space, the above lemma simplifies slightly:

\begin{lemma}\label{lem:convergence_to_reachable_safe_region_with_expanders}
  For any $n \geq 0$ and $\epsilon, \beta \geq 0$, if ${\spE_n \subseteq \spA_n}$ then with high probability, \begin{align*}
    \spS_n \cup (\spA_n \cap \spR_{\epsilon,\beta}(\spS_n)) = \spR_{\epsilon,\beta}(\spS_n).
  \end{align*}
\end{lemma}
\begin{proof}
  With high probability, $\spR_{\epsilon,\beta}(\spS_n) \subseteq \widehat{\spS}_n = \spS_n \cup \spE_n$.
  The lemma is a direct consequence.
\end{proof}

The above lemmas can be combined to yield our main result of this subsection, establishing the convergence of \itl to the reachable safe set.

\begin{theorem}[Convergence to reachable safe set]\label{lem:convergence_to_reachable_safe_region}
  For any $\epsilon > 0$, let $n'$ be the smallest integer satisfying the condition of \cref{lem:main}, and define $\opt{n} \defeq (|\opt{\spS}| + 1) n'$.
  Let $\bar{\beta}_{\opt{n}} \geq \beta_{n,i}$ for all $n \leq \opt{n}, i \in \spI_s$.
  Assume that the sequence of target spaces is monotonically decreasing, i.e., $\spA_{n+1} \subseteq \spA_n$.
  Then, the following inequalities hold jointly with probability at least $1-\delta$: \begin{thmenum}
    \item $\forall n \geq 0, \; \forall i \in \spI_s : \opt{g}_i(\vx_n) \geq 0$, {\flushright\hfill\footnotesize{safety}}
    \item\label{lem:convergence_to_reachable_safe_region_2} $\spA_{\opt{n}} \cap \bar{\spR}_{\epsilon,\bar{\beta}_{\opt{n}}} \subseteq \spS_{\opt{n}} \subseteq \bar{\spR}_{0,0} = \opt{\spS}$, {\flushright\hfill\footnotesize{convergence to safe region}}
    \item ${\forall \vx \in \spA_{\opt{n}}, \; \forall i \in \spI : w_{\opt{n},i}(\vx) \leq \bar{\beta}_{\opt{n}} \eta_i(\vx; \bar{\spR}_{\epsilon,\bar{\beta}_{\opt{n}}}) + \epsilon}$, {\flushright\hfill\footnotesize{convergence of width}}
    \item $\forall \vx \in \bar{\spR}_{\epsilon,\bar{\beta}_{\opt{n}}}, \; \forall i \in \spI : \eta_{i}(\vx; \bar{\spR}_{\epsilon,\bar{\beta}_{\opt{n}}}) = 0$. {\flushright\hfill\footnotesize{convergence of width within safe region}}
  \end{thmenum}
\end{theorem}
\begin{proof}
  \e{1} is a direct consequence of \cref{lem:safety}.
  ${\spS_{\opt{n}} \subseteq \opt{\spS}}$ follows directly from the pessimistic safe set $\spS_{\opt{n}}$ from \e{2} being a subset of the true safe set $\opt{\spS}$.
  \e{4} follows directly from the definition of irreducible uncertainty.
  Thus, it remains to establish $\spA_{\opt{n}} \cap \bar{\spR}_{\epsilon,\bar{\beta}_{\opt{n}}} \subseteq \spS_{\opt{n}}$ and \e{3}.

  Recall that with high probability $|\spS_n| \in [0, |\opt{\spS}|]$ for all $n \geq 0$.
  Thus, the size of the pessimistic safe set can increase at most $|\opt{\spS}|$ many times.
  By \cref{lem:main}, using the assumption on $n'$, the size of the pessimistic safe set increases at least once every $n'$ iterations, or else: \begin{align}\begin{split}
    &\forall \vx \in \spA_{n_0+n'}, \; \forall i \in \spI : w_{n_0+n',i}(\vx) \leq \beta_{n_0+n'} \brackets*{\eta_{i}(\vx; \spS_{n_0+n'}) + \epsilon} \\
    &\text{and}\quad \spA_{n_0+n'} \cap \spR_{\epsilon,\beta_{n_0+n'}}(\spS_{n_0+n'}) \subseteq \spS_{n_0+n'}.
  \end{split}\label{eq:main_helper}\end{align}
  Because the safe set can expand at most $|\opt{\spS}|$ many times, \cref{eq:main_helper} occurs eventually for some $n_0 \leq |\opt{\spS}| n'$.
  In this case, since $\bar{\beta}_{\opt{n}}~\geq~\beta_{n_0+n'}$ and $\spA_{\opt{n}} \subseteq \spA_{n_0+n'}$ (as $n_0 + n' \leq \opt{n}$) we have that \begin{align*}
    \spA_{\opt{n}} \cap \spR_{\epsilon,\bar{\beta}_{\opt{n}}}(\spS_{n_0+n'}) &\subseteq \spA_{n_0+n'} \cap \spR_{\epsilon,\beta_{n_0+n'}}(\spS_{n_0+n'}) \\
    &\subseteq \spS_{n_0+n'}.
  \end{align*}
  By \cref{lem:safe_region_prop2}, this implies \begin{align*}
    \spA_{\opt{n}} \cap \bar{\spR}_{\epsilon,\bar{\beta}_{\opt{n}}} \subseteq \spS_{n_0+n'} \subseteq \spS_{\opt{n}}.
  \end{align*}
\end{proof}

We emphasize that \cref{lem:convergence_to_reachable_safe_region} holds for arbitrary target spaces $\spA_n$.
If additionally, $\spE_n \subseteq \spA_n$ for all $n \geq 0$ then by \cref{lem:convergence_to_reachable_safe_region_with_expanders}, \cref{lem:convergence_to_reachable_safe_region_2} strengthens to ${\bar{\spR}_{\epsilon,\bar{\beta}_{\opt{n}}} \subseteq \spS_{\opt{n}}}$.
Intuitively, $\spE_n \subseteq \spA_n$ ensures that one aims to expand the safe set in \emph{all} directions.
Conversely, if $\spE_n \not\subseteq \spA_n$ then one aims only to expand the safe set in the direction of $\spA_n$ (or not at all if $\spA_n \subseteq \spS_n$).

\paragraph{``Free'' convergence guarantees in many applications}

\Cref{lem:convergence_to_reachable_safe_region} can be specialized to yield convergence guarantees in various settings by choosing an appropriate target space $\spA_n$.
Straightforward application of \cref{lem:convergence_to_reachable_safe_region} (informally) requires that the sequence of target spaces is monotonically decreasing (i.e., $\spA_{n+1} \subseteq \spA_n$), and that each target space $\spA_n$ is an ``over-approximation'' of the actual set of targeted points (such as the set of optimas in the Bayesian optimization setting).
We discuss two such applications in the following.
\begin{enumerate}
  \item \emph{Pure expansion:} For example, for the target space $\spE_n$, \cref{lem:convergence_to_reachable_safe_region} bounds the convergence of the safe set to the reachable safe set.
  In this case, the transductive active learning problem corresponds to the ``pure expansion'' setting, also addressed by the ISE baseline discussed in \cref{sec:safe_bo}.
  The ISE baseline, however, does not establish convergence guarantees of the kind of \cref{lem:convergence_to_reachable_safe_region}.
  Note that $\spE_n$ satisfies the (informal) requirements laid out previously, since it is monotonically decreasing by definition, and with high probability, any point $\vx \in \opt{\spS}$ that is not in $\spS_n$ is contained within $\spE_n$.

  \item \emph{Level set estimation:} Given any $\tau \in \R$, we denote the (safe) $\tau$-level set of $\opt{f}$ by ${\spL^\tau \defeq \{\vx \in \opt{\spS} \mid \opt{f}(\vx) = \tau\}}$.
  We define the \emph{potential level set} as \begin{align}
    \spL_n^\tau \defeq \{\vx \in \widehat{\spS}_n \mid \text{$l_{n}^f(\vx) \leq \tau \leq u_{n}^f(\vx)$}\}.
  \end{align}
  That is, $\spL_n^\tau$ is the subset of the optimistic safe set $\widehat{\spS}_n$ where the $\tau$-level set of $\opt{f}$ may be located.
  Analogously to the potential expanders, it is straightforward to show that $\spL_n^\tau$ over-approximates the true $\tau$-level set and is monotonically decreasing.
\end{enumerate}

We remark that our guarantees from this section also apply to the standard (``unsafe'') setting where $\opt{\spS} = \spS_0 = \spX$.

\subsubsection{Convergence to Safe Optimum}\label{sec:proofs:convergence_to_optimum}

In this section, we specialize \cref{lem:convergence_to_reachable_safe_region} for the case that the target space contains the potential maximizers $\spM_n$ (cf.~\cref{eq:potential_maximizers}).
It is straightforward to see that the sequence $\spM_n$ is monotonically decreasing (i.e., $\spM_{n+1} \subseteq \spM_n$).
The following lemma shows that the potential maximizers over-approximate the set of safe maxima ${\spX^* \defeq \argmax_{\vx \in \opt{\spS}} \opt{f}(\vx)}$.

\begin{lemma}[Potential maximizers over-approximate safe maxima]
  For all $n \geq 0$ and with probability at least $1-\delta$, \begin{lemenum}
    \item $\vx \in \spX^*$ implies $\vx \in \spM_n$ and
    \item $\vx \not\in \spM_n$ implies $\vx \not\in \spX^*$.
  \end{lemenum}
\end{lemma}
\begin{proof}
  If $\vx \not\in \spM_n$ then \begin{align*}
    u_{n,f}(\vx) < \max_{\vxp \in \spS_n} l_{n,f}(\vxp) \leq \max_{\vxp \in \opt{\spS}} l_{n,f}(\vxp)
  \end{align*} where we used $\spS_n \subseteq \opt{\spS}$ with high probability, which directly implies with high probability that $\vx \not\in \spX^*$.

  For the other direction, if $\vx \in \spX^*$ then \begin{align*}
    u_{n,f}(\vx) \geq \max_{\vxp \in \opt{\spS}} l_{n,f}(\vxp) \geq \max_{\vxp \in \spS_n} l_{n,f}(\vxp)
  \end{align*} with high probability.
\end{proof}

We denote the set of optimal actions which are safe up to $(\epsilon,\beta)$-slack by \begin{align*}
  \spX^*_{\epsilon,\beta} \defeq \argmax_{\vx \in \bar{\spR}_{\epsilon,\beta}} \opt{f}(\vx),
\end{align*} and by $f^*_{\epsilon,\beta}$ the maximum value attained by $\opt{f}$ at any of the points in $\spX^*_{\epsilon,\beta}$.
The regret can be expressed as \begin{align*}
  r_n(\bar{\spR}_{\epsilon,\beta}) = f^*_{\epsilon,\beta} - \opt{f}(\widehat{\vx}_n)
\end{align*}
The following theorem formalizes \cref{thm:safebo_main} and establishes convergence to the safe optimum.

\begin{theorem}[Convergence to safe optimum]
  For any $\epsilon > 0$, let $n'$ be the smallest integer satisfying the condition of \cref{lem:main}, and define $\opt{n} \defeq (|\opt{\spS}| + 1) n'$.
  Let $\bar{\beta}_{\opt{n}} \geq \beta_{n,i}$ for all $n \leq \opt{n}, i \in \spI_s$.
  Then, the following inequalities hold jointly with probability at least $1-\delta$: \begin{thmenum}
    \item $\forall n \geq 0, \; \forall i \in \spI_s : \opt{g}_i(\vx_n) \geq 0$, {\flushright\hfill\footnotesize{safety}}
    \item $\forall n \geq \opt{n} : r_n(\bar{\spR}_{\epsilon,\bar{\beta}_{\opt{n}}}) \leq \epsilon$. {\flushright\hfill\footnotesize{convergence to safe optimum}}
  \end{thmenum}
\end{theorem}
\begin{proof}

  Fix any $\vx^* \in \spX^*_{\epsilon,\bar{\beta}_{\opt{n}}} \subseteq \bar{\spR}_{\epsilon,\bar{\beta}_{\opt{n}}}$.
  Assume w.l.o.g. that ${\vx^* \in \spM_{\opt{n}}}$.\footnote{Otherwise, with high probability, ${\opt{f}(\widehat{\vx}_{\opt{n}}) > f^*_{\epsilon,\bar{\beta}_{\opt{n}}}}$.}
  Then, with high probability, \begin{align*}
    f^*_{\epsilon,\bar{\beta}_{\opt{n}}} = \opt{f}(\vx^*) &\leq u_{\opt{n},f}(\vx^*) \\
    &= l_{\opt{n},f}(\vx^*) + w_{\opt{n},f}(\vx^*) \\
    \eleq{1} l_{\opt{n},f}(\widehat{\vx}_{\opt{n}}) + w_{\opt{n},f}(\vx^*) \\
    &\leq \opt{f}(\widehat{\vx}_{\opt{n}}) + w_{\opt{n},f}(\vx^*) \\
    \eleq{2} \opt{f}(\widehat{\vx}_{\opt{n}}) + \epsilon
  \end{align*} where \e{1} follows from the definition of $\widehat{\vx}_n$; and \e{2} follows from \cref{lem:convergence_to_reachable_safe_region} and noting that $\vx^* \in \spM_{\opt{n}} \cap \bar{\spR}_{\epsilon,\bar{\beta}_{\opt{n}}}$.

  We have shown that ${\opt{f}(\widehat{\vx}_{\opt{n}}) \geq f^*_{\epsilon,\bar{\beta}_{\opt{n}}} - \epsilon}$, which implies $r_{\opt{n}}(\bar{\spR}_{\epsilon,\bar{\beta}_{\opt{n}}}) \leq \epsilon$.
  Since the upper- and lower-confidence bounds are monotonically decreasing / increasing, respectively, we have that for all $n \geq \opt{n}$, $r_n(\bar{\spR}_{\epsilon,\bar{\beta}_{\opt{n}}}) \leq \epsilon$.
\end{proof}

\subsection{Useful Facts and Inequalities}\label{sec:proofs:useful_facts}

We denote by $\preceq$ the Loewner partial ordering of symmetric matrices.

\begin{lemma}\label{lem:qf_upper_bound_}
  Let $\mA \in \R^{n \times n}$ be a positive definite matrix with diagonal $\mD$.
  Then, $\mA \preceq n \mD$.
\end{lemma}
\begin{proof}
  Equivalently, one can show $n \mD - \mA \succeq \mzero$.
  We write $\mA \eqdef \msqrt{\mD}\mQ\msqrt{\mD}$, and thus, $\mQ = \mD^{-\nicefrac{1}{2}} \mA \mD^{-\nicefrac{1}{2}}$ is a positive definite symmetric matrix with all diagonal elements equal to $1$.
  It remains to show that \begin{align*}
    n \mD - \mA = \mD^{\nicefrac{1}{2}} (n \mI - \mQ) \mD^{\nicefrac{1}{2}} \succeq \mzero.
  \end{align*}
  Note that $\sum_{i=1}^n \lambda_i(\mQ) = \tr{\mQ} = n$, and hence, all eigenvalues of $\mQ$ belong to $(0,n)$.
\end{proof}

\begin{lemma}\label{lem:difference_bound_by_log}
  If $a, b \in (0,M]$ for some $M > 0$ and $b \geq a$ then \begin{align}
    b - a \leq M \cdot \log\parentheses*{\frac{b}{a}}.
  \end{align}
  If additionally, $a \geq M'$ for some $M' > 0$ then \begin{align}
    b - a \geq M' \cdot \log\parentheses*{\frac{b}{a}}.
  \end{align}
\end{lemma}
\begin{proof}
  Let $f(x) \defeq \log x$.
  By the mean value theorem, there exists $c \in (a,b)$ such that \begin{align*}
    \frac{1}{c} = f'(c) = \frac{f(b) - f(a)}{b - a} = \frac{\log b - \log a}{b - a} = \frac{\log(\frac{b}{a})}{b - a}.
  \end{align*}
  Thus, \begin{align*}
    b - a = c \cdot \log\parentheses*{\frac{b}{a}} < M \cdot \log\parentheses*{\frac{b}{a}}.
  \end{align*}
  Under the additional condition that $a \geq M'$, we obtain \begin{align*}
    b - a = c \cdot \log\parentheses*{\frac{b}{a}} > M' \cdot \log\parentheses*{\frac{b}{a}}.
  \end{align*}
\end{proof}

%% file: backmatter/A_approximations.tex
\section{Interpretations \& Approximations of Principle (\pref)}\label{sec:interpretations_approximations}

We give a brief overview of interpretations and approximations of \itl, as well as alternative decision rules adhering to the fundamental principle (\pref).\looseness=-1

The discussed interpretations of (\pref) differ mainly in how they quantify the ``uncertainty'' about~$\spA$.
In the GP setting, this ``uncertainty'' is captured by the covariance matrix~$\mSigma$ of~$\vfsub{\spA}$, and we consider two main ways of ``scalarizing''~$\mSigma$: \begin{enumerate}
    \item the total (marginal) variance $\tr{\mSigma}$, and
    \item the ``generalized variance'' $\det{\mSigma}$.
\end{enumerate}

The generalized variance --- which was originally suggested by \cite{wilks1932certain} as a generalization of variance to multiple dimensions --- takes into account correlations.
In contrast, the total variance discards all correlations between points in~$\spA$.\looseness=-1

All discussed decision rules following principle (\pref) (i.e., \itl, \vtl, \mmitl) differ only in their weighting of the points in~$\spA$, and they coincide when $\abs{\spA} = 1$.\looseness=-1

\subsection{Interpretations of \itl}

We briefly discuss three interpretations of \itl.

\paragraph{Minimizing generalized variance}

In the GP setting, \itl can be equivalently characterized as minimizing generalized posterior variance: \begin{align}
    \vx_n &= \argmax_{\vx \in \spS} \I{\vfsub{\spA}}{y_{\vx}}[\spD_n] \nonumber \\
    &= \argmax_{\vx \in \spS} \frac{1}{2} \log\parentheses*{\frac{\det{\Var{\vfsub{\spA} \mid \spD_{n-1}}}}{\det{\Var{\vfsub{\spA} \mid \spD_{n-1}, y_{\vx}}}}} \nonumber \\
    &= \argmin_{\vx \in \spS} \det{\Var{\vfsub{\spA} \mid \spD_{n-1}, y_{\vx}}}. \label{eq:itl_as_min_gen_var}
\end{align}

\paragraph{Maximizing relevance and minimizing redundancy}

An alternative interpretation of \itl is \begin{align}
    \I{\vfsub{\spA}}{y_{\vx}}[\spD_n] = \underbrace{\I{\vfsub{\spA}}{y_{\vx}}}_{\text{relevance}} - \underbrace{\I{\vfsub{\spA}}{y_{\vx} ; \spD_n}}_{\text{redundancy}}
\end{align} where $\I{\vfsub{\spA}}{y_{\vx} ; \spD_n} = \I{\vfsub{\spA}}{y_{\vx}} - \I{\vfsub{\spA}}{y_{\vx}}[\spD_n]$ denotes the \emph{multivariate information gain} (cf.~\cref{sec:definitions}).
In this way, \itl can be seen as maximizing observation relevance while minimizing observation redundancy.
This interpretation is common in the literature on feature selection \citep{peng2005feature,vergara2014review,beraha2019feature}.\looseness=-1

\paragraph{Steepest descent in measure spaces}

\itl can be seen as performing steepest descent in the space of probability measures over $\vfsub{\spA}$, with the KL divergence as metric: \begin{align*}
    \I{\vfsub{\spA}}{y_{\vx}}[\spD_{n}] = \E[y_{\vx}]{\KL{p(\vfsub{\spA} \mid \spD_{n}, y_{\vx})}{p(\vfsub{\spA} \mid \spD_{n})}}.
\end{align*}
That is, \itl finds the observation yielding the ``largest update'' to the current density.\looseness=-1

\subsection{Interpretations of \vtl}\label{sec:interpretations_approximations:vtl}

Quantifying the uncertainty about $\vfsub{\spA}$ by the marginal variance of points in $\spA$ rather than entropy (or generalized variance), the principle (\pref) leads to \vtl.
Note that if $\abs{\spA} = 1$, then \vtl is equivalent to \itl.
Unlike the similar, but more technical, \textsc{TruVar} algorithm proposed by \cite{bogunovic2016truncated}, \vtl does not require truncated variances, and hence, \vtl can be applied to constrained settings (where $\spA \not\subseteq \spS$) as well.\looseness=-1

\paragraph{Relationship to \itl}

Note that the \itl criterion in the GP setting can be expressed as \begin{align}
    \vx_n = \argmin_{\vx \in \spS} \tr{\!\log \Var{\vfsub{\spA} \mid \spD_{n-1}, y_{\vx}}} \label{eq:tr_interpr_of_itl}
\end{align} where for a positive semi-definite matrix $\mA$ with spectral decomposition $\mA = \mV \mLambda \transpose{\mV}$ we write $\log \mA = \mV \log \mLambda \transpose{\mV}$ for the logarithmic matrix function.
To derive \cref{eq:tr_interpr_of_itl} we use that $\log \det{\mA} = \sum_i \log \lambda_i(\mA) = \tr{\!\log \mA}$.
Hence, \itl and \vtl are identical up to a different weighting of the eigenvalues of the posterior covariance matrix.\looseness=-1

\paragraph{Minimizing a bound to the approximation error}

\cite{chowdhury2017kernelized} (page 19) bound the approximation error $\abs{\opt{f}(\vx) - \mu_n(\vx)}$ by \begin{align*}
    \underbrace{\abs{\vk_t(\vx)^\top \inv{(\mK_t + \mP_t)} \vvarepsilon_{1:t}}}_{\text{variance}} + \underbrace{\abs{\opt{f}(\vx) - \vk_t(\vx)^\top \inv{(\mK_t + \mP_t)} \vf_{1:t}}}_{\text{bias}}
\end{align*} where $\vk_t(\vx) \defeq \mKsub{\vx \vx_{1:t}}$, $\mK_t \defeq \mKsub{\vx_{1:t} \vx_{1:t}}$, and $\mP_t \defeq \mPsub{\vx_{1:t}}$.
Similar to a standard bias-variance decomposition, the first term measures variance and the second term measures bias.
Following \cref{lem:confidence_intervals}, \vtl can be seen as greedily minimizing this bound to the approximation error (i.e., both bias and variance).

\paragraph{Maximizing correlation to prediction targets weighted by their variance}

It can be shown (see the proof below) that the \vtl decision rule is equivalent to \begin{align}
    \vx_n = \argmax_{\vx \in \spS} \sum_{\vxp \in \spA} \Var{f_{\vxp} \mid \spD_{n-1}} \cdot \Cor{f_{\vxp}, y_{\vx} \mid \spD_{n-1}}^2. \label{eq:vtl_vs_ctl}
\end{align}
That is, \vtl maximizes the squared correlation between the next observation and the prediction targets, weighted by their variance.
Intuitively, prediction targets are weighted by their variance since more can be learned about a prediction target with higher variance.
This is precisely what leads to the ``diverse'' sample selection, and is akin to ``uncertainty sampling'' among the prediction targets and then selecting the observation which is most correlated with the selected prediction target.\looseness=-1

\begin{proof}
    Starting with the \vtl objective, we have \begin{align*}
        \argmin_{\vx \in \spS} \sum_{\vxp \in \spA} \Var{f_{\vxp} \mid \spD_n, y_{\vx}} &= \argmin_{\vx \in \spS} \sum_{\vxp \in \spA} \parentheses*{\Var{f_{\vxp} \mid \spD_n} - \frac{\Cov{f_{\vxp}, y_{\vx} \mid \spD_n}^2}{\Var{y_{\vx} \mid \spD_n}}} \\
        &= \argmax_{\vx \in \spS} \sum_{\vxp \in \spA} \frac{\Var{f_{\vxp} \mid \spD_n} \cdot \Cov{f_{\vxp}, y_{\vx} \mid \spD_n}^2}{\Var{f_{\vxp} \mid \spD_n} \cdot \Var{y_{\vx} \mid \spD_n}} + \const \\
        &= \argmax_{\vx \in \spS} \sum_{\vxp \in \spA} \Var{f_{\vxp} \mid \spD_n} \cdot \Cor{f_{\vxp}, y_{\vx} \mid \spD_n}^2 + \const.
    \end{align*}
\end{proof}

\subsection{Mean Marginal \itl}\label{sec:interpretations_approximations:mm_itl}

\cite{mackay1992information} previously proposed ``mean-marginal'' \itl (\mmitl) in the setting where ${\spS = \spX}$, which selects \begin{align}
    \vx_n &= \argmax_{\vx \in \spS} \sum_{\vxp \in \spA} \I{f_{\vxp}}{y_{\vx}}[\spD_{n-1}]
    \intertext{and which simplifies in the GP setting to}
    \vx_n &= \argmax_{\vx \in \spS} \frac{1}{2} \sum_{\vxp \in \spA} \log\parentheses*{\frac{\Var{f_{\vxp} \mid \spD_{n-1}}}{\Var{f_{\vxp} \mid \spD_{n-1}, y_{\vx}}}} \nonumber \\
    &= \argmin_{\vx \in \spS} \sum_{\vxp \in \spA} \log \Var{f_{\vxp} \mid \spD_{n-1}, y_{\vx}} \nonumber \\
    &= \argmin_{\vx \in \spS} \tr{\!\log \diag{\Var{\vfsub{\spA} \mid \spD_{n-1}, y_{\vx}}}}. \label{eq:mm_itl_and_vtl_comp}
    \intertext{Analogously to the derivation of \cref{eq:tr_interpr_of_itl}, this can also be expressed as}
    \vx_n &= \argmin_{\vx \in \spS} \det{\diag{\Var{\vfsub{\spA} \mid \spD_{n-1}, y_{\vx}}}}.
\end{align}

Effectively, \mmitl ignores the mutual interaction between points in $\spA$.
As can be seen from \cref{eq:mm_itl_and_vtl_comp} and as is also mentioned by \cite{mackay1992information}, \mmitl is equivalent to \vtl up to a different weighting of the points in $\spA$: instead of minimizing the average posterior variance (as in \vtl), \mmitl minimizes the average posterior log-variance.
Under the lens of principle (\pref), this can be seen as minimizing the average marginal entropy of predictions within the target space: \begin{align*}
    \vx_n &= \argmin_{\vx \in \spS} \sum_{\vxp \in \spA} \H{f_{\vxp}}[\spD_{n-1}, y_{\vx}].
\end{align*}

We remark that \mmitl is a special case of EPIG \citep[][Appendix E.2]{smith2023prediction}.\looseness=-1

\paragraph{Not a generalization of uncertainty sampling}

Unlike \itl, \mmitl is not a generalization of uncertainty sampling.
The reason is precisely that \mmitl ignores the mutual interaction between points in $\spA$.
Consider the example where ${\spX = \spS = \spA = \{1, \dots, 10\}}$ where $\vfsub{1:9}$ are highly correlated while $f_{10}$ is mostly independent of the other points.
Visually, imagine a smooth function (i.e., under a Gaussian kernel) with points $1$ through $9$ close to each other and point $10$ far away.
Further, suppose that point $10$ has a slightly larger marginal variance than the others.
Then, \mmitl would select one of the points $1:9$ since this leads to the largest reduction in the marginal (log-)variances (i.e., to a small posterior ``uncertainty'').\footnote{This is because the observation reduces uncertainty not just about the observed point itself.}
In contrast, \itl selects the point with the largest prior marginal variance (cf. \cref{sec:proofs:undirected_itl}), point $10$, since this leads to the largest reduction in entropy.\footnote{Because points $\vfsub{1:9}$ are highly correlated, $\H{\vfsub{1:9}}$ is already ``small''.}\looseness=-1

\paragraph{Similarity to \vtl}

Observe that the following decision rule is equivalent to \vtl: \begin{align*}
    \vx_n = \argmax_{\vx \in \spS} \tr{\Var{\vfsub{\spA} \mid \spD_{n-1}}} - \tr{\Varsm{\vfsub{\spA} \mid \spD_{n-1}, y_{\vx}}}.
\end{align*}
By \cref{lem:difference_bound_by_log}, for any $\vx \in \spS$, this objective value can be tightly lower- and upper-bounded (up to constant-factors) by \looseness=-1 \begin{align}
    &\sum_{\vxp \in \spA} \log\parentheses*{\frac{\Var{f_{\vxp} \mid \spD_{n-1}}}{\Var{f_{\vxp} \mid \spD_{n-1}, y_{\vx}}}} \nonumber \\
    &= 2 \sum_{\vxp \in \spA} \I{f_{\vxp}}{y_{\vx}}[\spD_{n-1}] \nonumber \tag{see \mmitl} \\
    \eeq{1} - \sum_{\vxp \in \spA} \log\parentheses*{1 - \Cor{f_{\vxp}, y_{\vx} \mid \spD_{n-1}}^2} \label{eq:mmitl_vs_ctl}
\end{align} where \e{1} is detailed in example 8.5.1 of \cite{cover1999elements}.
Thus, \vtl and \mmitl are closely related.

\paragraph{Experiments}

In our experiments with Gaussian processes from \cref{fig:gps:directed_advantage,fig:gps:appendix}, we observe that \mmitl performs similarly to \vtl and \ctl.\looseness=-1

\paragraph{Convergence of uncertainty}

We derive a convergence guarantee for \mmitl which is analogous to the guarantees for \itl from \cref{thm:objective_convergence_itl} and for \vtl from \cref{thm:objective_convergence_vtl}.
We will assume for simplicity that $\Gamma_n$ is monotonically decreasing in $n$ (i.e., $\alpha_n \leq 1$).\looseness=-1

\begin{theorem}[Convergence of uncertainty reduction of \mmitl]\label{lem:mm_itl_reduction_in_uncertainty}
    Assume that \cref{asm:bayesian_prior,asm:bayesian_noise} are satisfied.
    Then for any $n \geq 1$, if ${\Gamma_{0} \geq \cdots \geq \Gamma_{n-1}}$ and the sequence $\{\vx_i\}_{i=1}^n$ is generated by \mmitl, then \begin{align}
    \Gamma_{n-1} &\leq \frac{1}{n} \sum_{\vxp \in \spA} \gamma_n(\{\vxp\}; \spS).
  \end{align}
\end{theorem}
\begin{proof}
    We have \begin{align*}
        \Gamma_{n-1} &= \frac{1}{n} \sum_{i=0}^{n-1} \Gamma_{n-1} \\
        \eleq{1} \frac{1}{n} \sum_{i=0}^{n-1} \Gamma_i \\
        &= \frac{1}{n} \sum_{i=0}^{n-1} \max_{\vx \in \spS} \sum_{\vxp \in \spA} \Ism{f_{\vxp}}{y_{\vx}}[\spD_{n}] \\
        \eeq{2} \frac{1}{n} \sum_{i=0}^{n-1} \sum_{\vxp \in \spA} \Ism{f_{\vxp}}{y_{\vx_{n+1}}}[\spD_{n}] \\
        \eeq{3} \frac{1}{n} \sum_{\vxp \in \spA} \sum_{i=0}^{n-1} \Ism{f_{\vxp}}{y_{\vx_{n+1}}}[\vysub{\vx_{1:n}}] \\
        \eeq{4} \frac{1}{n} \sum_{\vxp \in \spA} \I{f_{\vxp}}{\vysub{\vx_{1:n}}} \\
        &\leq \frac{1}{n} \sum_{\vxp \in \spA} \max_{\substack{X \subseteq \spS \\ \abs{X} = n}} \I{f_{\vxp}}{\vysub{X}} \\
        &= \frac{1}{n} \sum_{\vxp \in \spA} \gamma_n(\{\vxp\}; \spS)
    \end{align*} where \e{1} follows by assumption; \e{2} follows from the \mmitl decision rule; \e{3} uses that the posterior variance of Gaussians is independent of the realization and only depends on the \emph{location} of observations; and \e{4} uses the chain rule of mutual information.
    The remainder of the proof is analogous to the proof of \cref{thm:objective_convergence_itl} (cf.~\cref{sec:proofs:objective_convergence}).
\end{proof}

Noting that \begin{align*}
    \I{f_{\vxp}}{y_{\vx}}[\spD_{n-1}] \leq \sum_{\vxp \in \spA} \Ism{f_{\vxp}}{y_{\vx}}[\spD_{n-1}]
\end{align*} for any $n \geq 1$, $\vx \in \spX$, and $\vxp \in \spA$, \cref{thm:variance_convergence} can be readily rederived for \mmitl (cf. \cref{lem:convergence_within_S,lem:convergence_helper2}).
Hence, the posterior marginal variances of \mmitl can be bounded uniformly in terms of $\Gamma_n$ analogously to \itl.\looseness=-1

\subsection{Correlation-based Transductive Learning}\label{sec:interpretations_approximations:ctl}

We will briefly look at the \ctl (\emph{\textbf{C}orrelation-based \textbf{TL}}) decision rule \begin{align}
    \vx_{n} = \argmax_{\vx \in \spS} \sum_{\vxp \in \spA} \Cor{f_{\vx},f_{\vxp} \mid \spD_{n-1}}
\end{align} which permits no interpretation under principle (\pref).
However, if all correlations are non-negative (such as for the standard Gaussian and Matérn kernels), \ctl is closely related to \itl, \vtl, and \mmitl~(cf.~\cref{eq:vtl_vs_ctl,eq:mmitl_vs_ctl}).
In this case, if $\abs{\spA} = 1$, then all decision rules coincide.\looseness=-1

If, on the other hand, correlations may be negative then there is a crucial difference between \ctl and the decision rules motivated from principle (\pref).
Namely, decision rules following (\pref) exhibit a preference for points with high \emph{absolute} correlation to prediction targets as opposed to \ctl which prefers points with high \emph{positive} correlation.
This stems from the intuitive fact that points with a strong negative correlation are equally informative as points with a strong positive correlation.
Nevertheless, we observe in our experiments that (even for a linear kernel which does not ensure non-negative correlations) points selected by \itl and \vtl are typically positively correlated with prediction targets.\looseness=-1

\subsection{Summary}

We have seen that \itl, \vtl, and \mmitl can be seen as different interpretations of the same fundamental principle (\pref), with the approximations \ctl.
If ${\abs{\spA} = 1}$ and correlations are non-negative, then all four decision rules are equivalent.
\ctl prefers points with high positive correlation whereas the other decision rules prefer points with high absolute correlation.
\itl is the only decision rule that takes into account the mutual dependence between points in $\spA$, and \vtl and \mmitl differ only in their weighting of the posterior marginal variances of points in $\spA$.\looseness=-1

%% file: backmatter/D_generalizations.tex
\section{Stochastic Target Spaces}\label{sec:generalizations:roi}

When the target space $\spA$ is large, it may be computationally infeasible to compute the exact objective.
A natural approach to address this issue is to approximate the target space by a smaller set of size $K$.\looseness=-1

\paragraph{Discretizing the target space}

One possibility is to discretize the target space~$\spA$.
Compact target spaces can be addressed, e.g., via discretization arguments which are common in the Bayesian optimization literature (see, e.g., appendix C.1 of \cite{srinivas2009gaussian}).
That is, if the target space can be covered approximately using a finite (possibly large) set of points, the guarantees of \cref{thm:variance_convergence} extend directly.
This, however, can be impractical when the required size of discretization for sufficiently small approximation error is large.
In the following, we briefly discuss a natural alternative approach based on sampling points from $\spA$.

\paragraph{Target distributions}

Let $\spA \subseteq \spX$ be a (possibly continuous) target space, and let $\spPA$ be a probability distribution supported on $\spA$.
In iteration $n$, a subset $A_n$ of $K$ points is sampled independently from $\spA$ according to the distribution~$\spPA$ and the objective is computed on this subset.
Formally, this amounts to a single-sample Monte Carlo approximation of \begin{align}
  \vx_n \in \argmax_{\vx \in \spS} \E[A \iid \spPA]{\I{\vfsub{A}}{y_{\vx}}[\spD_{n-1}]}.
\end{align}

The convergence guarantees from \cref{sec:proofs} can be generalized to the setting of stochastic target spaces by estimating how often points ``near'' a specified prediction target $\vx \in \spA$ are sampled.

\begin{definition}[$\gamma$-ball at $\vx$]
  Given $\vx \in \spA$ and any $\gamma \geq 0$, we call the set \begin{align*}
    B_\gamma(\vx) \defeq \{\vxp \in \spX \mid \norm{\vx - \vxp} \leq \gamma\}
  \end{align*} the $\gamma$-ball at $\vx$.
  Further, we call $\spPA(B_\gamma(\vx))$ the weight of that ball.\looseness=-1
\end{definition}

\begin{proposition}[sketch]
  Given any $n \geq 1, K \geq 1, \gamma > 0$, and ${\vx \in \spA}$, suppose that $B_\gamma(\vx)$ has weight $p > 0$.
  Assume that the \itl objective is $L_I$-Lipschitz continuous.
  Then, with probability at least ${1-\exp(-(1-p) n / (8 K))}$, \begin{align*}
      \sigma_n^2(\vx) \lesssim \eta_{\spS}^2(\vx) + C L_I \gamma \frac{\gamma_{k(n)}}{\sqrt{k(n)}}
    \end{align*} where $k(n) \defeq K p n / 2$.
\end{proposition}
\begin{proof}[Proof sketch]
  Let $Y_i \sim \mathrm{Binom}(K, p)$ denote the random variable counting the number of occurrences of a point from $B_\gamma(\vx)$ in $A_i$.
  Moreover, we write ${X_i \defeq \Ind{B_\gamma(\vx) \cap A_i \neq \emptyset}}$.
  Note that \begin{align*}
    \nu \defeq \E*{X_i} &= \Pr{B_\gamma(\vx) \cap A_i \neq \emptyset} = 1 - \Pr{Y_i = 0} = 1 - (1-p)^K \approx K p
  \end{align*} where the approximation stems from a first-order truncation of the Bernoulli series.
  Let $X \defeq \sum_{i=1}^n X_i$ with $\E*{X} = n \nu \approx K p n$.

  Using the assumed Lipschitz-continuity of the objective, we know that ${\Ism{\vfsub{A'}}{y_{\vx}}[\spD_{n-1}] \leq L_I \gamma \Ism{\vfsub{A}}{y_{\vx}}[\spD_{n-1}]}$ where ${A' \defeq (A \setminus \{\vx_\gamma\}) \cup \{\vx\}}$ and $\vx_\gamma$ is the point from the $\gamma$-ball at $\vx$.
  The bound then follows analogously to \cref{thm:variance_convergence}.

  Finally, by Chernoff's bound, at least $K p n / 2$ iterations contain a point from $B_\gamma(\vx)$ with probability at least $1 - \exp(-K p n / 8)$.
\end{proof}

This strategy can also be used to generalize the \vtl, \ctl, and \mmitl objectives to stochastic target spaces.

%% file: backmatter/1_closed_form.tex
\section{Closed-form Decision Rules}\label{sec:closed_form_decision_rules}

Below, we list the closed-form expressions for the \itl and \vtl objectives.
In the following, $k_{n}$ denotes the kernel conditional on $\spD_n$.

\paragraph{\itl}

\begin{align}
    \Ism{\vfsub{\spA}}{y_{\vx}}[\spD_{n-1}] &= \frac{1}{2} \log\parentheses*{\frac{\Var{y_{\vx} \mid \spD_{n-1}}}{\Var{y_{\vx} \mid \vfsub{\spA}, \spD_{n-1}}}} \label{eq:objective_gp_setting} \\
    &= \frac{1}{2} \log\parentheses*{\frac{k_{n-1}(\vx, \vx) + \rho^2}{\hat{k}_{n-1}(\vx, \vx) + \rho^2}} \nonumber
\end{align} where $\hat{k}_n(\vx, \vx) = k_n(\vx, \vx) - \vk_n(\vx, \spA) \inv{\mK_n(\spA, \spA)} \vk_n(\spA, \vx)$.

\paragraph{\vtl}

\begin{align*}
    \tr{\Var{\vfsub{\spA} \mid \spD_{n-1}, y_{\vx}}} &= \sum_{\vxp \in \spA} \parentheses*{k_{n-1}(\vxp, \vxp) - \frac{k_{n-1}(\vx, \vxp)^2}{k_{n-1}(\vx, \vx) + \rho^2}}.
\end{align*}

%% file: backmatter/E_computational_complexity.tex
\section{Computational Complexity}\label{sec:computational_complexity}

Evaluating the acquisition function of \itl in round $n$ requires computing for each $\vx \in \spS$, \begin{align*}
  &\Ism{\vfsub{\spA}}{y_{\vx}}[\spD_n] \\
  &= \frac{1}{2} \log\parentheses*{\frac{\det{\Var{\vfsub{\spA} \mid \spD_n}}}{\det{\Var{\vfsub{\spA} \mid y_{\vx}, \spD_n}}}} &&\text{(forward)} \\
  &= \frac{1}{2} \log\parentheses*{\frac{\Var{y_{\vx} \mid \spD_n}}{\Var{y_{\vx} \mid \vfsub{\spA}, \spD_n}}} &&\text{(backward)}.
\end{align*}
Let $|\spS| = m$ and $|\spA| = k$.
Then, the forward method has complexity $\BigO{m \cdot k^3}$.
For the backward method, observe that the variances are scalar and the covariance matrix $\Var{\vfsub{\spA} \mid \spD_n}$ only has to be inverted once for all points $\vx$. Thus, the backward method has complexity $\BigO{k^3 + m}$.\looseness=-1

When the size $m$ of $\spS$ is relatively small (and hence, all points in $\spS$ can be considered during each iteration of the algorithm), GP inference corresponds simply to computing conditional distributions of a multivariate Gaussian.
The performance can therefore be improved by keeping track of the full posterior distribution over $\vfsub{\spS}$ of size $\BigO{m^2}$ and conditioning on the latest observation during each iteration of the algorithm.
In this case, after each observation the posterior can be updated at a cost of $\BigO{m^2}$ which does not grow with the time $n$, unlike classical GP inference.\looseness=-1

Overall, when $m$ is small, the computational complexity of \itl is $\BigO{k^3 + m^2}$.
When $m$ is large (or possibly infinite) and a subset of $\tilde{m}$ points is considered in a given iteration, the computational complexity of \itl is $\BigO{k^3 + \tilde{m} \cdot n^3}$, neglecting the complexity of selecting the $\tilde{m}$ candidate points.
In the latter case, the computational cost of \itl is dominated by the cost of GP inference.\looseness=-1

\cite{khanna2017scalable} discuss distributed and stochastic approximations of greedy algorithms to (weakly) submodular problems that are also applicable to \itl.\looseness=-1

%% file: backmatter/G_additional_gp_experiments.tex
\section{Additional GP Experiments \& Details}\label{sec:gps_appendix}

We use homoscedastic Gaussian noise with standard deviation $\rho = 0.1$ and a discretization of ${\spX = [-3,3]^2}$ of size $2\,500$.
Uncertainty bands correspond to one standard error over $10$ random seeds.\looseness=-1

\paragraph{Additional experiments}

\begin{figure*}[t]
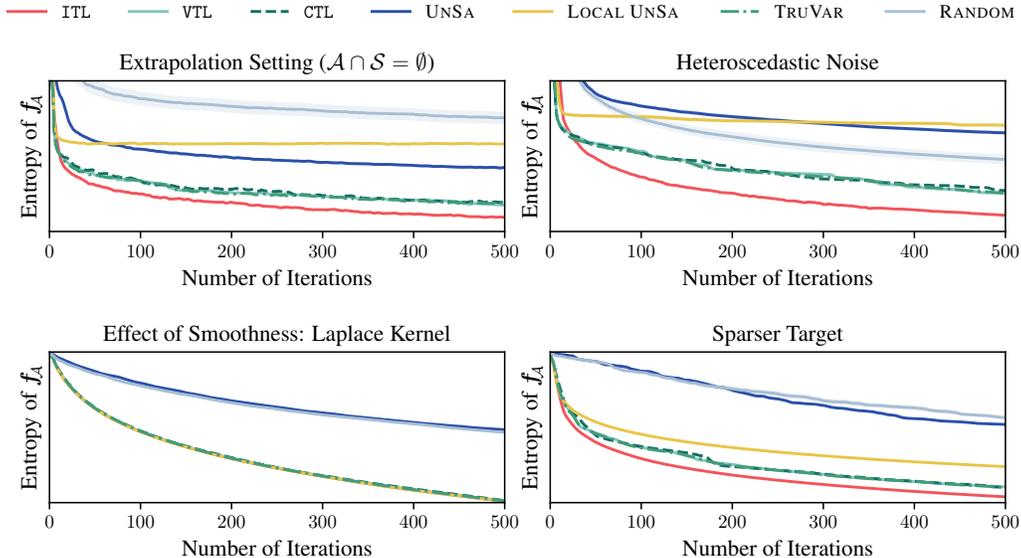

    \incplt[\textwidth]{gps_appendix}
    \vspace{-0.5cm}
    \caption{Additional GP experiments}
    \label{fig:gps:appendix}
\end{figure*}

\Cref{fig:gps:appendix} includes the following additional experiments: \begin{enumerate}
  \item \emph{Extrapolation Setting ($\spA \cap \spS = \emptyset$)}: Right experiment from \cref{fig:gps:directed_advantage} under the Gaussian kernel. \itl has a similar advantage as in the setting shown in \cref{fig:gps:directed_advantage2}.

  \item \emph{Heteroscedastic Noise}: Left experiment from \cref{fig:gps:directed_advantage} under the Gaussian kernel with heteroscedastic Gaussian noise \begin{align*}
    \rho(\vx) = \begin{cases}
      1 & \text{if $\vx \in [-\frac{1}{2}, \frac{1}{2}]^2$} \\
      0.1 & \text{otherwise}
    \end{cases}.
  \end{align*}
  If observation noise is heteroscedastic, in considering \emph{posterior} rather than \emph{prior} uncertainty, \itl avoids points with high aleatoric uncertainty, which accelerates learning.

  \item \emph{Effect of Smoothness}: Experiment from \cref{fig:gps:directed_advantage2} under the Laplace kernel.
  All algorithms except for US and \textsc{Random} perform equally well.
  This validates our claims from \cref{sec:gps}: in the extreme non-smooth case of a Laplace kernel and $\spA \subseteq \spS$, points outside $\spA$ do not provide any additional information, and \itl and ``local'' \textsc{UnSa} coincide.

  \item \emph{Sparser Target}: Experiment from \cref{fig:gps:directed_advantage2} under the Gaussian kernel, but with domain extended to ${\spX = [-10, 10]^2}$.
\end{enumerate}

\paragraph{Hyperparameters of \textsc{TruVar}}

As suggested by \cite{bogunovic2016truncated}, we use $\smash{\tilde{\eta}^2_{(1)} = 1}$, $r = 0.1$, and $\delta = 0$ (even though the theory only holds for $\delta > 0$).
The \textsc{TruVar} baseline only applies when $\spA \subseteq \spS$ (cf. \cref{sec:related_work}).

\paragraph{Smoothing to reduce numerical noise}

Applied running average with window $5$ to entropy curves of \cref{fig:gps:directed_advantage,fig:gps:appendix} to smoothen out numerical noise.

%% file: backmatter/H_additional_nn_experiments.tex
\section{Alternative Settings for Active Fine-Tuning}\label{sec:nns_fine_tuning_settings}

In our main experiments, we consider the setting $\spA \cap \spS = \emptyset$, i.e., the prediction targets cannot be used for fine-tuning since their labels are not known.
This setting is particularly relevant for practical applications where the model is fine-tuned dynamically at test time to each prediction target (or a small set of prediction target).
Put differently, in this ``transductive'' setting, extrapolation to new prediction targets happens at \emph{test-time} with knowledge of the prediction target(s).
This is in contrast to a more traditional ``inductive'' setting, where extrapolation happens at \emph{train-time} without knowledge of the concrete prediction targets, but under the assumption of samples from (or knowledge of) the target distribution.
In the following, we briefly survey two settings motivated from an ``inductive'' perspective.\looseness=-1

\subsection{Prediction Targets are Contained in Sample Space: $\spA \subseteq \spS$}

If labels can be obtained cheaply, one can also fine-tune on the prediction targets directly, i.e., $\spA \subseteq \spS$.
Note, however, that the set $\spA$ is still assumed to be small (e.g., $|\spA| = 100$ in the CIFAR-100 experiment).
We perform an experiment in this setting and report the results in \cref{fig:nns_a_sub_s}.
The experiment shows that --- similarly to the GP experiment from \cref{fig:gps:directed_advantage} --- there can be \emph{additional value} in fine-tuning the model on relevant data selected from $\spS$ beyond simply fine-tuning the model on $\spA$.\looseness=-1

\begin{figure*}[]
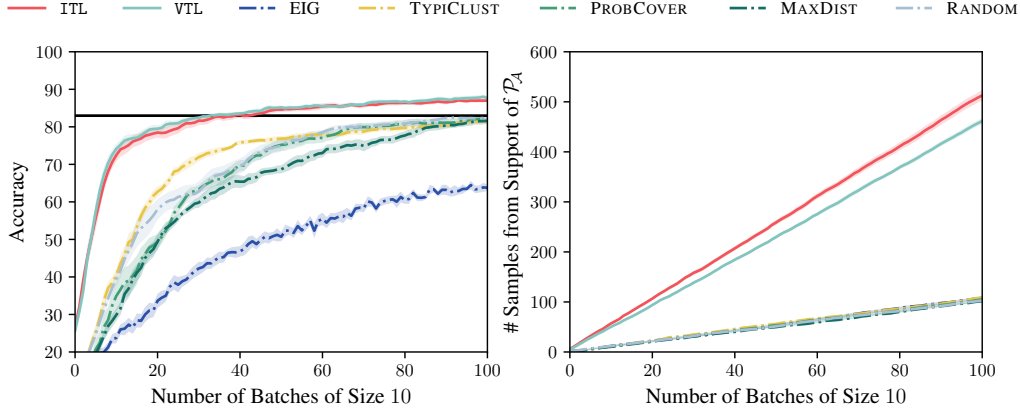

  \incplt[\textwidth]{nns_a_sub_s}
  \vspace{-0.5cm}
  \caption{Evaluation of CIFAR-100 experiment in the setting $\spA \subseteq \spS$, i.e., one can also sample from the $100$ prediction targets $\spA$. The solid black line denotes the performance of the model fine-tuned on all of $\spA$. This experiment shows that there is \emph{additional value} in fine-tuning the model on relevant data from $\spS$ beyond simply fine-tuning the model on $\spA$. The baselines are summarized in \cref{sec:nns_appendix:undirected}}
  \label{fig:nns_a_sub_s}
\end{figure*}

\subsection{Active Domain Adaptation}

Active DA \citep{rai2010domain,saha2011active,berlind2015active} studies the problem of selecting the most informative samples from a (large) target domain $\spA$, given a model trained on a source domain $\spS$.
This problem can be cast as an instance of transductive active learning with target space $\spA$ and sample space $\spS' = \spS \cup \spA$ where the model is already conditioned on all of $\spS$.
This is slightly different from the setting considered in \cref{sec:nns} where $\spA$ is small and not necessarily part of the sample space.
We hypothesize that \itl behaves similarly to recent work on active DA \citep{su2020active,prabhu2021active,fu2021transferable}: querying informative and diverse samples from~$\spA$ that are dissimilar to~$\spS$. Evaluating \itl and \vtl empirically in this setting is a promising direction for future work.\looseness=-1

\section{Additional NN Experiments \& Details}\label{sec:nns_appendix}

We outline the active fine-tuning of NNs in \cref{alg:itl_nns}.\looseness=-1

\begin{algorithm}[htb]
  \caption{Active Fine-Tuning of NNs}
  \label{alg:itl_nns}
\begin{algorithmic}
  \STATE {\bfseries Given:} initialized or pre-trained model $f$, \emph{small} sample $A \sim \spPA$
  \STATE initialize dataset $\spD = \emptyset$
  \STATE \textbf{repeat}
  \STATE\hspace{1em} sample $S \sim \spPS$
  \STATE\hspace{1em} subsample target space $A' \uar A$
  \STATE\hspace{1em} initialize batch $B = \emptyset$
  \STATE\hspace{1em} compute kernel matrix $\mK$ over domain $[S, A']$
  \STATE\hspace{1em} \textbf{repeat $b$ times}
  \STATE\hspace{2em} compute acquisition function w.r.t. $A'$, based on $\mK$
  \STATE\hspace{2em} add maximizer $\vx \in S$ of acquisition function to $B$
  \STATE\hspace{2em} update conditional kernel matrix $\mK$
  \STATE\hspace{1em} obtain labels for $B$ and add to dataset $\spD$
  \STATE\hspace{1em} update $f$ using data $\spD$
\end{algorithmic}
\end{algorithm}

In \cref{sec:nns_appendix:details}, we detail metrics and hyperparameters.
We describe in \cref{sec:nns_appendx:embeddings,sec:nns_appendx:uncertainty_quantification} how to compute the (initial) conditional kernel matrix $\mK$, and in \cref{sec:nns_appendx:batch_selection} how to update this matrix $\mK$ to obtain conditional embeddings for batch selection.\looseness=-1

In \cref{sec:nns_appendix:undirected}, we show that \itl and \ctl significantly outperform a wide selection of commonly used heuristics.
In \cref{sec:nns_appendix:additional_experiments,sec:nns_appendix:ablation_sigma}, we conduct additional experiments and ablations.\looseness=-1

\cite{hubotter2024active} discusses additional motivation and related work that has previously studied active fine-tuning, but which has largely focused on the training algorithm rather than data selection.\looseness=-1

\subsection{Experiment Details}\label{sec:nns_appendix:details}

\begin{table}[t]
    \caption{Hyperparameter summary of NN experiments. (*) we train until convergence on oracle validation accuracy.}
    \label{table:nn_hyperparams}
    \vskip 0.15in
    \begin{center}
    \begin{tabular}{@{}lll@{}}
      \toprule
      & MNIST & CIFAR-100 \\
      \midrule
      $\rho$ & $0.01$ & $1$ \\
      $M$ & $30$ & $100$ \\
      $m$ & $3$ & $10$ \\
      $k$ & $1\,000$ & $1\,000$ \\
      batch size $b$ & $1$ & $10$ \\
      \# of epochs & (*) & $5$ \\
      learning rate & $0.001$ & $0.001$ \\
      \bottomrule
    \end{tabular}
    \end{center}
    \vskip -0.1in
\end{table}

We evaluate the accuracy with respect to $\spPA$ using a Monte Carlo approximation with out-of-sample data: \begin{align*}
  \text{accuracy}(\vthetahat) \approx \E*[(\vx, y) \sim \spPA]{\Ind{y = \argmax_i f_i(\vx; \vthetahat)}}.
\end{align*}

We provide an overview of the hyperparameters used in our NN experiments in \cref{table:nn_hyperparams}.
The effect of noise standard deviation $\rho$ is small for all tested $\rho \in [1,100]$ (cf.~ablation study in \cref{table:nn_sigma_ablation}).\footnote{We use a larger noise standard deviation $\rho$ in CIFAR-100 to stabilize the numerics of batch selection via conditional embeddings (cf. \cref{table:nn_sigma_ablation}).}
$M$ denotes the size of the sample ${A \sim \spPA}$.
In each iteration, we select the target space $\spA \gets A'$ as a random subset of $m$ points from $A$.\footnote{This appears to improve the training, likely because it prevents overfitting to peculiarities in the finite sample~$A$~(cf. \cref{fig:nns_subsampled_target_frac}).}
We provide an ablation over $m$ in \cref{sec:nns_appendix:additional_experiments}.\looseness=-1

During each iteration, we select the batch $B$ according to the decision rule from a random sample from $\spPS$ of size $k$.\footnote{In large-scale problems, the work of \cite{coleman2022similarity} suggests to use an (approximate) nearest neighbor search to select the (large) candidate set rather than sampling u.a.r. from $\spPS$. This can be a viable alternative to simply increasing $k$ and suggests future work.}\looseness=-1

Since we train the MNIST model from scratch, we train from random initialization until convergence on oracle validation accuracy.\footnote{That is, to stop training as soon as accuracy on a validation set from $\spPA$ decreases in an epoch.}
We do this to stabilize the learning curves, and provide the least biased (due to the training algorithm) results.
For CIFAR-100, we train for $5$ epochs (starting from the previous iterations' model) which we found to be sufficient to obtain good performance.\looseness=-1

We use the ADAM optimizer \citep{kingma2014adam}.
In our CIFAR-100 experiments, we use a pre-trained EfficientNet-B0 \citep{tan2019efficientnet}, and fine-tune the final and penultimate layers.
We freeze earlier layers to prevent overfitting to the ``few-shot'' training data.\looseness=-1

To prevent numerical inaccuracies when computing the \itl objective, we optimize \begin{align}
  \Ism{\vysub{\spA}}{y_{\vx}}[\spD_{n-1}] = \frac{1}{2} \log\parentheses*{\frac{\Var{y_{\vx}}[\spD_{n-1}]}{\Var{y_{\vx}}[\vysub{\spA}, \spD_{n-1}]}}
\end{align} instead of \cref{eq:objective_gp_setting}, which amounts to adding $\rho^2$ to the diagonal of the covariance matrix before inversion.
This appears to improve numerical stability, especially when using gradient embeddings.\footnote{In our experiments, we observe that the effect of various choices of $\rho$ on this slight adaptation of the \itl decision rule has negligible impact on performance. The more prominent effect of $\rho$ appears to arise from the batch selection via conditional embeddings (cf. \cref{table:nn_sigma_ablation}).}\looseness=-1

In our experiments, we use last-layer neural tangent embeddings\footnote{We observe essentially the same performance with loss gradient embeddings, cf. \cref{sec:nns_appendx:embeddings}.} and ${\mSigma = \mI}$ to evaluate \itl and \vtl, and select inputs for labeling and training~$f$.
Notably, we use this linear Gaussian approximation of $f$ only to guide the active data selection and not for inference.\looseness -1

\subsection{Embeddings and Kernels}\label{sec:nns_appendx:embeddings}

Using a neural network to parameterize~$f$, we evaluate the canonical approximations of~$f$ by a stochastic process in the following.\looseness=-1

An embedding~$\vphi(\vx)$ is a latent representation of an input~$\vx$.
Collecting the embeddings as rows in the design matrix~$\mPhi$ of a set of inputs~$X$, one can approximate the network by the linear function~${\vfsub{X} = \mPhi \vbeta}$ with weights~$\vbeta$.
Approximating the weights by ${\vbeta \sim \N{\vmu}{\mSigma}}$ implies that ${\vfsub{X} \sim \N{\mPhi\vmu}{\mPhi \mSigma \transpose{\mPhi}}}$.
The covariance matrix~${\mKsub{XX} = \mPhi \mSigma \transpose{\mPhi}}$ can be succinctly represented in terms of its associated kernel~${k(\vx, \vxp) = \transpose{\vphi(\vx)} \mSigma \vphi(\vxp)}$.
Here, \begin{itemize}[noitemsep]
  \item $\vphi(\vx)$ is the latent representation of $\vx$, and
  \item $\mSigma$ captures the dependencies in the latent space.
\end{itemize}

While any choice of embedding~$\vphi$ is possible, the following are common choices: \begin{enumerate}
  \item \emph{Last-Layer}: A common choice for $\vphi(\vx)$ is the representation of $\vx$ from the penultimate layer of the neural network \citep{holzmuller2023framework}.
  Interpreting the early layers as a feature encoder, this uses the low-dimensional feature map akin to random feature methods \citep{rahimi2007random}.\looseness=-1

  \item \emph{Output Gradients (eNTK)}: Another common choice is $\vphi(\vx) = \grad_\vtheta \vf(\vx; \vtheta)$ where $\vtheta$ are the network parameters \citep{holzmuller2023framework}.
  Its associated kernel is known as the \emph{empirical neural tangent kernel} (eNTK) and the posterior mean of this GP approximates ultra-wide NNs trained with gradient descent \citep{jacot2018neural,arora2019exact,lee2019wide,khan2019approximate,he2020bayesian,malladi2023kernel}.
  \cite{kassraie2022neural} derive bounds of $\gamma_n$ under this kernel.
  If $\vtheta$ is restricted to the weights of the final linear layer, then this embedding is simply the last-layer embedding.\looseness=-1

  \item \emph{Loss Gradients}: Another possible choice is \begin{align*}
    \vphi(\vx) = \left. \grad_{\vtheta} \ell(\vf(\vx; \vtheta), \hat{y}(\vx)) \right\rvert_{\vtheta = \vthetahat}
  \end{align*} where $\ell$ is a loss function, $\hat{y}(\vx)$ is the predicted label, and $\vthetahat$ are the current parameter estimates \citep{ash2020deep}.\looseness=-1

  \item \emph{Outputs (eNNGP)}: Another possible choice is $\vphi(\vx) = \vf(\vx)$, i.e., the output of the network.
  Its associated kernel is known as the \emph{empirical neural network Gaussian process} (eNNGP) kernel \citep{lee2017deep}.\looseness=-1

  \item \emph{Predictive} \citep{kirsch2023black}: Given a Bayesian neural network \citep{blundell2015weight} or probabilistic (deep) ensemble \citep{lakshminarayanan2017simple}, which induce samples $\vtheta_1, \dots, \vtheta_K \sim p(\vtheta)$ from the distribution over network parameters, one can approximate the predictive covariance $k(\vx, \vxp) = \Cov[\vtheta]{f(\vx; \vtheta), f(\vxp; \vtheta)}$.
  This kernel measures proximity in the prediction space rather than parameter space and as such does not require gradient information.
  The corresponding feature map is $\smash{\vphi(\vx) = \frac{1}{\sqrt{K}} \transpose{[\bar{f}(\vx; \vtheta_1) \; \cdots \; \bar{f}(\vx; \vtheta_K)]}}$ where $\smash{\bar{f}(\vx; \vtheta_k) \defeq f(\vx; \vtheta_k) - \frac{1}{K} \sum_{l=1}^K f(\vx; \vtheta_l)}$.\looseness=-1
\end{enumerate}

In the additional experiments from this appendix we use last-layer embeddings unless noted otherwise.
We compare the performance of last-layer and the loss gradient embedding \begin{align}
  \vphi(\vx) = \left. \grad_{\vthetap} \ell_{\mathrm{CE}}(\vf(\vx; \vtheta), \hat{y}(\vx)) \right\rvert_{\vtheta = \vthetahat} \label{eq:loss_gradient_embedding}
\end{align} where $\vthetap$ are the parameters of the final output layer, $\smash{\vthetahat}$ are the current parameter estimates, $\smash{\hat{y}(\vx) = \argmax_i f_i(\vx; \vthetahat)}$ are the associated predicted labels, and $\ell_{\mathrm{CE}}$ denotes the cross-entropy loss.
This gradient embedding captures the potential update direction upon observing a new point \citep{ash2020deep}.
Moreover, \cite{ash2020deep} show that for most neural networks, the norm of these gradient embeddings are a conservative lower bound to the norm assumed by taking any other proxy label $\hat{y}(\vx)$.
In \cref{fig:nns_embeddings}, we observe only negligible differences in performance between this and the last-layer embedding.\looseness=-1

\begin{figure*}[]
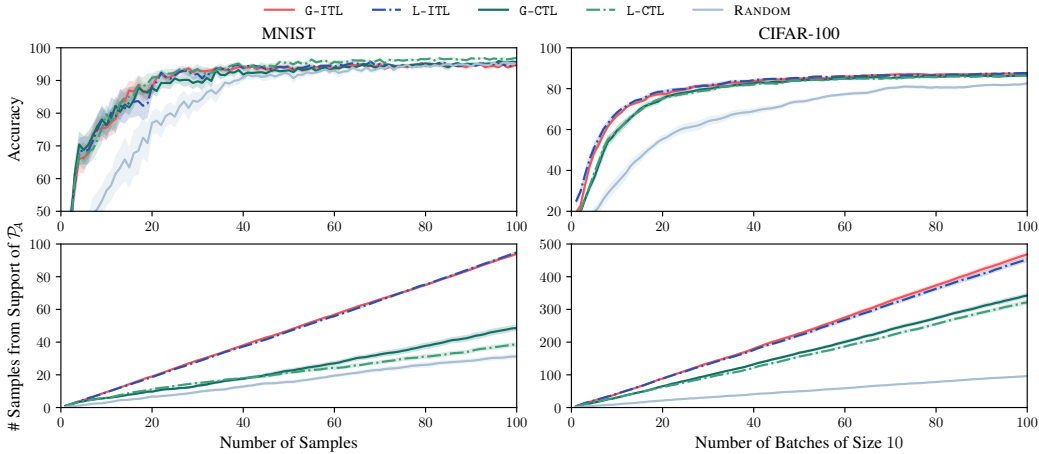

  \incplt[\textwidth]{nns_embeddings}
  \vspace{-0.5cm}
  \caption{Comparison of loss gradient (``G-'') and last-layer embeddings (``L-'').}
  \label{fig:nns_embeddings}
\end{figure*}

\subsection{Towards Uncertainty Quantification in Latent Space}\label{sec:nns_appendx:uncertainty_quantification}

A straightforward and common approximation of the uncertainty about NN weights is given by $\mSigma = \mI$, and we use this approximation throughout our experiments.\looseness=-1

The poor performance of \textsc{UnSa} (cf. \cref{sec:nns_appendix:undirected}) with this approximation suggests that with more sophisticated approximations, the performance of \itl, \vtl, and \ctl can be further improved.
Further research is needed to study the effect of more sophisticated approximations of ``uncertainty'' in the latent space.
For example, with parameter gradient embeddings, the latent space is the network parameter space where various approximations of $\mSigma$ based on Laplace approximation \citep{daxberger2021laplace,antoran2022adapting}, variational inference \citep{blundell2015weight}, or Markov chain Monte Carlo \citep{maddox2019simple} have been studied.
We also evaluate Laplace approximation (LA, \cite{daxberger2021laplace})
for estimating $\mSigma$ but see no improvement (cf. \cref{fig:nns_uncertainty_quantification}).
Nevertheless, we believe that uncertainty quantification is a promising direction for future work, with the potential to improve performance of \itl and its variations substantially.\looseness=-1

\begin{figure*}[]
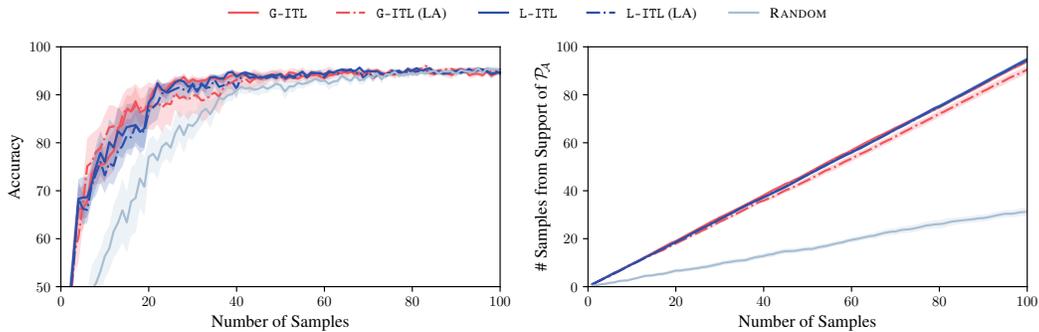

  \incplt[\textwidth]{nns_uncertainty_quantification}
  \vspace{-0.5cm}
  \caption{Uncertainty quantification (i.e., estimation of $\mSigma$) via a Laplace approximation (LA, \cite{daxberger2021laplace}) over last-layer weights using a Kronecker factored log-likelihood Hessian approximation \citep{martens2015optimizing} and the loss gradient embeddings from \cref{eq:loss_gradient_embedding}. The results are shown for the MNIST experiment. We do not observe a performance improvement beyond the trivial approximation $\mSigma = \mI$.}
  \label{fig:nns_uncertainty_quantification}
\end{figure*}

\subsection{Batch Selection via Conditional Embeddings}\label{sec:nns_appendx:batch_selection}

We will refer to the greedy decision rule from \cref{eq:batch_selection} as \textsc{BaCE}, short for \emph{\textbf{Ba}tch selection via \textbf{C}onditional \textbf{E}mbeddings}.
\textsc{BaCE} can be implemented efficiently using the Gaussian approximation of $\vfsub{X}$ from \cref{sec:nns_appendx:embeddings} by iteratively conditioning on the previously selected points $\vx_{n,1:i-1}$, and updating the kernel matrix $\mKsub{XX}$ using the closed-form formula for the variance of conditional Gaussians: \begin{align}
  \mKsub{XX} \gets \mKsub{XX} - \frac{1}{\mKsub{\vx_j \vx_j} + \rho^2} \mKsub{X \vx_j} \mKsub{\vx_j X}
\end{align} where $j$ denotes the index of the selected $\vx_{n,i}$ within $X$ and $\rho^2$ is the noise variance.
Note that $\mKsub{\vx_j \vx_j}$ is a scalar and $\mKsub{X \vx_j}$ is a row vector, and hence, this iterative update can be implemented efficiently.\looseness=-1

We remark that \cref{eq:batch_selection} is a natural extension of previous non-adaptive active learning methods, which typically maximize some notion of ``distance'' between points in the batch, to the ``directed'' setting \citep{ash2020deep,zanette2021design,holzmuller2023framework,pacchiano2024experiment}.
\textsc{BaCE} simultaneously maximizes ``distance'' between points in a batch and minimizes ``distance'' to points in~$\spA$.\looseness=-1

\paragraph{The sample efficiency of \textsc{BaCE}}

$B_n$, and therefore also the greedily constructed $B'_n$ (which gives a constant-factor approximation with respect to the objective), yields diverse batches by design.
In \cref{fig:nns_batch_selection}, we compare \textsc{BaCE} to selecting the top-$b$ points according to the decision rule (which does \emph{not} yield diverse batches).
We observe a significant improvement in accuracy and data retrieval when using \textsc{BaCE}.
We expect the gap between both approaches to widen further with larger batch sizes.\looseness=-1

\begin{figure*}[]
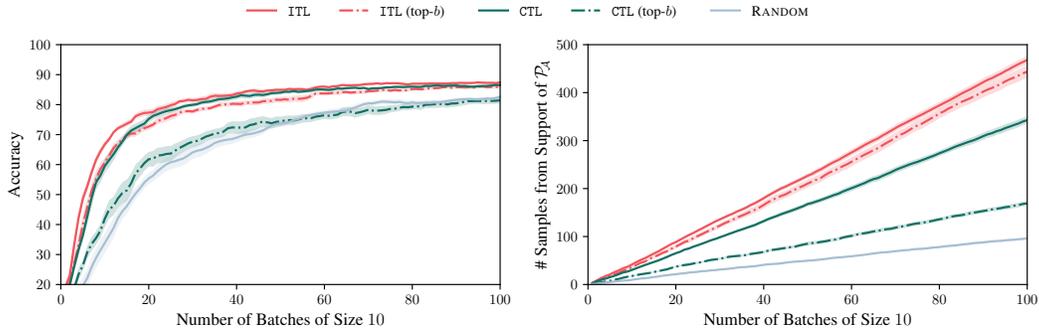

  \incplt[\textwidth]{nns_batch_selection}
  \vspace{-0.5cm}
  \caption{Advantage of batch selection via conditional embeddings over top-$b$ selection in the CIFAR-100 experiment.}
  \label{fig:nns_batch_selection}
\end{figure*}

\paragraph{Computational complexity of \textsc{BaCE}}

As derived in \cref{sec:computational_complexity}, a single batch selection step of \textsc{BaCE} has complexity $\BigO{b (k^3 + m^2)}$ where $b$ is the size of the batch, $k = \abs{\spA}$ is the size of the target space, and $m = \abs{\spS}$ is the size of the candidate set.
In the case of large $m$, an alternative implementation whose runtime does not depend on $m$ is described in \cref{sec:computational_complexity}.\looseness=-1

\subsection{Baselines}\label{sec:nns_appendix:undirected}

In \cref{fig:nns_more_baselines}, we compare against additional baselines: \begin{itemize}
  \item Both \textsc{TypiClust}~\citep{hacohen2022active} and \textsc{ProbCover}~\citep{yehuda2022active} are recent methods to select points that ``cover'' the data distribution well.
  To maintain comparability between algorithms, we use the same embeddings as for ITL which are re-computed before every new batch selection.
  ITL significantly outperforms \textsc{TypiClust} \& \textsc{ProbCover}, which only attempt to cover $\spS$ well without taking $\spA$ into account (i.e., are ``undirected'').

  \item \cite{mehta2022information} introduced EIG for training neural classification models, which uses the same decision rule as \itl, but approximates the conditional entropy based on the networks' softmax output rather than using a GP approximation.
  We approximate the conditional entropy using a single gradient step of the hallucinated updates on the parameters of the final layer, as mentioned by \cite{mehta2022information}.
  We observe that EIG is not competitive for batch-wise selection (CIFAR-100) since it does not encourage batch diversity.
  Moreover, we observe that EIG is orders of magnitude slower than \itl (since it has to compute $|\spS| \cdot C$ individual gradient steps where $C$ is the number of classes).
  We note that since our datasets are balanced, the AEIG algorithm from \cite{mehta2022information} coincides with EIG.
\end{itemize}
Since, EIG does not have an open-source implementation, we implemented it ourselves following \cite{mehta2022information}.
For \textsc{TypiClust} \& \textsc{ProbCover}, we use the author's implementation.
In the figure, we show that \itl \& \vtl substantially outperform all baselines.\looseness=-1

\begin{figure*}[]
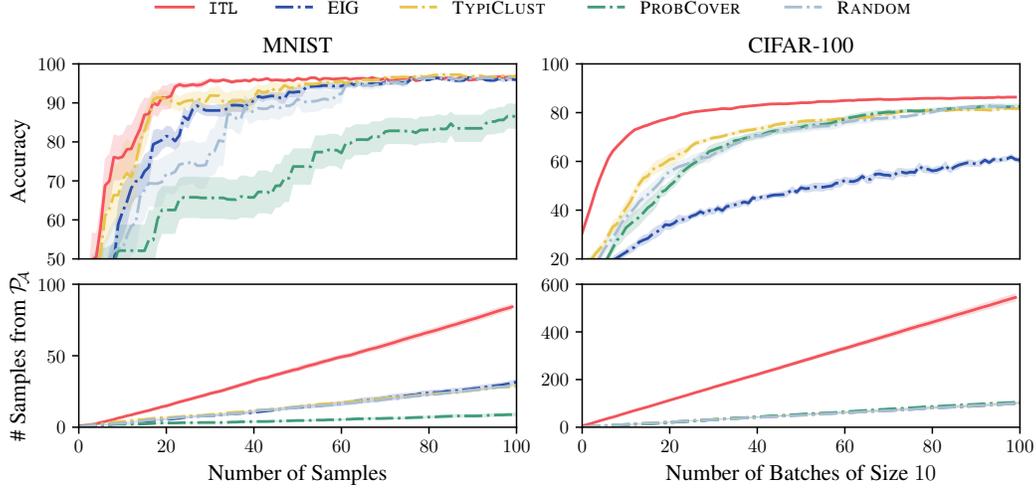

  \incplt[\textwidth]{nns_more_baselines}
  \vspace{-0.5cm}
  \caption{Comparison to baselines for the experiment of \cref{fig:nns}.}
  \label{fig:nns_more_baselines}
\end{figure*}

In the following, we briefly describe other commonly used ``undirected'' decision rules.\looseness=-1

Denote the softmax distribution over labels $i$ at inputs $\vx$ by \begin{align*}
  p_i(\vx; \vthetahat) \propto \exp(f_i(\vx; \vthetahat)).
\end{align*}
The following heuristics computed based on the softmax distribution aim to quantify the ``uncertainty'' about a particular input $\vx$: \begin{itemize}
  \item \textsc{MaxEntropy} \citep{settles2008analysis}: \begin{align*}
    \vx_{n} = \argmax_{\vx \in \spS} \Hsm{p(\vx; \vthetahat_{n-1})}.
  \end{align*}

  \item \textsc{MaxMargin} \citep{scheffer2001active,settles2008analysis}: \begin{align*}
    \vx_{n} = \argmin_{\vx \in \spS} p_1(\vx; \vthetahat_{n-1}) - p_2(\vx; \vthetahat_{n-1})
  \end{align*} where $p_1$ and $p_2$ are the two largest class probabilities.

  \item \textsc{LeastConfidence} \citep{lewis1994sequential,settles2008analysis,hendrycks2017baseline,tamkin2022active}: \begin{align*}
    \vx_{n} = \argmin_{\vx \in \spS} p_1(\vx; \vthetahat_{n-1})
  \end{align*} where $p_1$ is the largest class probability.
\end{itemize}

An alternative class of decision rules aims to select diverse batches by maximizing the distances between points.
Embeddings $\vphi(\vx)$ induce the (Euclidean) embedding distance \begin{align*}
  d_{\vphi}(\vx, \vxp) \defeq \norm{\vphi(\vx) - \vphi(\vxp)}_2.
\end{align*}
Similarly, a kernel $k$ induces the kernel distance \begin{align*}
  d_k(\vx, \vxp) \defeq \sqrt{k(\vx, \vx) + k(\vxp, \vxp) - 2 k(\vx, \vxp)}.
\end{align*}
It is straightforward to see that if $k(\vx, \vxp) = \transpose{\vphi(\vx)} \vphi(\vxp)$, then embedding and kernel distances coincide, i.e., $d_{\vphi}(\vx, \vxp) = d_k(\vx, \vxp)$.\looseness=-1

\begin{itemize}
  \item \textsc{MaxDist} \citep{holzmuller2023framework,yu2010passive,sener2017active,geifman2017deep} constructs the batch by choosing the point with the maximum distance to the nearest previously selected point: \begin{align*}
    \vx_n = \argmax_{\vx \in \spS} \min_{i < n} d(\vx, \vx_i)
  \end{align*}

  \item Similarly, \textsc{k-means++} \citep{holzmuller2023framework} selects the batch via \textsc{k-means++} seeding \citep{arthur2007k,ostrovsky2013effectiveness}.
  That is, the first centroid $\vx_1$ is chosen uniformly at random and the subsequent centroids are chosen with a probability proportional to the square of the distance to the nearest previously selected centroid: \begin{align*}
    \Pr{\vx_n = \vx} \propto \min_{i < n} d(\vx, \vx_i)^2.
  \end{align*}
  When using the loss gradient embeddings from \cref{eq:loss_gradient_embedding}, this decision rule is known as \textsc{BADGE} \citep{ash2020deep}.
\end{itemize}

Finally, we summarize common kernel-based decision rules.
\begin{itemize}
  \item \textsc{Undirected \itl} chooses \begin{align*}
    \vx_n &= \argmax_{\vx \in \spS} \I{\vfsub{\spS}}{y_{\vx}}[\spD_{n-1}] \\
    &= \argmax_{\vx \in \spS} \I{f_{\vx}}{y_{\vx}}[\spD_{n-1}].
  \end{align*}
  This can be shown to be equivalent to \textsc{MaxDet} \citep{holzmuller2023framework} which selects \begin{align*}
    \vx_n = \argmax_{\vx \in \spS} \det{\mKsub{\vx} + \sigma^2 \mI}
  \end{align*} where $\mKsub{\vx}$ denotes the kernel matrix over $\vx_{1:n-1} \cup \{\vx\}$, conditioned on the prior observations $\spD_{n-1}$.

  \item \textsc{UnSa} \citep{lewis1994heterogeneous} which with embeddings $\vphi_{n-1}$ after round $n-1$ corresponds to: \begin{align*}
    \vx_{n} = \argmax_{\vx \in \spS} \sigma_{n-1}^2(\vx) = \argmax_{\vx \in \spS} \norm{\vphi_{n-1}(\vx)}_2^2.
  \end{align*}
  With batch size $b = 1$, \textsc{UnSa} coincides with \textsc{Undirected \itl}.
  When evaluated with gradient embeddings, this acquisition function is similar to previously used ``embedding length'' or ``gradient length'' heuristics \citep{settles2008analysis}.

  \item \textsc{Undirected \vtl} \citep{cohn1993neural} is the special case of \vtl without specified prediction targets (i.e., $\spA = \spS$).
  In the literature, this decision rule is also known as \textsc{Bait} \citep{holzmuller2023framework,ash2021gone}.
\end{itemize}

We compare to the abovementioned decision rules and summarize the results in \cref{fig:nns_undirected}.
We observe that most ``undirected'' decision rules perform worse (and often significantly so) than \textsc{Random}.
This is likely due to frequently selecting points from the support of $\spPS$ which are not in the support of $\spPA$ since the points are ``adversarial examples'' that the model $\vthetahat$ is not trained to perform well on.
In the case of MNIST, the poor performance can also partially be attributed to the well-known ``cold-start problem'' \citep{gao2020consistency}.\looseness=-1

\begin{figure*}[]
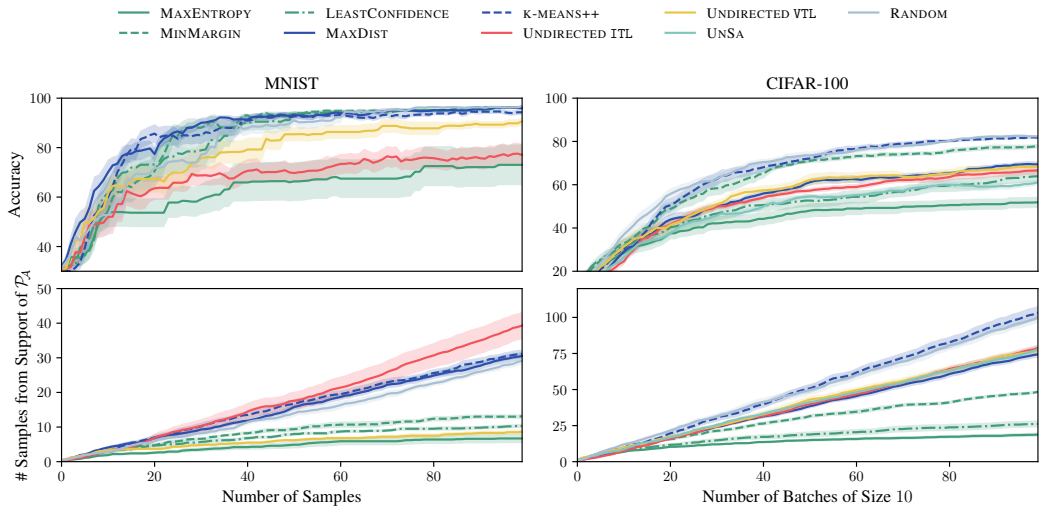

  \incplt[\textwidth]{nns_undirected}
  \vspace{-0.5cm}
  \caption{Comparison of ``undirected'' baselines for the experiment of \cref{fig:nns}. In the MNIST experiment, \textsc{UnSa} and \textsc{Undirected \itl} coincide, and we therefore only plot the latter.}
  \label{fig:nns_undirected}
\end{figure*}

In \cref{fig:nns}, we also compare to the following ``directed'' decision rules: \begin{itemize}
  \item \textsc{CosineSimilarity} \citep{settles2008analysis} selects $\vx_{n} = \argmax_{\vx \in \spS} \angle_{\vphi_{n-1}}(\vx,\spA)$ where \begin{align*}
    \angle_{\vphi}(\vx,\spA) \defeq \frac{1}{\abs{\spA}} \sum_{\vxp \in \spA} \frac{\transpose{\vphi(\vx)} \vphi(\vxp)}{\norm{\vphi(\vx)}_2 \norm{\vphi(\vxp)}_2}.
  \end{align*}

  \item \textsc{InformationDensity} \citep{settles2008analysis} is defined as the multiplicative combination of \textsc{MaxEntropy} and \textsc{CosineSimilarity}: \begin{align*}
    \vx_{n} = \argmax_{\vx \in \spS} \Hsm{p(\vx; \vthetahat_{n-1})} \cdot \parentheses*{\angle_{\vphi_{n-1}}(\vx, \spA)}^\beta
  \end{align*} where $\beta > 0$ controls the relative importance of both terms.
  We set $\beta = 1$ in our experiments.
\end{itemize}

\subsection{Additional experiments}\label{sec:nns_appendix:additional_experiments}

\begin{figure*}[]
  \incplt[\textwidth]{nns_imbalanced_train}
  \vspace{-0.5cm}
  \caption{Imbalanced $\spPS$ experiment.}
  \label{fig:nns_imbalanced_train}
\end{figure*}

\begin{figure*}[]
  \incplt[\textwidth]{nns_imbalanced_test}
  \vspace{-0.5cm}
  \caption{Imbalanced $A \sim \spPA$ experiment.}
  \label{fig:nns_imbalanced_test}
\end{figure*}

We conduct the following additional experiments: \begin{enumerate}
  \item \emph{Imbalanced $\spPS$} (\cref{fig:nns_imbalanced_train}): We artificially remove $80\%$ of the support of $\spPA$ from $\spPS$.
  For example, in case of MNIST, we remove $80\%$ of the images with labels $3$, $6$, and $9$ from $\spPS$.
  This makes the learning task more difficult, as $\spPA$ is less represented in $\spPS$, meaning that the ``targets'' are more sparse.
  The trend of \itl outperforming \ctl which outperforms \textsc{Random} is even more pronounced in this setting.

  \item \emph{Imbalanced $A \sim \spPA$} (\cref{fig:nns_imbalanced_test}): We artificially remove $50\%$ of part of the support of $\spPA$ while generating $A \sim \spPA$ to evaluate the robustness of \itl and \ctl in presence of an imbalanced target space $\spA$.
  Concretely, in case of MNIST, we remove $50\%$ of the images with labels $3$ and $6$ from $A$.
  In case of CIFAR-100, we remove $50\%$ of the images with labels $\{0, \dots, 4\}$ from $A$.
  We still observe the same trends as in the other experiments.

  \item \emph{\vtl \& choice of $k$} (\cref{fig:nns_vtl}): We observe that \vtl performs almost as well as \itl.
  Additionally, we evaluate the effect of the number of points $k$ at which the decision rule is evaluated.
  Not surprisingly, we observe that the performance of \itl, \vtl, and \ctl improves with larger $k$.

  \item \emph{Choice of $m$} (\cref{fig:nns_subsampled_target_frac}): Next, we evaluate the choice of $m$, i.e., the size of the target space $\spA$ relative to the number $M$ of candidate points $A \sim \spPA$.
  We write $p = m / M$.
  We generally observe that a larger $p$ leads to better performance (with $p=1$ being the best choice).
  However, it appears that a smaller $p$ can be beneficial with respect to accuracy when a large number of batches are selected.
  We believe that this may be because a smaller $p$ improves the diversity between selected batches.

  \item \emph{Choice of $M$} (\cref{fig:nns_n_init}): Finally, we evaluate the choice of $M$, i.e., the size of $A \sim \spPA$.
  Not surprisingly, we observe that the performance of \itl improves with larger $M$.
\end{enumerate}

\begin{figure*}[]
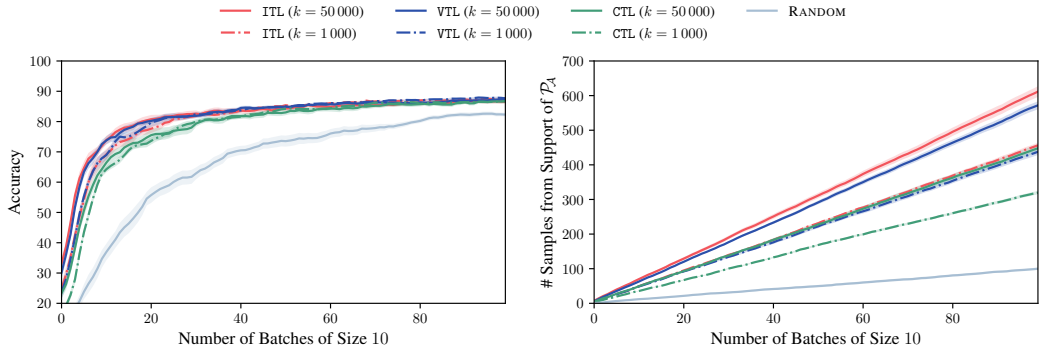

  \incplt[\textwidth]{nns_vtl}
  \vspace{-0.5cm}
  \caption{Performance of \vtl \& choice of $k$ in the CIFAR-100 experiment.}
  \label{fig:nns_vtl}
\end{figure*}

\begin{figure*}[]
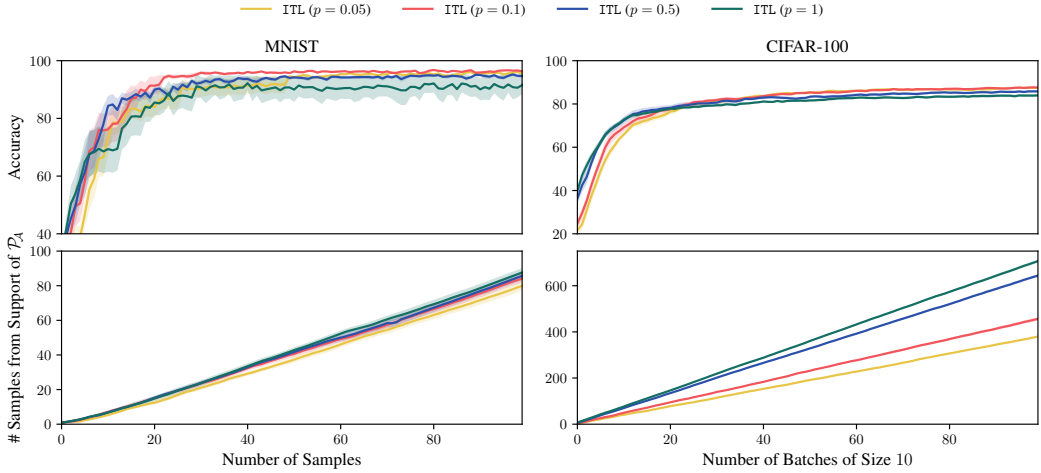

  \incplt[\textwidth]{nns_subsampled_target_frac}
  \vspace{-0.5cm}
  \caption{Evaluation of the choice of $m$ relative to the size $M$ of $A \sim \spPA$. Here, $p = m / M$.}
  \label{fig:nns_subsampled_target_frac}
\end{figure*}

\begin{figure*}[]
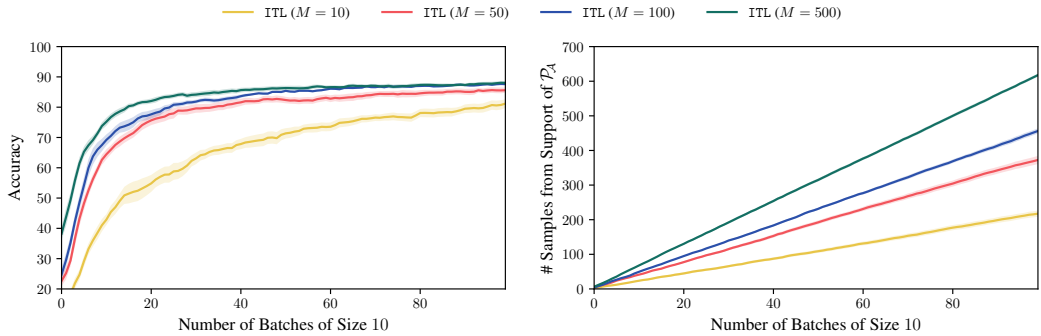

  \incplt[\textwidth]{nns_n_init}
  \vspace{-0.5cm}
  \caption{Evaluation of the choice of $M$, i.e., the size of $A \sim \spPA$, in the CIFAR-100 experiment.}
  \label{fig:nns_n_init}
\end{figure*}

\subsection{Ablation study of noise standard deviation $\rho$}\label{sec:nns_appendix:ablation_sigma}

In \cref{table:nn_sigma_ablation}, we evaluate the CIFAR-100 experiment with different noise standard deviations $\rho$.
We observe that the performance of batch selection via conditional embeddings drops (mostly for the less numerically stable gradient embeddings) if $\rho$ is too small, since this leads to numerical inaccuracies when computing the conditional embeddings.
Apart from this, the effect of $\rho$ is negligible.\looseness=-1

\begin{table*}[]
  \caption{Ablation study of noise standard deviation $\rho$ in the CIFAR-100 experiment. We list the accuracy after $100$ rounds per decision rule, with its standard error over $10$ random seeds. ``(top-$b$)'' denotes variants where batches are selected by taking the top-$b$ points according to the decision rule rather than using batch selection via conditional embeddings. Shown in \textbf{bold} are the best performing decision rules, and shown in \textit{italics} are results due to numerical instability.}
  \label{table:nn_sigma_ablation}
  \vskip 0.15in
  \begin{center}
  \begin{tabular}{@{}lllll@{}}
    \toprule
    $\rho$ & $0.0001$ & $0.01$ & $1$ & $100$ \\
    \midrule
    \gitl & $\mathit{78.26\pm1.40}$ & $\mathit{79.12\pm1.19}$ & $\mathbf{87.16\pm0.29}$ & $\mathbf{87.18\pm0.28}$ \\
    \litl & $\mathbf{87.52\pm0.48}$ & $\mathbf{87.52\pm0.41}$ & $\mathbf{87.53\pm0.35}$ & $86.47\pm0.22$ \\
    \gctl & $\mathit{58.68\pm2.11}$ & $\mathit{81.44\pm1.04}$ & $86.52\pm0.44$ & $\mathbf{86.92\pm0.56}$ \\
    \lctl & $\mathbf{86.40\pm0.71}$ & $\mathbf{86.38\pm0.75}$ & $86.00\pm0.69$ & $84.78\pm0.39$ \\
    \gitl (top-$b$) & $85.84\pm0.54$ & $85.92\pm0.52$ & $85.84\pm0.54$ & $85.55\pm0.46$ \\
    \litl (top-$b$) & $85.44\pm0.58$ & $85.46\pm0.54$ & $85.44\pm0.59$ & $85.29\pm0.36$ \\
    \gctl (top-$b$) & $82.27\pm0.67$ & $82.27\pm0.67$ & $82.27\pm0.67$ & $82.27\pm0.67$ \\
    \lctl (top-$b$) & $80.73\pm0.68$ & $80.73\pm0.68$ & $80.73\pm0.68$ & $80.73\pm0.68$ \\
    \textsc{BADGE} & $83.24\pm0.60$ & $83.24\pm0.60$ & $83.24\pm0.60$ & $83.24\pm0.60$ \\
    \textsc{InformationDensity} & $79.24\pm0.51$ & $79.24\pm0.51$ & $79.24\pm0.51$ & $79.24\pm0.51$ \\
    \textsc{Random} & $82.49\pm0.66$ & $82.49\pm0.66$ & $82.49\pm0.66$ & $82.49\pm0.66$ \\
    \bottomrule
  \end{tabular}
  \end{center}
  \vskip -0.1in
\end{table*}

%% file: backmatter/I_additional_safe_bo_experiments.tex
\clearpage\section{Additional Safe BO Experiments \& Details}\label{sec:safe_bo_appendix}

In \cref{sec:safe_bo_appendix:thompson_sampling}, we discuss the use of stochastic target spaces in the safe BO setting.
We provide a comprehensive overview of prior works in \cref{sec:safe_bo_appendix:comparison} and an additional experiment highlighting that \itl, unlike \textsc{SafeOpt}, is able to ``jump past local barriers'' in \cref{sec:safe_bo_appendix:jumping_past_local_barriers}.
In \cref{sec:safe_bo_appendix:details}, we provide details on the experiments from \cref{fig:safe_bo}.\looseness=-1

\subsection{A More Exploitative Stochastic Target Space}\label{sec:safe_bo_appendix:thompson_sampling}

Alternatively to the target space~$\spA_n$ which comprises all potentially optimal points, we evaluate the stochastic target space \begin{align}
    {\spPA}_n(\cdot) = \Prsm{\argmax_{\vx \in \spX : g(\vx) \geq 0} f(\vx) = \cdot \mid \spD_n} \label{eq:safe_bo_stochastic_target_space}
\end{align} which effectively weights points in $\spA_n$ according to how likely they are to be the safe optimum, and is therefore more exploitative than the uniformly-weighted target space discussed so far.
Samples from ${\spPA}_n$ can be obtained efficiently via Thompson sampling~\citep{thompson1933likelihood,russo2018tutorial}.
Observe that ${\spPA}_n$ is supported precisely on the set of potential maximizers $\spA_n$.
We provide a formal analysis of stochastic target spaces in \cref{sec:generalizations:roi}.
Whether transductive active learning with $\spA_n$ or ${\spPA}_n$ performs better is task-dependent, as we will see in the following.\looseness=-1

Note that performing \itl with this target space is analogous to output-space entropy search \citep{wang2017max}.
Samples from ${\spPA}_n$ can be obtained via Thompson sampling \citep{thompson1933likelihood,russo2018tutorial}.
That is, in iteration $n+1$, we sample $K \in \Nat$ independent functions $\smash{f^{(j)} \sim f \mid \spD_n}$ from the posterior distribution and select $K$ points $\smash{\vx^{(1)}, \dots, \vx^{(K)}}$ which are a safe maximum of $\smash{f^{(1)}, \dots, f^{(K)}}$, respectively.\looseness=-1

\paragraph{Experiments}

In \cref{fig:safe_bo_ts}, we contrast the performance of \itl with ${\spPA}_n$ to the performance of \itl with the exact target space $\spA_n$.
We observe that their relative performance is instance dependent: in tasks that require more difficult expansion, \itl with $\spA_n$ converges faster, whereas in simpler tasks (such as the 2d experiment), \itl with ${\spPA}_n$ converges faster.
We compare against the \textsc{GoOSE} algorithm \citep{turchetta2019safe} which is a heuristic extension of \textsc{SafeOpt} that explores more greedily in directions of (assumed) high reward (cf. \cref{sec:safe_bo_appendix:comparison:goose}).
\textsc{GoOSE} suffers from the same limitations as \textsc{SafeOpt}, which were highlighted in \cref{sec:safe_bo}, and additionally is limited by its heuristic approach to expansion which fails in the 1d task and safe controller tuning task.
Analogously to our experiments with \textsc{SafeOpt}, we also compare against \textsc{Oracle GoOSE} which has oracle knowledge of the true Lipschitz constants.\looseness=-1

\begin{figure*}[t]
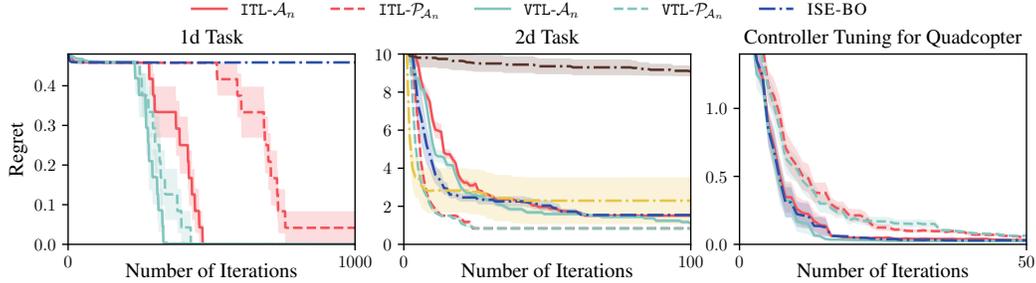

    \incplt[\textwidth]{safe_bo_ts}
    \vspace{-0.5cm}
    \caption{We perform the tasks of \cref{fig:safe_bo} using Thompson sampling to evaluate the stochastic target space ${\spPA}_n$. We additionally compare to \textsc{GoOSE} (cf. \cref{sec:safe_bo_appendix:comparison:goose}) and \textsc{ISE-BO} (cf. \cref{sec:safe_bo_appendix:comparison:ise}).}
    \label{fig:safe_bo_ts}
\end{figure*}

The different behaviors of \itl with $\spA_n$ and ${\spPA}_n$, respectively, as well as \textsc{SafeOpt} and \textsc{GoOSE} are illustrated in \cref{fig:safe_bo_2d_islands_samples}.
We observe that \itl with $\spA_n$ and \textsc{SafeOpt} expand the safe set more ``uniformly'' since the set of potential maximizers encircles the true safe set.\footnote{This is because typically, there will always remain points in $\widehat{\spS}_n \setminus \spS_n$ of which the safety cannot be fully determined, and since, they cannot be observed, it can also not be ruled out that they have high objective value.}
Intuitively, this is because the set of potential maximizers \emph{conservatively} captures migh points might be safe and optimal.
In contrast, \itl with ${\spPA}_n$ and \textsc{GoOSE} focus exploration and expansion in those regions where the objective is likely to be high.\looseness=-1

\begin{figure}[H]
  \centering
  \includesvg[width=\columnwidth]{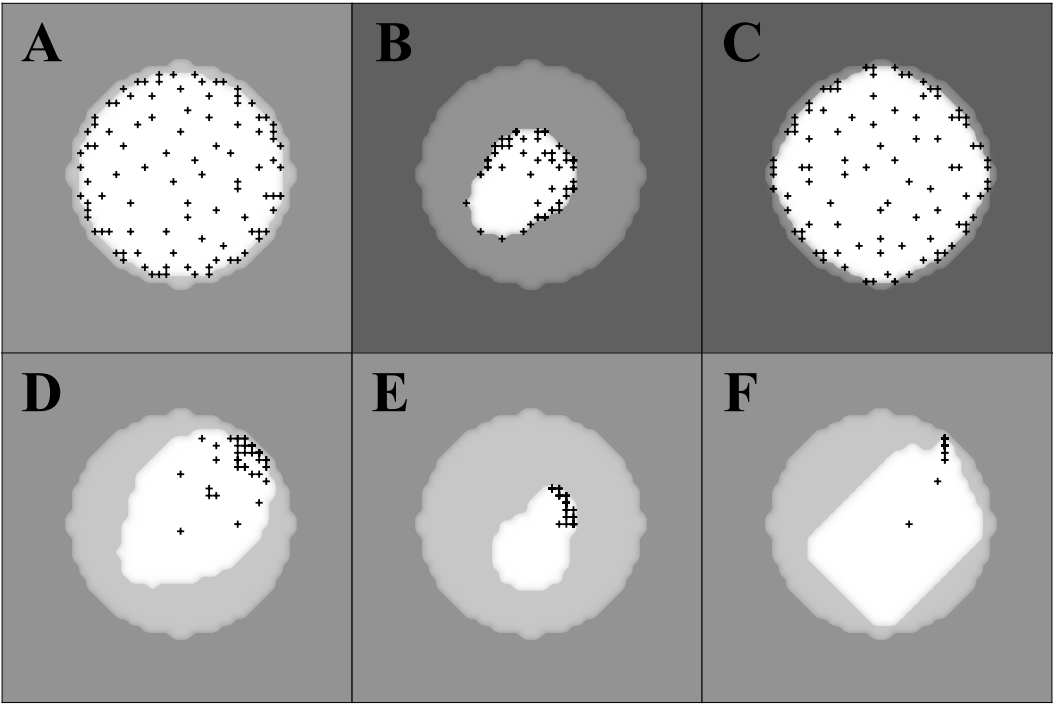}
  \caption{The first $100$ samples of \textbf{(A)} \itl with $\spA_n$, \textbf{(B)} \textsc{SafeOpt}, \textbf{(C)} \textsc{Oracle SafeOpt}, \textbf{(D)} \itl with ${\spPA}_n$, \textbf{(E)} \textsc{GoOSE}, \textbf{(F)} \textsc{Oracle GoOSE}. The white region denotes the pessimistic safe set $\spS_{100}$, the light gray region denotes the true safe set $\opt{\spS}$ (i.e., the ``island''), and the darker gray regions denotes unsafe points (i.e., the ``ocean'').}
  \label{fig:safe_bo_2d_islands_samples}
\end{figure}

\subsection{Detailed Comparison with Prior Works}\label{sec:safe_bo_appendix:comparison}

The most widely used method for Safe BO is \textsc{SafeOpt}~\citep{sui2015safe,berkenkamp2021bayesian} which keeps track of separate candidate sets for expansion and exploration and uses \textsc{UnSa} to pick one of the candidates in each round.
Treating expansion and exploration separately, sampling is directed towards expansion in \emph{all} directions --- even those that are known to be suboptimal.
The safe set is expanded based on a Lipschitz constant of $\opt{g}$, which is assumed to be known.
In most real-world settings, this constant is unknown and has to be estimated using the GP.
This estimate is generally conservative and results in suboptimal performance.
To this end, \cite{berkenkamp2016safe} proposed \textsc{Heuristic SafeOpt} which relies solely on the confidence intervals of~$g$ to expand the safe set, but lacks convergence guarantees.
More recently, \cite{bottero2022information} proposed ISE which queries parameters from $\spS_n$ that yield the most ``information'' about the safety of another parameter in $\spX$.
Hence, ISE focuses solely on the expansion of the safe set $\spS_n$ and does not take into account the objective $f$.
In practice, this can lead to significantly worse performance on the simplest of problems~(cf.~\cref{fig:safe_bo}).
In contrast, \itl balances expansion of and exploration within the safe set.
Furthermore, ISE does not have known convergence guarantees of the kind of \cref{thm:safebo_main}.
In parallel independent work, \cite{bottero2024information} proposed a combination of ISE and max-value entropy search \citep{wang2017max} for which they derive a similar guarantee to \cref{thm:safebo_main}.\footnote{We provide an empirical evaluation in \cref{sec:safe_bo_appendix:comparison:ise}.}
Similar to \textsc{SafeOpt}, their method aims to expand the safe set in all directions including those that are known to be suboptimal.
In contrast, \itl directs expansion only towards potentially optimal regions.\looseness=-1

In the 1d task and quadcopter experiment (cf. \cref{fig:safe_bo}), we observe that \textsc{SafeOpt} and even \textsc{Oracle SafeOpt} converge significantly slower than \itl to the safe optima.
We believe this is due to their conservative Lipschitz-continuity/global smoothness-based expansion, as opposed to \itl's expansion, which adapts to the local smoothness of the constraints.
\textsc{Heuristic SafeOpt}, which does not rely on the Lipschitz constant for expansion, does not efficiently expand the safe set due to its heuristic that only considers single-step expansion. This is especially the case for the 1d task. %
Furthermore, in the 2d task, we notice the suboptimality of ISE since it does not take into account the objective, and purely aims to expand the safe set.
\itl, on the other hand, balances expansion and exploration.\looseness=-1

\subsubsection{\textsc{SafeOpt}}

\textsc{SafeOpt} \citep{sui2015safe,berkenkamp2021bayesian} is a well-known algorithm for Safe BO.\looseness=-1

\paragraph{Lipschitz-based expansion}

\textsc{SafeOpt} expands the set of known-to-be safe points by assuming knowledge of an upper bound $L_i$ to the Lipschitz constant of the unknown constraints $\opt{g}_i$.\footnote{Recall that due to the assumption that $\norm{\opt{g}_i}_k < \infty$, $\opt{g}_i$ is indeed Lipschitz continuous.}
In each iteration, the (pessimistic) safe set $\spS_n$ is updated to include all points which can be reached safely (with respect to the Lipschitz continuity) from a known-to-be-safe point $\vx \in \spS_n$.
Formally, \begin{align}
  \spS_{n}^{\textsc{SafeOpt}} \defeq \begin{multlined}[t]
    \bigcup_{\vx \in \spS_{n-1}^{\textsc{SafeOpt}}} \{\vxp \in \spX \mid \\ \text{$l_{n,i}(\vx) - L_i \norm*{\vx - \vxp}_2 \geq 0$ for all $i \in \spI_s$}\}.
  \end{multlined} \label{eq:safeopt_safe_set}
\end{align}
The expansion of the safe set is illustrated in \cref{fig:safe_opt_expanders}.\looseness=-1

We remark two main limitations of this approach.
First, the Lipschitz constant is an additional safety critical hyperparameter of the algorithm, which is typically not known.
The RKHS assumption (cf. \cref{asm:rkhs}) induces an assumption on the Lipschitz continuity, however, the worst-case a-priori Lipschitz constant is typically very large, and prohibitive for expansion.
Second, the Lipschitz constant is global property of the unknown function, meaning that it does not adapt to the local smoothness.
For example, a constraint may be ``flat'' in one direction (permitting straightforward expansion) and ``steep'' in another direction (requiring slow expansion).
Furthermore, the Lipschitz constant is constant over time, whereas \itl is able to adapt to the local smoothness and reduce the (induced) Lipschitz constant over time.\looseness=-1

\paragraph{Undirected expansion}

\textsc{SafeOpt} addresses the trade-off between expansion and exploration by focusing learning on two different sets.
First, the set of \emph{maximizers} \begin{align*}
  \spM_n^{\textsc{SafeOpt}} \defeq \begin{multlined}[t]
    \{\vx \in \spS_n^{\textsc{SafeOpt}} \mid \\ u_{n,f}(\vx) \geq \max_{\vxp \in \spS_n^{\textsc{SafeOpt}}} l_{n,f}(\vx)\}
  \end{multlined}
\end{align*} which contains all \emph{known-to-be-safe} points which are potentially optimal.
Note that if $\spS_n^{\textsc{SafeOpt}} = \spS_n$ then $\spM_n^{\textsc{SafeOpt}} \subseteq \spA_n$ since $\spA_n$ contains points which are potentially optimal and potentially safe \emph{but possibly unsafe}.\looseness=-1

To facilitate expansion, for each point $\vx \in \spS_n$, the algorithm considers a set of \emph{expanding points} \begin{align*}
  \spF_n^{\textsc{SafeOpt}}(\vx) \defeq \begin{multlined}[t]
    \{\vxp \in \spX \setminus \spS_n^{\textsc{SafeOpt}} \mid \\ \hspace{-0.5cm}\text{$u_{n,i}(\vx) - L_i \norm*{\vx - \vxp}_2 \geq 0$ for all $i \in \spI_s$}\}
  \end{multlined}
\end{align*}
A point is expanding if it is unsafe initially and can be (optimistically) deduced as safe by observing $\vx$.
The set of \emph{expanders} corresponds to all known-to-be-safe points which optimistically lead to expansion of the safe set: \begin{align*}
  \spG_n^{\textsc{SafeOpt}} \defeq \{\vx \in \spS_n \mid |\spF_n(\vx)| > 0\}.
\end{align*}
That is, an expander is a safe point $\vx$ which is ``close'' to at least one expanding point $\vxp$.
Observe that here, we start with a safe $\vx$ and then find a close and potentially safe $\vxp$ using the Lipschitz-property of the constraint function.
Thus, the set of expanding points is inherently limited by the assumed Lipschitzness (cf. \cref{fig:safe_opt_expanders}), and generally a subset of the potential expanders $\spE_n$ (cf. \cref{eq:potential_expanders}):\looseness=-1

\begin{lemma}\label{lem:safeopt_expanding_points}
  For any $n \geq 0$, if $\spS_n^{\textsc{SafeOpt}} = \spS_n$ then \begin{align*}
    \bigcup_{\vx \in \spS_n} \spF_n^{\textsc{SafeOpt}}(\vx) \subseteq \spE_n.
  \end{align*}
\end{lemma}
\begin{proof}
  Without loss of generality, we consider the case where $\spI_s = \{i\}$.
  We have \begin{align*}
    \spE_n = \widehat{\spS}_n \setminus \spS_n = \{\vx \in \spX \setminus \spS_n \mid u_{n,i}(\vx) \geq 0\}.
  \end{align*}
  The result follows directly by observing that ${L_i \norm*{\vx - \vxp}_2 \geq 0}$.
\end{proof}

\begin{figure}[t]
  \centering
  \includegraphics[width=0.6\columnwidth]{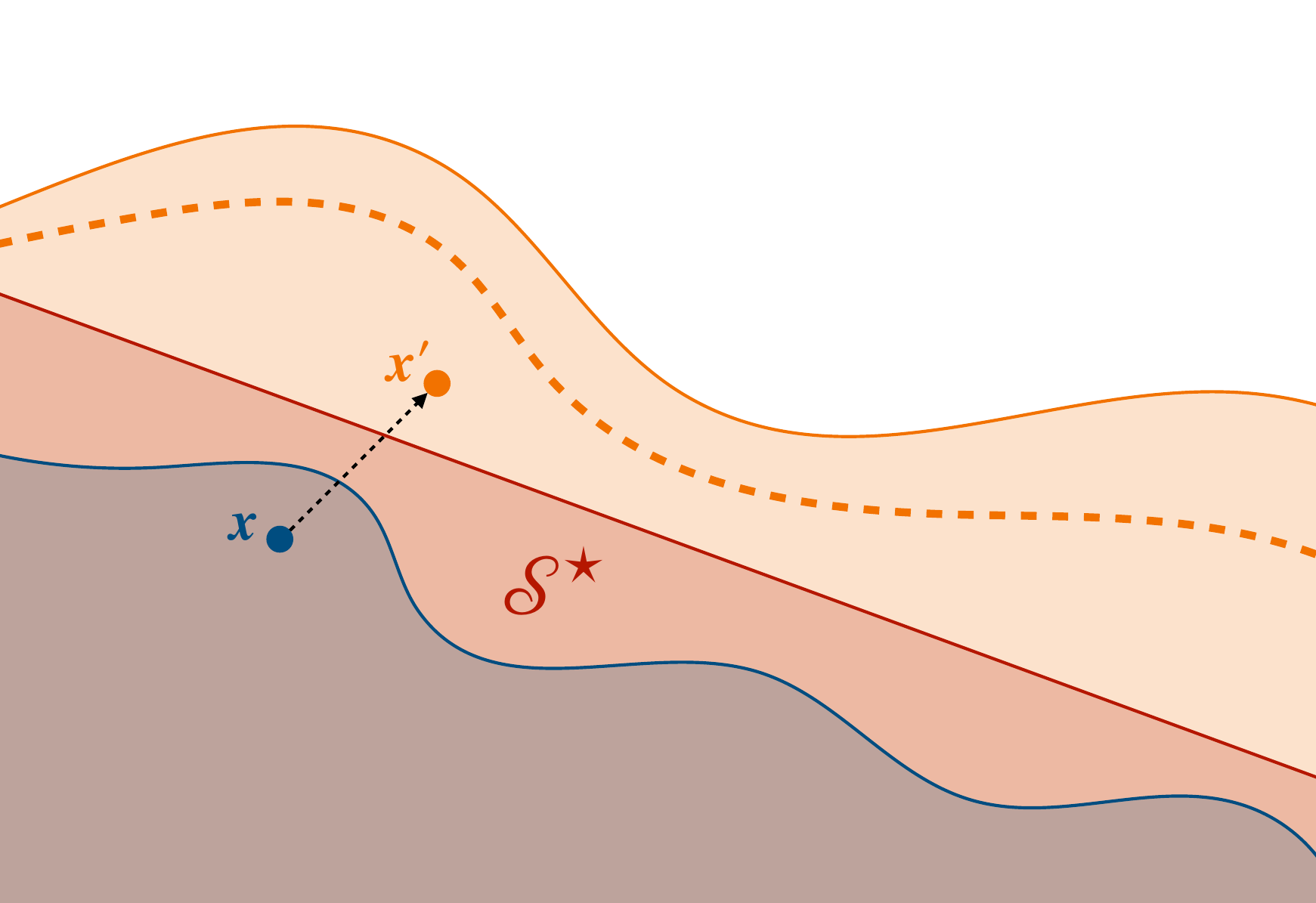}
  \caption{Illustration of the expansion of the safe set à la \textsc{SafeOpt}.
  Here, the blue region denotes the pessimistic safe set $\spS$, the red region denotes the true safe set $\opt{\spS}$, and the orange region denotes the optimistic safe set $\widehat{\spS}$.
  Whereas \itl learns about the point $\vxp$ \emph{directly}, \textsc{SafeOpt} expands the safe set using the reduction of uncertainty at $\vx$, and then extrapolating using the Lipschitz constant (cf. \cref{eq:safeopt_safe_set}).
  The dashed orange line denotes the expanding points of \textsc{SafeOpt} which under-approximate the optimistic safe set of \itl (cf. \cref{lem:safeopt_expanding_points}).
  Thus, \itl may even learn about points in $\widehat{\spS}$ which are ``out of reach'' for \textsc{SafeOpt}.}
  \label{fig:safe_opt_expanders}
\end{figure}

\textsc{SafeOpt} then selects $\vx_{n+1}$ according to uncertainty sampling \emph{within} the maximizers and expanders: $\spM_n^{\textsc{SafeOpt}}~\cup~\spG_n^{\textsc{SafeOpt}}$.
We remark that due to the separate handling of expansion and exploration, \textsc{SafeOpt} expands the safe set in \emph{all} directions --- even those that are known to be suboptimal.
In contrast, \itl only expands the safe set in directions that are potentially optimal by balancing expansion and exploration through the single set of potential maximizers $\spA_n$.\looseness=-1

\paragraph{Based on uncertainty sampling}

As mentioned in the previous paragraph, \textsc{SafeOpt} selects as next point the maximizer/expander with the largest prior uncertainty.\footnote{The use of uncertainty sampling for safe sequential decision-making goes back to \cite{schreiter2015safe} and \cite{sui2015safe}.}
In contrast, \itl selects the point within $\spS_n$ which minimizes the posterior uncertainty within $\spA_n$.
Note that the two approaches are not identical as typically $\spM_n^{\textsc{SafeOpt}}~\cup~\spG_n^{\textsc{SafeOpt}} \subset \spS_n^{\textsc{SafeOpt}}$ and $\spA_n \not\supseteq \spS_n$.\looseness=-1

We show empirically in \cref{sec:gps} that depending on the kernel choice (i.e., the smoothness assumptions), uncertainty sampling within a given target space neglects higher-order information that can be attained by sampling outside the set.
This can be seen even more clearly when considering linear functions, in which case points outside the maximizers and expanders can be equally informative as points inside.\looseness=-1

Finally, note that the set of expanders is constructed ``greedily'', i.e., only considering \emph{single-step} expansion.
This is necessitated as the inference of safety is based on single reference points.
Instead, \itl directly quantifies the information gained towards the points of interest without considering intermediate reference points.\looseness=-1

\paragraph{Requires homoscedastic noise}

\textsc{SafeOpt} imposes a homoscedasticity assumption on the noise which is an artifact of the analysis of uncertainty sampling.
It is well known that in the presence of heteroscedastic noise, one has to distinguish epistemic and aleatoric uncertainty.
Uncertainty sampling fails because it may continuously sample a high variance point where the variance is dominated by aleatoric uncertainty, potentially missing out on reducing epistemic uncertainty at points with small aleatoric uncertainty.
In contrast, maximizing mutual information naturally takes into account the two sources of uncertainty, preferring those points where epistemic uncertainty is large and aleatoric uncertainty is small (cf. \cref{sec:proofs:undirected_itl}).\looseness=-1

\paragraph{Suboptimal reachable safe set}

\cite{sui2015safe} and \cite{berkenkamp2021bayesian} show that \textsc{SafeOpt} converges to the optimum within the closure $\bar{\spR}_\epsilon^{\textsc{SafeOpt}}(\spS_0)$ of \begin{align*}
  \spR_\epsilon^{\textsc{SafeOpt}}(\spS) \defeq \begin{multlined}[t]
    \spS \cup \{\vx \in \spX \mid \text{$\exists \vxp \in \spS$ such that} \\ \hspace{-1.5cm}\text{$\opt{f}_i(\vxp) - (L_i \norm*{\vx - \vxp}_2 + \epsilon) \geq 0$ for all $i \in \spI_s$}\}.
  \end{multlined}
\end{align*}
Note that analogously to the expansion of the safe set, the ``expansion'' of the reachable safe set is based on ``inferring safety'' through a reference point in $\spS$ and using Lipschitz continuity.
This is opposed to the reachable safe set of \itl~(cf. \cref{defn:reachable_safe_set}).

We remark that under the additional assumption that a Lipschitz constant is known, \itl can easily be extended to expand its safe set based on the kernel \emph{and} the Lipschitz constant, resulting in a strictly larger reachable safe set than \textsc{SafeOpt}.
We leave the concrete formalization of this extension to future work.
Moreover, we do not evaluate this extension in our experiments, as we observe that even without the additional assumption of a Lipschitz constant, \itl outperforms \textsc{SafeOpt} in practice.\looseness=-1

\subsubsection{\textsc{Heuristic SafeOpt}}

\cite{berkenkamp2016safe} also implement a heuristic variant of \textsc{SafeOpt} which does not assume a known Lipschitz constant.
This heuristic variant uses the same (pessimistic) safe sets $\spS_n$ as \itl.
The set of maximizers is identical to \textsc{SafeOpt}.
As expanders, the heuristic variant considers all safe points $\vx \in \spS_n$ that if $\vx$ were to be observed next with value $\vu_n(\vx)$ lead to $|\spS_{n+1}| > |\spS_n|$.
We refer to this set as $\spG_n^{\textsc{H-SafeOpt}}$.
The next point is then selected by uncertainty sampling within $\spM_n^{\textsc{SafeOpt}} \cup \spG_n^{\textsc{H-SafeOpt}}$.\looseness=-1

The heuristic variant shares some properties with \textsc{SafeOpt}, such that it is based on uncertainty sampling, not adapting to heteroscedastic noise, and separate notions of maximizers and expanders (leading to an ``undirected'' expansion of the safe set).
Note that there are no known convergence guarantees for heuristic \textsc{SafeOpt}.
Importantly, note that similar to \textsc{SafeOpt} the set of expanders is constructed ``greedily'', and in particular, does only take into account \emph{single-step} expansion.
In contrast, an objective such as \itl which quantifies the ``information gained towards expansion'' also actively seeks out \emph{multi-step} expansion.\looseness=-1

\subsubsection{\textsc{GoOSE}}\label{sec:safe_bo_appendix:comparison:goose}

To address the ``undirected'' expansion of \textsc{SafeOpt} discussed in the previous section, \cite{turchetta2019safe} proposed \emph{goal-oriented safe exploration} (\textsc{GoOSE}).
\textsc{GoOSE} extends any unsafe BO algorithm (which we subsequently call an oracle) to the safe setting.
In our experiments, we evaluate \textsc{GoOSE-UCB} which uses UCB as oracle and which is also the variant studied by \cite{turchetta2019safe}.
In the following, we assume for ease of notation that $\spI_s = \{c\}$.\looseness=-1

Given the oracle proposal $\opt{\vx}$, \textsc{GoOSE} first determines whether $\opt{\vx}$ is safe.
If $\opt{\vx}$ is safe, $\opt{\vx}$ is queried next.
Otherwise, \textsc{GoOSE} first learns about the safety of $\opt{\vx}$ by querying ``expansionist'' points until the oracle's proposal is determined to be either safe or unsafe.\looseness=-1

\textsc{GoOSE} expands the safe set identically to \textsc{SafeOpt} according to \cref{eq:safeopt_safe_set}.
In the context of \textsc{GoOSE}, $\spS_n^{\textsc{SafeOpt}}$ is called the \emph{pessimistic safe set}.
To determine that a point cannot be deduced as safe, \textsc{GoOSE} also keeps track of a Lipschitz-based \emph{optimistic safe set}: \begin{align*}
  \widehat{\spS}_{n,\epsilon}^{\textsc{GoOSE}} \defeq \begin{multlined}[t]
    \bigcup_{\vx \in \spS_{n-1}^{\textsc{SafeOpt}}} \{\vxp \in \spX \mid \\ \text{$u_{n,c}(\vx) - L_c \norm*{\vx - \vxp}_2 - \epsilon \geq 0$}\}.
  \end{multlined}
\end{align*}
We summarize the algorithm in \cref{alg:goose} where we denote by $\spO(\spX)$ the oracle proposal over the domain $\spX$.\looseness=-1

\begin{algorithm}[htb]
  \caption{\textsc{GoOSE}}
  \label{alg:goose}
\begin{algorithmic}
  \STATE {\bfseries Given:} Lipschitz constant~$L_c$, prior model~$\{f, g_c\}$, oracle~$\spO$, and precision~$\epsilon$
  \STATE Set initial safe set $\spS_0^{\textsc{SafeOpt}}$ based on prior
  \STATE $\widehat{\spS}_{n,\epsilon}^{\textsc{GoOSE}} \gets \spX$
  \STATE $n \gets 0$
  \FOR{$k$ from $1$ to $\infty$}
    \STATE $\opt{\vx}_k \gets \spO(\widehat{\spS}_{n,\epsilon}^{\textsc{GoOSE}})$
    \WHILE{$\opt{\vx}_k \not\in \spS_{n}^{\textsc{SafeOpt}}$}
      \STATE Observe ``expansionist'' point $\vx_{n+1}$, set $n \gets n + 1$, and update model and safe sets
    \ENDWHILE
    \STATE Observe $\opt{\vx}_k$, set $n \gets n + 1$, and update model and safe sets
  \ENDFOR
\end{algorithmic}
\end{algorithm}

It remains to discuss the heuristic used to select the ``expansionist'' points.
\textsc{GoOSE} considers all points $\vx \in \spS_n^{\textsc{SafeOpt}}$ with confidence bands of size larger than the accuracy $\epsilon$, i.e., \begin{align*}
  \spW_{n,\epsilon}^{\textsc{GoOSE}} \defeq \{\vx \in \spS_{n}^{\textsc{SafeOpt}} \mid u_{n,c}(\vx) - l_{n,c}(\vx) > \epsilon\}.
\end{align*}
Which of the points in this set is evaluated depends on a set of learning targets $\spA_{n,\epsilon}^{\textsc{GoOSE}} \defeq \widehat{\spS}_{n,\epsilon}^{\textsc{GoOSE}} \setminus \spS_{n}^{\textsc{SafeOpt}}$ akin to the ``potential expanders''~$\spE_n$ (cf. \cref{eq:potential_expanders}), to each of which we assign a priority $h(\vx)$.
When $h(\vx)$ is large, this indicates that the algorithm is prioritizing to determine whether $\vx$ is safe.
We use as heuristic the negative $\ell_1$-distance between $\vx$ and $\opt{\vx}$.
\textsc{GoOSE} then considers the set of \emph{potential immediate expanders} \begin{align*}
  \spG_{n,\epsilon}^{\textsc{GoOSE}}(\alpha) \defeq \begin{multlined}[t]
    \{\vx \in \spW_{n,\epsilon}^{\textsc{GoOSE}} \mid \text{$\exists \vxp \in \spA_{n,\epsilon}^{\textsc{GoOSE}}$ with} \\ \hspace{-1.3cm}\text{priority $\alpha$ such that $u_{n,c}(\vx) - L_c \norm*{\vx - \vxp}_2 \geq 0$}\}.
  \end{multlined}
\end{align*}
The ``expansionist'' point selected by \textsc{GoOSE} is then any point in $\spG_{n,\epsilon}^{\textsc{GoOSE}}(\opt{\alpha})$ where $\opt{\alpha}$ denotes the largest priority such that $|\spG_{n,\epsilon}^{\textsc{GoOSE}}(\opt{\alpha})| > 0$.\looseness=-1

We observe empirically that the sample complexity of \textsc{GoOSE} is not always better than that of \textsc{SafeOpt}.
Notably, the expansion of the safe set is based on a ``greedy'' heuristic.
Moreover, determining whether a single oracle proposal $\opt{\vx}$ is safe may take significant time.
Consider the (realistic) example where the prior is uniform, and UCB proposes a point which is far away from the safe set and suboptimal.
\textsc{GoOSE} will typically attempt to derive the safety of the proposed point until the uncertainty at \emph{all} points within $\spS_0^{\textsc{SafeOpt}}$ is reduced to $\epsilon$.\footnote{This is because the proposed point typically remains in the optimistic safe set when it is sufficiently far away from the pessimistic safe set.}
Thus, \textsc{GoOSE} can ``waste'' a significant number of samples, aiming to expand the safe set towards a known-to-be suboptimal point.
In larger state spaces, due to the greedy nature of the expansion strategy, this can lead to \textsc{GoOSE} being effectively stuck at a suboptimal point for a significant number of rounds.\looseness=-1

\subsubsection{ISE and ISE-BO}\label{sec:safe_bo_appendix:comparison:ise}

Recently, \cite{bottero2022information} proposed an information-theoretic approach to efficiently expand the safe set which they call \emph{information-theoretic safe exploration} (ISE).
Specifically, they choose the next action $\vx_{n}$ by approximating \begin{align*}
  \argmax_{\vx \in \spS_{n-1}} \underbrace{\max_{\vxp \in \spX} \I{\Ind{g_{\vxp} \geq 0}}{y_{\vx}}[\spD_{n-1}]}_{\alpha^{\mathrm{ISE}}(\vx)}. \tag{ISE}
\end{align*}
In a parallel independent work, \cite{bottero2024information} extended ISE to the Safe BO problem where they propose to choose $\vx_n$ according to \begin{align*}
  \argmax_{\vx \in \spS_{n-1}} \max \{\alpha^{\mathrm{ISE}}(\vx), \alpha^{\mathrm{MES}}(\vx)\} \tag{ISE-BO}
\end{align*} where $\alpha^{\mathrm{MES}}$ denotes the acquisition function of max-value entropy search \citep{wang2017max}.
Similarly to \textsc{SafeOpt}, ISE-BO treats expansion and exploration separately, which leads to ``undirected'' expansion of the safe set.
That is, the safe set is expanded in all directions, even those that are known to be suboptimal.
In contrast, \itl balances expansion and exploration through the single set of potential maximizers $\spA_n$.
With a stochastic target space, \itl generalizes max-value entropy search (cf. \cref{sec:safe_bo_appendix:thompson_sampling}).\looseness=-1

We evaluate ISE-BO in \cref{fig:safe_bo_ts} and observe that it does not outperform \itl and \vtl in any of the tasks, while performing poorly in the 1d~task and suboptimally in the 2d~task.\looseness=-1

\subsection{Jumping Past Local Barriers}\label{sec:safe_bo_appendix:jumping_past_local_barriers}

In this additional experiment we demonstrate that \itl is able to extrapolate safety beyond local unsafe ``barriers'', which is a fundamental limitation of Lipschitz-based methods such as \textsc{SafeOpt}.
We consider the ground truth function and prior statistical model shown in \cref{fig:boundary_exploration:jumping_past_local_barriers:1d:prior}.
Note that initially, there are three disjoint safe ``regions'' known to the algorithm corresponding to two of the three safe ``bumps'' of the ground truth function.
In this experiment, the main challenge is to ``jump past'' the local barrier separating the leftmost and initially unknown safe ``bump''.\looseness=-1

\begin{figure}[t]
  \centering
  \includesvg[width=0.4\columnwidth]{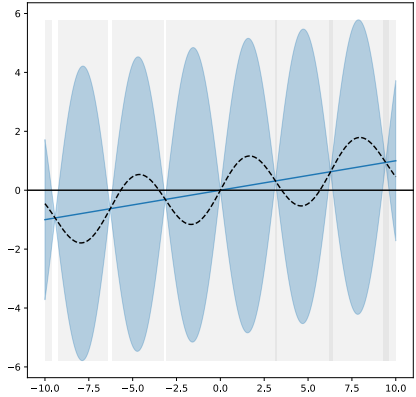}
  \caption{The ground truth $\opt{f}$ is shown as the dashed black line. The solid black line denotes the constraint boundary. The GP prior is given by a linear kernel with $\sin$-transform and mean $0.1 x$. The light gray region denotes the initial optimistic safe set $\widehat{\spS}_0$ and the dark gray region denotes the initial pessimistic safe set $\spS_0$.}
  \label{fig:boundary_exploration:jumping_past_local_barriers:1d:prior}
\end{figure}

\begin{figure}[t]
  \centering
  \includesvg[width=0.8\columnwidth]{figures/safe_bo/jumping_past_local_barriers/JPLB_samples}
  \caption{First $100$ samples of \itl using the potential expanders $\spE_n$ (cf. \cref{eq:potential_expanders}) as target space (left) and \textsc{SafeOpt} sampling only from the set of expanders $\spG_n^{\textsc{SafeOpt}}$ (right).}
  \label{fig:boundary_exploration:jumping_past_local_barriers:1d:samples}
\end{figure}

\Cref{fig:boundary_exploration:jumping_past_local_barriers:1d:samples} shows the sampled points during the first $100$ iterations of \textsc{SafeOpt} and \itl.
Clearly, \textsc{SafeOpt} does not discover the third safe ``bump'' while \itl does.
Indeed, it is a fundamental limitation of Lipschitz-based methods that they can never ``jump past local barriers'', even if the oracle Lipschitz constant were to be known and tight (i.e., locally accurate) around the barrier.
This is because Lipschitz-based methods expand to the point $\vx$ based on a reference point $\vxp$, and by definition, if $\vx$ is added to the safe set so are all points on the line segment between $\vx$ and $\vxp$.
Hence, if there is a single point on this line segment which is unsafe (i.e., a ``barrier''), the algorithm will \emph{never} expand past it.
This limitation does not exist for kernel-based algorithms as expansion occurs in function space.\looseness=-1

Moreover, note that for a non-stationary kernel such as in this example, \itl samples the ``closest points'' in function space rather than Euclidean space.
We observe that \textsc{SafeOpt} still samples ``locally at the boundary'' whereas \itl samples the most informative point which in this case is the local maximum of the sinusoidal function.
In other words, \itl adapts to the geometry of the function.
This generally leads us to believe that \itl is more capable to exploit (non-stationary) prior knowledge than distance-based methods such as \textsc{SafeOpt}.\looseness=-1

\subsection{Experiment Details}\label{sec:safe_bo_appendix:details}

\subsubsection{Synthetic Experiments}

\paragraph{1d task}

\begin{figure}[t]
  \centering
  \includesvg[width=0.4\columnwidth]{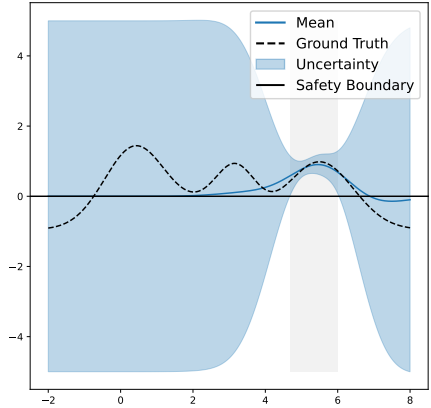}
  \caption{Ground truth and prior well-calibrated model in 1d synthetic experiment. The function serves simultaneously as objective and as constraint. The light gray region denotes the initial safe set~$\spS_0$.}
  \label{fig:safe_bo_1d_hard}
\end{figure}

\Cref{fig:safe_bo_1d_hard} shows the objective and constraint function, as well as the prior.
We discretize using $500$ points.
The main difficulty in this experiment lies in sufficiently expanding the safe set to discover the global maximum.
\Cref{fig:safe_bo_safe_set_size} plots the size of the safe set $\spS_n$ for the compared algorithms, which in this experiment matches the achieved regret closely.\looseness=-1

\begin{figure}[t]
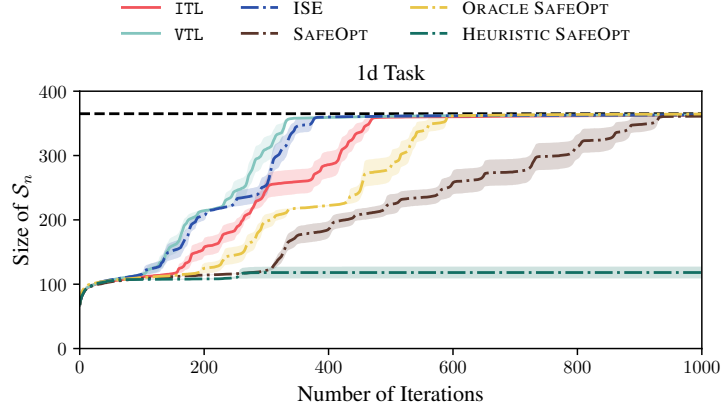

    \incplt[0.7\columnwidth]{safe_bo_safe_set_size}
    \vspace{-0.5cm}
    \caption{Size of $\spS_n$ in 1d synthetic experiment. The dashed black line denotes the size of $\opt{\spS}$. In this task, ``discovering'' the optimum is closely linked to expansion of the safe set, and \textsc{Heuristic SafeOpt} fails since it does not expand the safe set sufficiently.}
    \label{fig:safe_bo_safe_set_size}
\end{figure}

\paragraph{2d task}

\begin{figure}[t]
  \centering
  \includesvg[width=0.7\columnwidth]{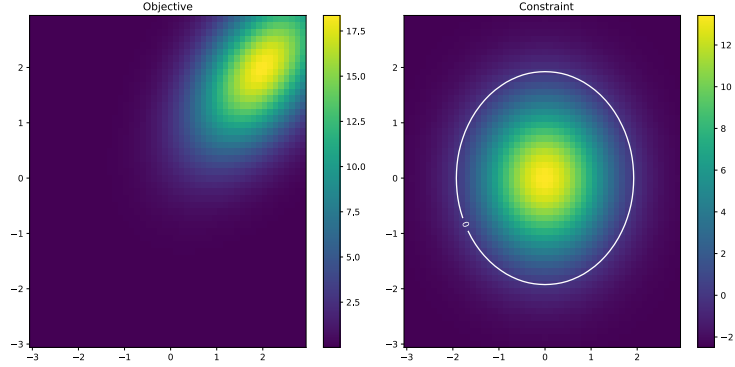}
  \caption{Ground truth in 2d synthetic experiment.}
  \label{fig:safe_bo_2d_islands}
\end{figure}

We model our constraint in the form of a spherical ``island'' where the goal is to get a good view of the coral reef located to the north-east of the island while staying in the interior of the island during exploration~(cf.~\cref{fig:safe_bo_2d_islands}).
The precise objective and constraint functions are unknown to the agent.
Hence, the agent has to gradually and safely update its belief about boundaries of the ``island'' and the location of the coral reef.
The prior is obtained by a single observation within the center of the island $[-0.5, 0.5]^2$.
We discretize using $2\,500$ points.\looseness=-1

\subsubsection{Safe Controller Tuning for Quadcopter}\label{sec:safe_bo_appendix:details:quadcopter}

\paragraph{Modeling the real-world dynamics}

We learn a feedback policy (i.e., ``control gains'') to compensate for inaccuracies in the initial controller.
In our experiment, we model the real world dynamics and the adjusted model using the PD control feedback \citep{widmer2023tuning}, \begin{align}
  \vdelta_t(\vx) \defeq (\opt{\vx} - \vx) [(\opt{\vs} - \vs_t) \; (\opt{\dot{\vs}} - \dot{\vs}_t)], \label{eq:pd_control}
\end{align} where $\opt{\vx}$ are the \emph{unknown} ground truth disturbance parameters, and $\opt{\vs}$ and $\opt{\dot{\vs}}$ are the desired state and state derivative, respectively.
This yields the following ground truth dynamics: \begin{align}
  \vs_{t+1}(\vx) = \vT(\vs_t, \vu_t + \vdelta_t(\vx)).
\end{align}

The feedback parameters $\vx = \transpose{[\vx_p \; \vx_d]}$ can be split into $\vx_p$ tuning the state difference which are called \emph{proportional parameters} and $\vx_d$ tuning the state derivative difference which are called \emph{derivative parameters}.
We use the ``critical damping'' heuristic to relate the proportional and derivative parameters: $\vx_d = 2 \sqrt{\vx_p}$.
We thus consider the restricted domain $\spX = [0, 20]^4$ where each dimension corresponds to the proportional feedback to one of the four rotors.\looseness=-1

Ground truth disturbance parameters are sampled from a chi-squared distribution with one degree of freedom (i.e., the square of a standard normal distribution), $\opt{\vx_p} \sim \chi_1^2$, and $\opt{\vx_d}$ is determined according to the critical damping heuristic.\looseness=-1

\paragraph{The learning problem}

The goal of our learning problem is to move the quadcopter from its initial position ${\vs(0) = \transpose{[1 \; 1 \; 1]}}$ (in Euclidean space with meter as unit) to position ${\opt{\vs} = \transpose{[0 \; 0 \; 2]}}$.
Moreover, we aim to stabilize the quadcopter at the goal position, and therefore regularize the control signal towards an action $\opt{\vu}$ which results in hovering (approximately) without any disturbances.
We formalize these goals with the following objective function: \begin{align}
  \opt{f}(\vx) \defeq - \sigma\parentheses*{\sum_{t=0}^T \norm{\opt{\vs} - \vs_t(\vx)}_{\mQ}^2 + \norm{\opt{\vu} - \vu_t(\vx)}_{\mR}^2} \label{eq:quadcopter_objective}
\end{align} where $\sigma(v) \defeq \tanh((v - 100) / 100)$ is used to smoothen the objective function and ensure that its range is $[-1,1]$.
The non-smoothed control objective in \cref{eq:quadcopter_objective} is known as a \emph{linear-quadratic regulator} (LQR) which we solve exactly for the undisturbed system using ILQR \citep{trajax2023github}.
Finally, we want to ensure at all times that the quadcopter is at least $0.5$ meter above the ground, that is, \begin{align}
  \opt{g}(\vx) \defeq \min_{t \in [T]} \vs_t^z(\vx) - 0.5
\end{align} where we denote by $\vs_t^z$ the z-coordinate of state $\vs_t$.\looseness=-1

We use a time horizon of $T = 3$ seconds which we discretize using $100$ steps.
The objective is modeled by a zero-mean GP with a Matérn(${\nu=5/2}$) kernel with lengthscale $0.1$, and the constraint is modeled by a GP with mean $-0.5$ and a Matérn(${\nu = 5/2}$) kernel with lengthscale $0.1$.
The prior is obtained by a single observation of the ``safe seed'' $\transpose{[0 \; 0 \; 0 \; 10]}$.\looseness=-1

\paragraph{Adaptive discretization}

We discretize the domain $\spX$ adaptively using coordinate \textsc{LineBO} \citep{kirschner2019adaptive}.
That is, in each iteration, one of the four control dimensions is selected uniformly at random, and the active learning oracle is executed on the corresponding one-dimensional subspace.\looseness=-1

\paragraph{Safety}

Using the (unsafe) constrained BO algorithm EIC~\citep{gardner2014bayesian} leads constraint violation,\footnote{On average, $1.6$ iterations of the first $50$ violate the constraints.} while \itl and \vtl do not violate the constraints during learning for any of the random seeds.

\paragraph{Hyperparameters}

The observation noise is Gaussian with standard deviation $\rho = 0.1$.
We let $\beta = 10$.
The control target is $\opt{\vu} = \transpose{[1.766 \; 0 \; 0 \; 0]}$.\looseness=-1

The state space is $12$-dimensional where the first three states correspond to the velocity of the quadcopter, the next three states correspond to its acceleration, the following three states correspond to its angular velocity, and the last three states correspond to its angular velocity in local frame.
The LQR parameters are given by \begin{align*}
  \mQ &= \diag{\braces{1, 1, 1, 1, 1, 1, 0.1, 0.1, 0.1, 0.1, 0.1, 0.1}} \quad\text{and} \\
  \mR &= 0.01 \cdot  \diag{\braces{5, 0.8, 0.8, 0.3}}.
\end{align*}

The quadcopter simulation was adapted from \cite{quadcopter_sim2023github}.\looseness=-1

Each one-dimensional subspace is discretized using $2\,000$ points.\looseness=-1

\paragraph{Random seeds}

We repeat the experiment for $25$ different seeds where the randomness is over the ground truth disturbance, observation noise, and the randomness in the algorithm.\looseness=-1